\documentclass{article} % For LaTeX2e
\makeatletter
\def\input@path{{iclr2027/}}
\makeatother
\usepackage{iclr2027_conference,times}
\usepackage{microtype}
\usepackage{graphicx}
\usepackage{booktabs}
\usepackage{xcolor}
\usepackage{amsmath}
\usepackage{amssymb}
\usepackage{mathtools}
\usepackage{amsthm}
\usepackage{multirow}
\usepackage{enumitem}
\usepackage{hyperref}
\usepackage{url}
\usepackage{wrapfig}
\usepackage{booktabs,makecell}
\usepackage{pgfplots}
\pgfplotsset{compat=1.18}
\usetikzlibrary{arrows.meta,patterns,calc,positioning}
\usepackage[table]{xcolor}
\usepackage{listings}
\usepackage[normalem]{ulem}
\useunder{\uline}{\ul}{}
\usepackage{amsfonts}
\usepackage{bbm}
\usepackage{subfigure}
\usepackage[textsize=tiny]{todonotes}
\usepackage{comment}
\usepackage{algorithm}
\usepackage{algorithmic}
\usepackage{amsmath}
\usepackage{amssymb}
\usepackage{mathtools}
\usepackage{amsthm}
\definecolor{myblue}{RGB}{77,123,174}

\usepackage{amsmath,amsfonts,bm}

\def\eqref#1{equation~\ref{#1}}
\def\1{\bm{1}}

\DeclareMathAlphabet{\mathsfit}{\encodingdefault}{\sfdefault}{m}{sl}
\SetMathAlphabet{\mathsfit}{bold}{\encodingdefault}{\sfdefault}{bx}{n}

\newcommand{\E}{\mathbb{E}}

\DeclareMathOperator*{\argmin}{arg\,min}

\theoremstyle{plain}
\newtheorem{theorem}{Theorem}[section]
\newtheorem{proposition}[theorem]{Proposition}
\newtheorem{lemma}[theorem]{Lemma}
\newtheorem{corollary}[theorem]{Corollary}

\theoremstyle{definition}
\newtheorem{definition}[theorem]{Definition}
\newtheorem{assumption}[theorem]{Assumption}

\theoremstyle{remark}
\newtheorem{remark}[theorem]{Remark}

\lstdefinelanguage{json}{
    basicstyle=\small\ttfamily,
    numbers=left,
    numberstyle=\scriptsize,
    stepnumber=1,
    numbersep=8pt,
    showstringspaces=false,
    breaklines=true,
    frame=lines,
    string=[s]{"}{"},
    stringstyle=\color{blue},
    comment=[l]{:},
    commentstyle=\color{black},
}

\usepackage[capitalize,noabbrev]{cleveref}

\usepackage{url}
\usepackage{hyperref}
\makeatletter
\let\iclr@maketitle\maketitle
\renewcommand{\maketitle}{%
  \microtypesetup{activate=false}%
  \iclr@maketitle
  \microtypesetup{activate=true}%
}
\makeatother

\title{Collaborative Personalized Preference Alignment for LLMs under Data Deficiency}

\author{
Liyan Yang \\
City University of Hong Kong \\
\texttt{liyanyang3-c@my.cityu.edu.hk} \\
\And
Yige Yuan \\
University of Washington \\
\texttt{yige@uw.edu} \\
\AND
Zhiqin Yang\thanks{Corresponding author.} \\
The Hong Kong University of Science and Technology \\
\texttt{yangzqccc@gmail.com}
}

\iclrfinalcopy % Uncomment for camera-ready version, but NOT for submission.
\begin{document}

\maketitle

\begin{abstract}
Real-world users often exhibit highly heterogeneous preferences over multiple objectives for LLM responses. A lightweight aligner can tailor these responses to individual preferences, but scarce user-specific feedback makes personalized training difficult. Learning shared initializations across users can support few-shot adaptation. However, heterogeneous preferences and competing objectives cause gradient conflicts across users and within each user, hindering effective initialization learning. This raises a central question: \textbf{how can we collaboratively learn aligner initializations that support few-shot adaptation to diverse user preferences?} To answer this question, we propose \textbf{A}pproximate \textbf{P}areto \textbf{O}ptimality (APO). We first group users whose updates are compatible, so that their information can be combined with less interference. Within each group, we combine gradient descent with controlled ascent to coordinate competing objectives and move towards preference-specific points on the Pareto front. This produces an initialization that is close to the optima of the users in the group. We then iteratively refine it using updates from few-shot local adaptation, making it more effective for personalization. Furthermore, we establish conditional suboptimality bounds for a one-local-step collaborative update and characterize how initialization error affects subsequent stochastic adaptation. Experiments on Fed-ChatbotPA and UltraFeedback show consistent improvements over existing methods using only 20 local examples.
%We theoretically show that each learned initialization lies within an $\epsilon_0$-band of the corresponding preference-specific optima, and that few-shot adaptation reaches a $\delta_0$-neighbourhood of the full-data optimum, including for unseen preferences.
\end{abstract}

\section{Introduction}
\label{sec: intro}
Large language models (LLMs)~\citep{openai_gpt6, aimlapi_gemini} serve clients with various cultural background, personal experience, and
individual values. As a result, it is a key requirement to tailor generated response to personal preferences.
Multi-objective preference alignment (MOPA) captures this need by generating outputs that match a client-specified balance across objectives such as helpfulness, harmlessness, and humor~\citep{yang2024rewards, shi2024decoding, guo2024controllable}. 
A straightforward way to personalize an LLM is to adapt a plug-and-play aligner~\citep{ji2024aligner, yang2024metaaligner} to adjust LLM responses on certain preference feedback of client. However, the scarcity of client-specific preference feedback creates a critical bottleneck for this approach, making it difficult to train an effective personalized aligner from scratch.

Existing approaches address this scarcity through data augmentation, including online imitation learning~\citep{shaikh2025aligning} and meta-learning with synthetic user preferences~\citep{singh2026fspo}. However, synthetic data may not fully capture real users' subtle preferences in style, tone, and context.
Personalized federated learning~\citep{wu2024towards,ye2024fedllm,jiang2019improving,nguyen2026learning} offers a promising alternative by enabling collaboration across clients using their local data. In this setting, clients can jointly learn a shared initialization that captures common patterns and then adapt it to their own preferences using a few local examples.
However, such collaboration faces \emph{preference gradient conflicts} at two levels~\citep{zhangrobust,guo2026few}. Across clients, heterogeneous preferences can cause locally beneficial updates to interfere during aggregation, leading to \emph{inter-client} conflicts. Within each client, gradients for different objectives may point in opposing directions, leading to \emph{intra-client} conflicts.
These challenges raise a key question: \textbf{how can we get aligner initializations enabling few-shot adaptation while addressing both inter- and intra-client preference gradient conflicts?}

%To make such few-shot personalization effective, the shared initialization should not merely perform well on average across clients; it should serve as a good starting point for rapid adaptation to diverse user preferences. This makes meta-learning a natural modeling choice. Classical MAML-style methods~\citep{finn2017model,jiang2019improving} are designed to learn transferable initializations for fast adaptation, and recent work such as MetaAligner~\citep{yang2024metaaligner} shows that meta-learning can improve multi-objective alignment for LLMs. 
%However, existing methods do not capture our collaborative personalized setting, in which the server must learn from decentralized clients whose preferences are heterogeneous, while each client must adapt from limited local feedback toward a preference-conditioned Pareto trade-off. This raises two coupled challenges, \emph{collaborative preference heterogeneity} and \emph{preference-conditioned few-shot adaptation}.

\begin{figure}[t]
\centering
\includegraphics[width=\linewidth]{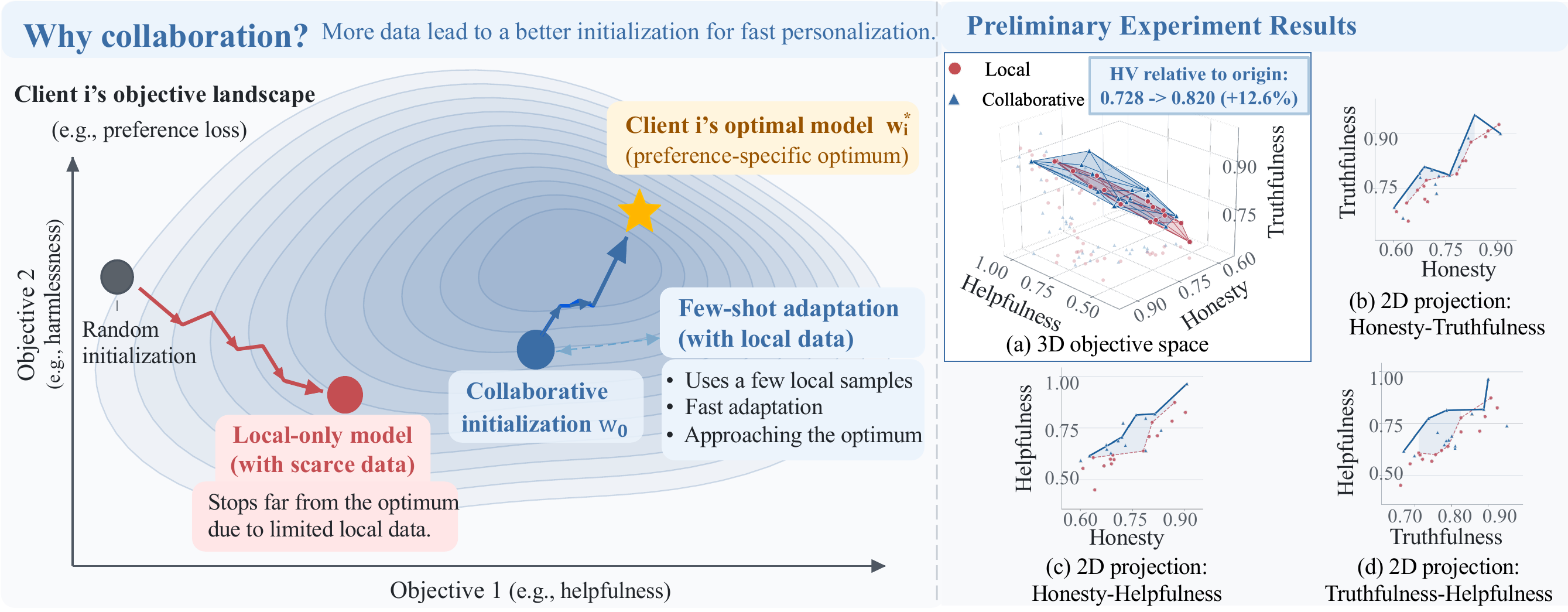}
\caption{Collaborative personalized preference alignment under data deficiency. \textbf{Left:} with scarce local feedback, local-only training from a random initialization stops far from a client's preference-specific optimum, whereas a collaboratively learned initialization enables effective few-shot personalization. \textbf{Right:} preliminary results on UltraFeedback~\citep{cui2023ultrafeedback} with Llama-3.2-3B-Instruct~\citep{meta2024llama32} (15 clients, 20 samples each). Over helpfulness, honesty, and truthfulness, collaborative initialization raises the pareto hypervolume (HV) from 0.728 to 0.820 ($+12.6\%$), and its frontier dominates the local one in the 2D projections.}
\label{fig:intro}
\vspace{-0.6cm}
\end{figure}

%\textcolor{blue}{1. cluster: 1.1 preference bottleneck, 1.2 adjustment vector clustering; 2. within cluster: collaborative training: 1.1 preference gradient update with control ascent, 1.2 meta-learning with server feedback;}
To solve this problem, we propose Approximate
Pareto Optimality (APO). The framework learns an initialization for each client cluster~\citep{ghosh2020ifca}. 
Compared with purely local training, each collaborative initialization reaches a point closer to the preference-specific optima of its clients. Starting from this initialization, few-shot adaptation is further guaranteed to approach the optimum attainable with sufficient local data more closely (see Figure~\ref{fig:intro}).
%Each initialization lies within an $\epsilon_0$-band of the preference-specific optima of its clients. Starting from the corresponding initialization, few-shot adaptation is provably guaranteed to reach a $\delta_0$-neighbourhood of the optimum attainable with sufficient local data 
% This guarantee also holds for preferences unseen during training. 
APO first addresses inter-client conflicts. It groups clients according to their \emph{bottleneck} objectives, defined as the objectives with the largest preference-weighted losses. These bottlenecks determine the clients' admissible update directions. We first partition clients by bottleneck ordering. We then refine each group using preference-adjustment vectors. This procedure reduces interference during aggregation and produces more reliable shared initializations. Furthermore, it combines preference gradient descent with controlled ascent~\citep{mahapatra2020multi} to traverse the pareto front within each cluster to solve intre-client conflicts. This yields the cluster-level initialization. Finally, we iteratively refine these initializations through adaptation-aware learning~\citep{nichol2018first,jiang2019improving}. In each round, clients adapt an initialization using a few local examples, and the server uses the resulting updates to refine it. Repeating this process produces initializations that adapt more effectively to individual preferences with limited local feedback. Under explicit assumptions, single-step collaborative updates achieve an asymptotic chebyshev suboptimality bound of $\epsilon_0$ for each client. Preconditioned stochastic adaptation yields an expected suboptimality bound of $\delta_0$ relative to the full-data optimum for unseen preferences over shared objectives.
Experiments on different datasets show that APO achieves the strongest overall performance among the evaluated methods in both preference-specific utility and pareto-front coverage.

In summary, we make the following contributions: \textbf{1)} We formalize collaborative personalized preference alignment of LLMs under data scarcity, and propose APO that learns a shared initialization for preference-conditioned few-shot adaptation. \textbf{2)} We introduce bottleneck-adjustment clustering and preference gradient descent with controlled ascent to address inter-client and intra-client preference gradient conflict. \textbf{3)} We establish an asymptotic chebyshev suboptimality bound for single-local-step collaborative updates, and an expected suboptimality bound for preconditioned stochastic adaptation to unseen preferences under shared objective functions. \textbf{4)} Extensive experiments on Fed-ChatbotPA and UltraFeedback with Llama-3.2-3B-Instruct and Qwen2.5-3B-Instruct demonstrate the superiority of APO over state-of-the-art methods.

\section{Related Work}
\label{sec: related_work}
\paragraph{Personalized Multi-Objective Alignment under Data Scarcity.}
Multi-objective alignment realizes client-specified trade-offs through preference-conditioned training~\citep{yang2024rewards}, post-training interpolation~\citep{rame2023rewarded}, or a single model that recovers the Pareto front~\citep{zhong2024panacea}.
MetaAligner~\citep{yang2024metaaligner} meta-learns a lightweight aligner that generalizes across unseen objectives, but not across users with heterogeneous multi-objective preferences.
Under limited feedback, FSPO~\citep{singh2026fspo} adapts from synthetic user preferences, while DITTO~\mbox{\citep{shaikh2025aligning}} learns from a few user demonstrations.
These methods rely on abundant centralized or synthetic data and fail to learn a personalized model from scarce real-client feedback under preference-specific settings.

\paragraph{Collaborative Alignment under Gradient Conflicts.}
Federated preference alignment applies local DPO or trains subpopulation selectors~\citep{ye2024openfedllm,wu2024towards}, while FIRM~\citep{roknilamouki2025firm} resolves multi-objective disagreement through client-side regularized MGDA; these methods nevertheless optimize a single population-level model without few-shot personalization guarantees.
Personalized FL learns shared initializations~\citep{jiang2019improving} or a few models for many clients~\citep{guo2026few}, but does not model preference-specific Pareto geometry.
For intra-model conflicts, EPO~\citep{mahapatra2020multi} uses controlled ascent to reach a preference-specific Pareto point, yet does not address interference among clients with heterogeneous preferences.
APO jointly addresses inter-client and intra-client gradient conflicts under data deficiency through bottleneck-adjustment clustering, controlled-ascent updates, and adaptation-aware meta-training.
Further related work is discussed in Appendix~\ref{app:related_work}.

\label{sec: method}
\section{Methodology}
We propose APO to learn aligner initialization that provide useful starting points for clients with different multi-objective preferences and limited local feedback.
We first characterize the preference corrections involved in collaboration and use them to organize clients into clusters.
Within each cluster, collaborative Pareto set initialization learning produces a basic initialization, which is then refined through adaptation-aware meta-training.
A new client selects the corresponding initialization and performs a few-shot adaptation to obtain a personalized model tailored to its own preference. In the following, we provide a detailed description of our framework.
Detailed proofs of all theoretical results are shown in Appendix~\ref{initial_aligner}.

\subsection{Preliminaries and Problem Setup}
\label{method:pre}
%Human preferences are inherently multi-dimensional: a response is judged along several, often conflicting, objectives such as helpfulness, harmlessness, and humor. Multi-objective preference alignment addresses this by steering the LLM toward outputs that match a client-specified trade-off among these objectives. 
Denoting the trainable parameters of the aligner as \( w \in \Omega \), where $\Omega$ is the parameter space, we associate each of the $m$ preference dimensions with an alignment loss $l^{(j)}:\Omega\to\mathbb{R}_{+}$, $j\in[m]$, computed on the preference feedback of dimension $j$ (a DPO loss in our implementation). For $\boldsymbol{l}^1,\boldsymbol{l}^2\in\mathbb{R}^m$ we write $\boldsymbol{l}^1\le\boldsymbol{l}^2$ if $\boldsymbol{l}^2-\boldsymbol{l}^1\in\mathbb{R}^m_{+}$ (componentwise), and $\boldsymbol{l}^1<\boldsymbol{l}^2$ if in addition $\boldsymbol{l}^1\ne\boldsymbol{l}^2$~\citep{mahapatra2020multi}. 
%Multi-objective preference alignment is then the multi-objective optimization problem $\min_{w \in \Omega} \boldsymbol{l}(w) = \bigl( l^{(1)}(w),\ l^{(2)}(w),\ \dots,\ l^{(m)}(w) \bigr).$
%Very often, no single solution \(w\) minimizes all dimensions simultaneously due to their conflicts. Instead, there exists a set of Pareto optimal solutions with distinct trade-offs among the dimensions. In the context of minimization, we say that \(w^{(a)}\) \emph{dominates} \(w^{(b)}\) if \( \boldsymbol{l}(w^{(a)}) < \boldsymbol{l}(w^{(b)}) \), i.e., $w^{(a)}$ is no worse on every dimension and strictly better on at least one~\citep{ehrgott2005multicriteria,miettinen1999nonlinear,mahapatra2020multi}. 
Based on this, Pareto optimality of multi-objective preference alignment is defined as:

\begin{definition}[Pareto optimality]
\label{po}
A solution \(w^*\) is \textit{Pareto optimal} if no other solution \(w' \in \Omega\) dominates \(w^*\). The set of all Pareto optimal solutions is the \textit{Pareto set (PS)} \(\mathcal{P}\), and its image \(\boldsymbol{l}(\mathcal{P})\subset\mathbb{R}^m_{+}\) is the \textit{Pareto front (PF)} \(\mathcal{T}\). A solution \(w^*\) is \textit{weakly Pareto optimal} if there is no \(w'\in\Omega\) with \( l^{(j)}(w') < l^{(j)}(w^*) \) for all \( j \in [m] \)~\citep{zhong2024panacea}.
\end{definition}
In the collaborative learning setting, each client $i$ holds its own preference feedback, which induces a client-specific loss vector $\boldsymbol{l}_i(w)=(l_i^{(1)}(w),\dots,l_i^{(m)}(w))$, and expresses its trade-off through a preference vector $\boldsymbol{\lambda}_i \in S^m$, where $S^m:=\{\boldsymbol{\lambda}\in\mathbb{R}^m_{++}:\sum_{j}\lambda^{(j)}=1\}$; a larger $\lambda_i^{(j)}$ means client $i$ tolerates a smaller loss on dimension $j$. Following the exact Pareto optimality of~\mbox{\citep{mahapatra2020multi}}, the \emph{preference-specific} Pareto optimal solutions of client $i$ form the set
\begin{equation}\label{eq:pareto_ray}
    \mathcal{P}_i(\boldsymbol{\lambda}_i) = \left\{ w_i^* \in \mathcal{P}_i \;\middle|\; \lambda_{i}^{(1)} l_{i}^{(1)*} = \cdots = \lambda_{i}^{(j)} l_{i}^{(j)*} = \cdots = \lambda_{i}^{(m)} l_{i}^{(m)*} \right\},
\end{equation}
where $\mathcal{P}_i$ is the Pareto set of $\boldsymbol{l}_i$ and $l_{i}^{(j)*} := l_{i}^{(j)}(w_i^*)$. Geometrically, for any $w_i^*\in\mathcal{P}_i(\boldsymbol{\lambda}_i)$ the loss vector $\boldsymbol{l}_i(w_i^*)$ is the intersection of the Pareto front with the ray in the direction of $\boldsymbol{\lambda}_i^{-1} := (1/\lambda_i^{(1)}, \cdots, 1/\lambda_i^{(m)})$. Such an intersection need not exist for every $\boldsymbol{\lambda}_i$; in that case we seek the Pareto optimal solution closest to the ray. In this paper, we take the \emph{weighted Chebyshev scalarization} as the guiding objective: minimizing $\psi_i(w)=\max_{j\in[m]}\lambda_i^{(j)} (l_i^{(j)}(w)-l_{i}^{(j)*})$ drives the iterate to a weakly Pareto optimal point. For proof, see the Lemma~\ref{prop:cheby2}. Leveraging Equation~\ref{eq:pareto_ray}, we reformulate the guiding objective as: minimizing $\psi_i(w)=\max_{j\in[m]}\lambda_i^{(j)}l_i^{(j)}(w)$.

To quantify how well the current loss vector $  l(w_i)  $ aligns with the client-specific preference vector $  \boldsymbol{\lambda_i}  $, we define the Preference Uniformity $  \beta_i(l(w_i))  $. For any point $w_i$ in the parameter space $\Omega$, 
$
    \beta_i(l(w_i)) = \sum_{j=1}^m \hat{l}_i^{(j)}(w_i) \log \left( \frac{\hat{l}_i^{(j)}(w_i)}{1/m} \right),
$
where $  \hat{l}_i^{(j)}(w_i) = \frac{\lambda_i^{(j)} l_i^{(j)}(w_i)}{\sum_{j'=1}^m \lambda_i^{(j')} l_i^{(j')}(w_i)}  $ denotes the weighted and normalized loss for task $  j  $ under preference $  \boldsymbol{\lambda_i}  $. The Preference Uniformity is exactly the Kullback–Leibler divergence $  \mathrm{KL}(\hat{l}_i \Vert \frac{1}{m}\mathbf{1})  $ and is nonnegative. It equals zero if and only if $  \lambda_i^{(j)} l_i^{(j)}$ is identical for all tasks $  j  $, i.e., when the weighted losses are perfectly balanced according to the client’s preference.

The direct differentiation of $\beta_i=\mathrm{KL}(\hat{l}_i\,\Vert\,\tfrac1m\mathbf1)$ with respect to $l_{i}^{(j)}$ gives $\partial\beta_i/\partial l_i^{(j)}=\lambda_i^{(j)} \left( \log \left( \frac{\hat{l}_i^{(j)}}{1/m} \right) - \beta_i(l) \right)/\sum_{j'}\lambda_i^{(j')} l_i^{(j')}$. Therefore, we define the Preference Adjustment Vector $  \boldsymbol{A_i} = (a_i^{(1)}, \dots, a_i^{(m)})^\top  $, whose $  j  $-th component 
$
    a_i^{(j)} = \lambda_i^{(j)} \left( \log \left( \frac{\hat{l}_i^{(j)}}{1/m} \right) - \beta_i(l) \right).
$
Writing the normalizer $Z_i:=\sum_{j'}\lambda_i^{(j')} l_i^{(j')}$, $\boldsymbol{A_i}=Z_i\,\nabla_{l}\beta_i$. Thus $\boldsymbol{A_i}$ points along the steepest increase of the Preference Uniformity, and moving against it reduces $\beta_i$ toward the balanced ray. Let $g^{(j)} = \nabla_w l_j$ be the gradient of the $j$th  preference dimension function 
at $w$, and $G \in \mathbb{R}^{n \times m}$ be the matrix having $g^{(j)}$ 
as its $j$th column. We can define the balancing direction $d_{bal} = G\boldsymbol{A_i}=Z_i\,\nabla_{w}\beta_i$, whose negative is the steepest-descent direction of the Preference Uniformity, i.e., along $d$ with $d^\top d_{\mathrm{bal},i}>0$. Accordingly, each client updates along a local direction $d_i=G_i\mu_i$ to decrease $\psi_i(w)$ with $\mu_i\in S^m$, a convex combination of its own per-dimension gradients chosen such that (a)~the bottleneck weighted loss is not increased and (b)~$d_i^\top d_{\mathrm{bal},i}>0$ whenever $\beta_i>0$. The concrete rule is given in Section~\ref{sec:initial_train}.

\subsection{Constructing Clusters: Bottleneck-Adjustment Clustering }
To address the challenge of inter-client preference gradient conflict, APO employs Bottleneck-Adjustment Clustering, which partitions clients by two stages: bottleneck ordering and then refines the groups using adjustment vectors.
The necessity of this design becomes clear once we examine how aggregation within a cluster can interfere with a client’s local balancing direction.

\phantomsection\label{cluster}
Consider a cluster $C_n$ whose clients aggregate their local directions with the FedAvg weights $\rho_i := N_i/N$, where $N_i$ is the number of local samples of client $i$ and $N=\sum_{i'\in C_n}N_{i'}$, yielding the shared direction $\bar d_{C_n}=\sum_{i\in C_n}\rho_i d_i$. And the $\bar d_{C_n}^\top d_{\mathrm{bal},i}$ can split as
$\bar d_{C_n}^\top d_{\mathrm{bal},i}
\;=\; d_i^\top d_{\mathrm{bal},i}
\;+\; (\bar d_{C_n}-d_i)^\top d_{\mathrm{bal},i}.$
The term $d_i^\top d_{\mathrm{bal},i}$ is what client $i$ would obtain from its own update; the term $(\bar d_{C_n}-d_i)^\top d_{\mathrm{bal},i}$ is contributed by the other clients and has no fixed sign. As an illustration of possible interference, take $m=2$ (e.g., helpfulness vs.\ harmlessness), let client $i$ have $\boldsymbol{\lambda}_i=(1-\epsilon,\epsilon)$, and let the remaining clients concentrate near the opposite corner $\boldsymbol{\lambda}_{i'}\approx(\epsilon,1-\epsilon)$. 
Their corrections aim at opposite ends of the Pareto front; if the gradient geometry at $w$ then yields $d_{i'}^\top d_{\mathrm{bal},i}<0$ for these clients, the aggregation effect grows with their total weight and can outweigh the local gain, so that $\bar d_{C_n}^\top d_{\mathrm{bal},i}<0$ and $\beta_i$ increases for all sufficiently small $\eta$. 
\paragraph{Stage 1: coarse partition by bottleneck ordering.} Therefore, to prevent such interference, we first eliminate differences in preference bottleneck structure among clients. The following sufficient condition makes this requirement precise.
\begin{theorem}[Collaborative preference uniformity descent]
\label{thm:collab_descent2}
Let Assumptions~\ref{asm:bounded}--\ref{asm:smooth} hold, let each client in $C_n$ run
$\tau$ local steps of size $\eta$ before aggregation, and write
$\gamma_i^* := d_i^\top d_{\mathrm{bal},i}$ for the local correction gain of property~(b) and
$\Delta_i := \|\bar d_{C_n} - d_i\|$.  If for every $i \in C_n$
\begin{equation}\label{eq:range_cond1}
    \gamma_i^*
    \;>\; \Delta_i\,\|d_{\mathrm{bal},i}\|
    \;+\; L_0 B^2\,\tau\,\eta ,
\end{equation}
which ensures that
$\bar d_{C_n}^\top d_{\mathrm{bal},i}(w_0) > 0$ for all $i \in C_n$, then the shared direction is a \emph{simultaneous} balancing direction, and there
exists $\eta_0 > 0$ such that a single global model update decreases \emph{every}
client's Preference Uniformity:
$
    \beta_i\bigl(l(w_0 - \eta\,\bar d_{C_n})\bigr)
    \;\leq\; \beta_i\bigl(l(w_0)\bigr),
     \forall\, i \in C_n,\ \forall\, \eta \in (0,\eta_0].
$
\end{theorem}
We therefore perform a coarse partition by bottleneck ordering. At a fixed reference state $w_{\mathrm{ref}}$ (the aligner before collaborative training), each client evaluates its weighted losses and its adjustment vector $\boldsymbol{A}_i^{\mathrm{ref}}:=\boldsymbol{A}_i(w_{\mathrm{ref}})$ on its own local feedback.  Specifically, for each client we compute the weighted losses $\lambda_i^{(j)}l_i^{(j)}(w_{\mathrm{ref}})$, $j\in[m]$, and record the permutation $\Gamma$ that sorts all $m$ dimensions in descending order of weighted loss. Clients with identical orderings $\Gamma$ form one coarse cluster. Nevertheless, sharing the bottleneck ordering does not make the aggregated direction automatically admissible, since the steps still depend on the client-specific adjustment vector $\boldsymbol{A_i}$.
\paragraph{Stage 2: refinement by adjustment vectors.} 
According to the Theorem~\ref{thm:collab_descent2}, the dispersion of the directions bounds the $\Delta_i$, which we measure by the preference-conditioned direction dispersion of cluster $C_n$ at parameter $w$,
\begin{equation}\label{eq:HS2}
    H_{C_n}(w) = \sum_{i \in C_n}\rho_i\,
    \bigl\|d_{i}(w) - \bar{d}_{C_n}(w)\bigr\|^2,
\end{equation}
where $d_{i}(w) = G_i(w)\mu_i$ and $\bar{d}_{C_n} = \sum_{i \in C_n}\rho_i\,d_{i}$. 
If all clients' LPs share a common non-degenerate optimal basis (Assumption~\ref{asm:lp_stable} in Appendix~\ref{sec:meta_quality}), the multipliers $\mu_i$ are affine in $\boldsymbol{A}_i$ with constant $\kappa_{\mathrm{LP}}$, and Proposition~\ref{prop:H_via_adj} gives, for arbitrary aggregation weights $\rho_i$,
$H_{C_n}(w)\;\le\;\kappa_{\mathrm{LP}}^2\,\sigma_{\max}\bigl(G^\top G\bigr)\,V_{C_n}(\boldsymbol{A}^{\mathrm{ref}}) ,$
where $V_{C_n}(\boldsymbol{A}^{\mathrm{ref}}) \;=\; \sum_{i\in C_n}\rho_i\bigl\|\boldsymbol{A}_i^{\mathrm{ref}}-\bar{\boldsymbol{A}}^{\mathrm{ref}}_{C_n}\bigr\|^2$ is the $\rho$-weighted within-cluster variance of the adjustment vectors around their weighted mean $\bar{\boldsymbol{A}}^{\mathrm{ref}}_{C_n}=\sum_{i\in C_n}\rho_i\boldsymbol{A}_i^{\mathrm{ref}}$.
We therefore refine the cluster by adjustment vectors, aiming to reduce the direction dispersion $H_{C_n}(w)$. Specifically, within each Stage-1 cluster we apply hierarchical clustering with Euclidean distance on $\{\boldsymbol{A}_i^{\mathrm{ref}}\}$, and cut the dendrograms so that the total number of clusters equals a preset budget $b$ (e.g., $b=4$ on Fed-ChatbotPA and $b=8$ on UltraFeedback).

\subsection{Learning Shared Initializations: Preference-Aware Collaborative Training}
\label{sec:initial_train}
Within each cluster, APO further addresses intra-client reference gradient conflict by combining preference gradient descent with controlled ascent to train the initialization. The procedure consists of two phases: collaborative Pareto set initialization learning, followed by adaptation-aware meta-training, which we describe in detail below.
\paragraph{Phase I: Collaborative Pareto Set Initialization Learning.}
Inspired by ~\citep{mahapatra2020multi}, we address the reference gradient conflict by combining the preference gradient descent with controlled ascent. More specifically, once preference balance is achieved ($\beta_i(\boldsymbol{l_i}) = 0$, so $\boldsymbol{A_i}=0$, and $\boldsymbol{l_i}$ lies on the $\boldsymbol{\lambda_i}^{-1}$ ray), we switch to a pure descent scheme to refine the model toward the task itself: from \citep{desideri2012multiple}, a common descent direction $d = G\mu$ in the convex hull of the gradients at $w_i^t$ points toward the Pareto front, so we find the $d = G\mu$ whose inner product with the sum of all gradients $G\mathbf{1}$ is maximum. In addition, we define the index sets $J = \{j \mid \boldsymbol{A_i^\top}G_{i}^\top g_{i}^{(j)} > 0\}$, $\bar{J} = \{j \mid \boldsymbol{A_i^\top}G_{i}^\top g_{i}^{(j)} \leq 0\}$ and the bottleneck set $J^* = \left\{j \mid \lambda_i^{(j)} l_{i,b}^{(j)}= \max_{j'}\{\lambda_i^{(j')} l_{i,b}^{(j')}\} \right\}$ (with $b$-th batch), and impose descent constraints that permit a controlled trade-off.  Combining the two modes through the indicator $\mathbbm{1}_{\beta_i}$, we define the Pareto set initialization learning objective as
\begin{subequations}\label{eq:6}
\begin{align}
\mu_{i}^* = \arg & \max_{\mu_{i} \in S^m} \mu_{i}^\top G_{i}^\top G_{i} 
    \bigl( \boldsymbol{A_i} \mathbbm{1}_{\beta_i} + \mathbf{1} (1 - \mathbbm{1}_{\beta_i}) \bigr) \label{eq:6a} \\
& \text{s.t.} \quad 
\mu_{i}^\top G_{i}^\top g_{i}^{(j)}  \geq \boldsymbol{A_i^\top} G_{i}^\top g_{i}^{(j)} \mathbbm{1}_J,
   &&&\forall j \in \bar{J} \setminus J^* \label{eq:6b} \\
&\quad \quad \mu_{i}^\top G_{i}^\top g_{i}^{(j)}  \geq 0,
   &&&\forall j \in J^* \label{eq:6c}
\end{align}
\end{subequations}
where $\mathbbm{1}_J$ is the scalar indicator for a non-empty index set $J$.
For the non-bottleneck objectives ($j \in \bar{J}\setminus J^*$), constraint Equation~\ref{eq:6b} requires $\mu_i^\top G_i^\top g_i^{(j)} \geq \boldsymbol{A_i^\top} G_i^\top g_i^{(j)}\mathbbm{1}_J$; since this bound is non-positive, these objectives are allowed to degrade, but by no more than the balancing direction. For bottleneck objectives ($j \in J^*$), constraint Equation~\ref{eq:6c} requires $\mu_i^\top G_i^\top g_i^{(j)} \geq 0$, that guaranties that the most preference-violating objective must not worsen (its loss is non-increasing). In the following, we refer to the process with $  \mathbbm{1}_{\beta_i} = 0  $ as uniform descent, and the process with $  \mathbbm{1}_{\beta_i} = 1  $ as balancing descent.
\begin{comment}
After determining the descent direction at local step $k$ on client $i$, given by $d_{i,k} = G \mu_{i}^*$, the client updates its local model parameters for $\tau$ steps as $w_{i}^{(k+1)} = w_{i}^{(k)} - \eta \, d_{i,k}, \quad k = 0, 1, \dots, \tau-1$, where $\eta$ is the local learning rate. In each communication round $r$, the server update its  initializer using FedAvg \citep{mcmahan2017communication}:
$w_0 = w_0 + \alpha \sum_{i \in C_n}\frac{N_i}{N} g_{i}$, where $  N = \sum_i N_i  $ is the total number of samples across all clients in cluster $C_n$, $  N_i  $ is the number of local samples on client $  i  $, and $\alpha$ is the server learning rate.
\end{comment}

\paragraph{Phase II: Adaptation-aware meta-training.}
Writing the model obtained in Phase I as $w_{\mathrm{FedBsc}}$,
inspired by \citep{jiang2019improving}, we meta-train $w_{\mathrm{FedBsc}}$ on $S$-shot episodes so that it can be adapted to a client preference within $K$ steps. Moreover, for a given preference $\boldsymbol{\lambda_i}$, we use the objective 
\begin{equation}\label{eq:fewshot_obj}
\mu^* = \arg\max_{\mu \in S^m} d^{T}d_{bal}
      = \arg  \max_{\mu \in S^m} \mu^T G^T G \boldsymbol{A_i}
      \quad 
       \text{s.t. } (\ref{eq:6b})\text{--}(\ref{eq:6c}).
\end{equation}
to drive the iterate onto the $\boldsymbol{\lambda_i}^{-1}$ ray---it reduces the Preference Uniformity $\beta_i$.
Phase II uses the same clients, the same local data (the $S$-shot episodes are subsamples of each client's own dataset), the same objective functions and the same FedAvg aggregation as Phase I.
The complete initialization procedure (both phases) is summarized in Algorithm~\ref{alg:1} in Appendix~\ref{app:algorithm}.

\subsection{Theoretical Guarantees: From Collaborative Convergence to Few-Shot Personalization}
\paragraph{Collaborative Initialization Converges to Cluster-Local Optima.}
\label{sec:convergence}
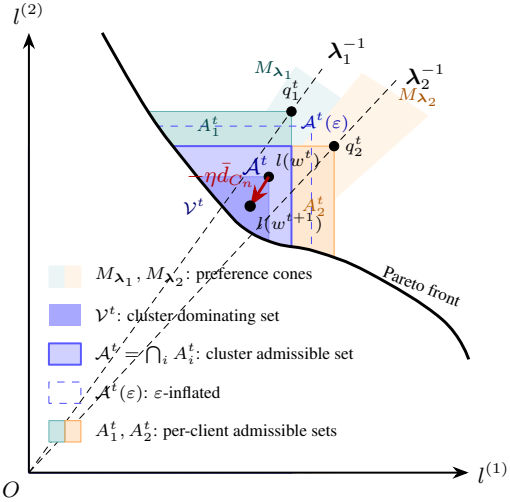
\begin{wrapfigure}[25]{r}{0.5\textwidth}
\centering
\resizebox{\linewidth}{!}{%
\begin{tikzpicture}[scale=1.9,>=Stealth,line join=round]
  % ===== drawn BEFORE the white erase: cones, boxes, limiting band =====
  % preference cones M_{lambda_i} = {beta_i <= bar beta} (trajectory hypothesis)
  \fill[teal!8!white]   (0,0) -- (2.00,3.05) -- (2.50,2.70) -- cycle;
  \fill[orange!9!white] (0,0) -- (2.50,3.00) -- (3.00,2.68) -- cycle;
  % per-client admissible boxes A_1^t, A_2^t (anchored at origin, far corner q_i^t)
  \fill[teal!20]   (0,0) rectangle (1.97,2.71);
  \fill[orange!18] (0,0) rectangle (2.29,2.45);
  % cluster admissible set A^t = A_1^t \cap A_2^t
  \fill[blue!18] (0,0) rectangle (1.97,2.45);
  % cluster dominating set V^t (innermost)
  \fill[blue!34] (0,0) rectangle (1.80,2.22);
  % box edges
  \draw[teal!60]   (0,0) rectangle (1.97,2.71);
  \draw[orange!70] (0,0) rectangle (2.29,2.45);
  \draw[blue!65,thick] (0,0) rectangle (1.97,2.45);
  \draw[blue!45]   (0,0) rectangle (1.80,2.22);
  % epsilon-inflated cluster admissible set A^t(eps) (dashed, outer)
  \draw[blue!70,dashed] (0,0) rectangle (2.12,2.60);

  % ===== white erase: blank the infeasible region below the front =====
  \fill[white] plot[smooth,tension=0.7] coordinates
    {(0.55,3.50) (0.90,2.70) (1.20,2.30) (1.45,2.00) (1.62,1.83)
     (1.85,1.72) (2.15,1.66) (2.55,1.50) (3.10,1.15) (3.70,0.85)}
    -- (3.70,0) -- (0,0) -- (0,3.50) -- cycle;

  % ===== drawn AFTER the erase =====
  % axes
  \draw[->,thick] (0,0) -- (3.3,0) node[right] {\small $l^{(1)}$};
  \draw[->,thick] (0,0) -- (0,3.3) node[above] {\small $l^{(2)}$};
  \node[below left] at (0,0) {\small $O$};

  % Pareto front
  \draw[very thick] plot[smooth,tension=0.7] coordinates
    {(0.55,3.30) (0.90,2.70) (1.20,2.30) (1.45,2.00) (1.62,1.83)
     (1.85,1.72) (2.15,1.66) (2.55,1.50) (3.10,1.15) (3.30,0.85)};
  \node[rotate=-22] at (2.95,1.4) {\scriptsize Pareto front};

  % preference rays lambda_i^{-1}
  \draw[densely dashed] (0,0) -- (2.20,3.03)
        node[above right=-2pt] {\small $\boldsymbol{\lambda}_1^{-1}$};
  \draw[densely dashed] (0,0) -- (2.78,2.97)
        node[right=-1pt] {\small $\boldsymbol{\lambda}_2^{-1}$};
  \node[teal!55!black]  at (1.84,3.02) {\scriptsize $M_{\boldsymbol{\lambda}_1}$};
  \node[orange!65!black] at (2.92,2.82) {\scriptsize $M_{\boldsymbol{\lambda}_2}$};

  % current-iterate ray points q_i^t
  \filldraw (1.97,2.71) circle (1.0pt) node[above=1pt] {\scriptsize $q_1^t$};
  \filldraw (2.29,2.45) circle (1.0pt) node[right=1pt] {\scriptsize $q_2^t$};

  % one shared FedAvg step -eta dbar from l(w^t) (far corner of V^t); the new loss
  % vector l(w^{t+1}) stays in every client's eps-inflated admissible box (Thm. collab_admissible)
  \filldraw (1.80,2.22) circle (1.0pt) node[above right=-1pt] {\scriptsize $l(w^{t})$};
  \draw[->,very thick,red!70!black] (1.80,2.22) -- (1.66,2.00)
        node[midway,above left=-3pt] {\small $-\eta\bar d_{C_n}$};
  \filldraw (1.66,2.00) circle (1.1pt) node[below right=-1pt] {\scriptsize $l(w^{t+1})$};

  % set labels
  \node[teal!55!black]   at (1.35,2.60) {\scriptsize $A_1^t$};
  \node[orange!70!black] at (2.14,2.00) {\scriptsize $A_2^t$};
  \node[blue!60!black]   at (1.70,2.31) {\small $\mathcal A^t$};
  \node[blue!45!black]   at (1.24,2.00) {\scriptsize $\mathcal V^t$};
  \node[blue!70!black]   at (2.22,2.63) {\scriptsize $\mathcal A^t(\varepsilon)$};

  % legend (in the infeasible wedge, lower-left)
  \begin{scope}[shift={(0.15,0.12)}]
    \fill[teal!8!white]   (0,1.27) rectangle (0.12,1.43);
    \fill[orange!9!white] (0.12,1.27) rectangle (0.24,1.43);
    \draw[teal!8!white]   (0,1.27) rectangle (0.12,1.43);
    \draw[orange!9!white] (0.12,1.27) rectangle (0.24,1.43);
    \node[right] at (0.28,1.35) {\scriptsize $M_{\boldsymbol{\lambda}_1},M_{\boldsymbol{\lambda}_2}$: preference cones};
    \fill[blue!34] (0,0.98) rectangle (0.24,1.14);
    \node[right] at (0.28,1.06) {\scriptsize $\mathcal V^t$: cluster dominating set};
    \fill[blue!18] (0,0.69) rectangle (0.24,0.85);
    \draw[blue!65,thick] (0,0.69) rectangle (0.24,0.85);
    \node[right] at (0.28,0.77) {\scriptsize $\mathcal A^t=\bigcap_i A_i^t$: cluster admissible set};
    \draw[blue!70,dashed] (0,0.40) rectangle (0.24,0.56);
    \node[right] at (0.28,0.48) {\scriptsize $\mathcal A^t(\varepsilon)$: $\varepsilon$-inflated};
    \fill[teal!20]   (0,0.11) rectangle (0.12,0.27);
    \fill[orange!18] (0.12,0.11) rectangle (0.24,0.27);
    \draw[teal!60]   (0,0.11) rectangle (0.12,0.27);
    \draw[orange!70] (0.12,0.11) rectangle (0.24,0.27);
    \node[right] at (0.28,0.19) {\scriptsize $A_1^t,A_2^t$: per-client admissible sets};
  \end{scope}
\end{tikzpicture}}
\caption{\small Cluster admissible set $\mathcal A^t$ for $m=2$. Dashed lines are the preference rays $\boldsymbol{\lambda}_i^{-1}$. One shared FedAvg step $-\eta\,\bar d_{C_n}$ moves $l(w^t)$ to $l(w^{t+1})$, which under Assumption~\ref{asm:descent_margin} stays in every $A_i^t(\varepsilon_{i,t})$ (Theorem~\ref{thm:collab_admissible}); the boxes are nested across rounds up to this slack (Corollary~\ref{cor:collab_monotone}).}
\label{fig:collab_adm}
\end{wrapfigure}
We analyze the collaborative initialization in three steps. First, we give a conditional bound on the preference uniformity of a client at which the balancing guarantee of Theorem~\ref{thm:collab_descent2} no longer applies, and an exact one-step bound (with second-order term) on the cluster losses along the shared update. Second, we define an admissible set $\mathcal A^t$ that contains each client's potential $l_{i}(w)$ values in an iteration $t$. Third, under an explicit \emph{descent-margin} condition on the LP direction (Assumption~\ref{asm:descent_margin}) and a trajectory hypothesis, we prove a contraction of each client's Chebyshev suboptimality and conclude that $\limsup_t\bigl(\psi_i(w^t)-\psi_i^*\bigr)$ is at most a fixed, horizon-independent radius $\epsilon_0$ (asymptotically bounded; for every $\zeta>0$ the suboptimality is eventually below $\epsilon_0+\zeta$); if the iterates have a limit point $w^*$, it inherits the bound and lies in the limiting band $\mathcal A^\infty$ (a one-sided upper bound on each client's losses). The descent-margin condition is a Polyak--{\L}ojasiewicz-type requirement on $\psi_i$ along the actual LP direction; it is not implied by smoothness, bounded gradients, preference balance or the LP constraints, and we state it as an assumption rather than derive it.

Along the collaborative iteration, Lemma~\ref{lem:collab_class} decreases every balancing descent client's $\beta_i$ in a cluster while the range condition Equation~\ref{eq:cp_mix_range} holds. Partition the cluster into $n_{+}:=|C_n^{+}|$ balancing ($\beta_i>0$) and $n_{0}:=|C_n^{0}|$ uniform ($\beta_i=0$) clients, with block masses $\rho^{+}:=\sum_{i\in C_n^{+}}\rho_i$ and $\rho^{0}:=\sum_{i\in C_n^{0}}\rho_i$; Lemma~\ref{lem:beta_floor} then bounds each $\beta_i$ by a tight $\bar\beta$, set by the preference-conditioned direction dispersion $H_{C_n}$ and, via Equation~\ref{eq:HS2} and Equation~\ref{eq:cp_delta_split}, by the cross-block pull $\rho^{0}\|\bar d^{0}-d_i\|$ in addition to preference heterogeneity.
Thus $l(w)$ lies in the preference cone $M_{\boldsymbol{\lambda}_i}=\{l:\beta_i(l)\le\bar\beta\}$ around each $\boldsymbol{\lambda}_i^{-1}$ ray (Figure~\ref{fig:collab_adm}) for any client at which the range condition of Theorem~\ref{thm:collab_descent2} fails---a per-client, per-iterate fact; the convergence result below takes $\beta_i\le\bar\beta$ along the trajectory as a hypothesis. Lemma~\ref{loss_avg} then gives an exact one-step bound on each loss and on $\psi_i$ along the shared update, including the second-order term $\tfrac{L_0B^2}{2}\eta^2$.

\begin{comment}
\begin{lemma}[One-step loss non-increase]\label{loss_avg}
By differentiability and $L_0$-smoothness (Assumption~\ref{asm:smooth}), the
first-order expansion of each loss along $\bar d_{C_n}$ gives, for $\eta\ge0$,
\begin{equation}\label{eq:cp_taylor}
\begin{aligned}
    l_i^{(j)}(w-\eta\bar d_{C_n})
    \;=\; l_i^{(j)}(w) - \eta\,\bar d_{C_n}^\top g_i^{(j)} + o(\eta)
     \;\le\; l_i^{(j)}(w)
        + \eta\bigl(B\,\delta_g
        + c_\beta\sqrt{\bar\beta}\bigr) + o(\eta),
\end{aligned}
\end{equation}
where the $c_\beta\sqrt{\bar\beta}$ is the residual balancing descent slack and $c_\beta=\rho^{+}BL_A\le BL_A$. Thus the shared update is an
approximate descent step with per-coordinate slack
$\varepsilon_t=\eta(B\delta_g+c_\beta\sqrt{\bar\beta})$.

\end{lemma}
\end{comment}

For client $i$ at iterate $w^t$ let
$O_i\subseteq\mathbb{R}^m_{+}$ be its attainable objective set and $\le$ the componentwise order. The $\psi_i^t:=\max_{j}\lambda_i^{(j)}\,l_i^{(j)}(w^t)$, we define the client ray point $q_i^t:=\psi_i^t\Bigl(\tfrac{1}{\lambda_i^{(1)}},\dots,
        \tfrac{1}{\lambda_i^{(m)}}\Bigr)$, the client 
dominating set $V_i^{t}:=\{l\in O_i: l\le l_i(w^t)\}$, and the client admissible set $A_i^{t}:=\{l\in O_i: l\le q_i^t\}$. The sets $A_i^t$ live in each client's own objective space $O_i$; a single shared model $w$, however, must serve every client at once.  We therefore lift the
construction to \emph{parameter space}, defining for the cluster dominating set $\mathcal V^t:=\bigcap_{i\in C_n}\{\,w:\ l_i(w)\le l_i(w^t)\,\} $, and the ($\varepsilon$-inflated) 
cluster admissible set $\mathcal A^t(\varepsilon) :=\bigcap_{i\in C_n}\{\,w:\ l_i(w)\le q_i^t+\varepsilon\mathbf1\,\}$.
\begin{comment}
\begin{equation}\label{eq:cp_cluster_adm2}
\begin{aligned}
     \mathcal V^t:=\bigcap_{i\in C_n}\{\,w:\ l_i(w)\le l_i(w^t)\,\},
    \quad 
     \mathcal A^t(\varepsilon):=\bigcap_{i\in C_n}\{\,w:\ l_i(w)\le q_i^t+\varepsilon\mathbf1\,\},
\end{aligned}
\end{equation}
\end{comment}
And we abbreviate $\mathcal A^t:=\mathcal A^t(0)=\bigcap_{i\in C_n}\{w:l_i(w)\le q_i^t\}$.
Thus $w\in\mathcal A^t(\varepsilon)$ iff $l_i(w)\in A_i^t(\varepsilon)$ for
\emph{every} client, i.e.\ the cluster admissible set is the set of shared
parameters whose induced losses fall in every client's inflated admissible box
simultaneously (the intersection is exactly the cluster wedge of
Figure~\ref{fig:collab_adm}).

Based on the Lemmas~\ref{lem:beta_floor}, ~\ref{loss_avg}, and the cluster admissible set analysis above, we can show:

\begin{theorem}[Cluster solution approximates every local optimum under a descent-margin condition]
\label{thm:collab_converge2}
Consider the synchronous shared iteration $w^{t+1}=w^t-\eta\bar d_{C_n}(w^t)$ (one local LP step per round) under Assumptions~\ref{asm:bounded}--\ref{asm:smooth} and the descent-margin condition of Assumption~\ref{asm:descent_margin}: there are $c_d>0$ and tolerances $b_i\ge0$ such that on a region $\mathcal N$, whenever $\beta_i\le\bar\beta$, $\min_{j\in J^*_{i,\varepsilon}(w)}\lambda_i^{(j)}d_i(w)^\top g_i^{(j)}(w)\ge c_d\,(\psi_i(w)-\psi_i^*)-b_i$. Let $0<\eta\le1/c_d$ and suppose that for all $t\ge t_0$ and $i\in C_n$, $w^t\in\mathcal N$, $\beta_i(w^t)\le\bar\beta$ and $\|\bar d_{C_n}(w^t)-d_i(w^t)\|\le\bar\Delta_i$. Then, with
$
\epsilon_{0,i}:=\frac{\|\boldsymbol{\lambda}_i\|_\infty}{c_d}\Bigl(B\,\bar\Delta_i+\frac{L_0B^2}{2}\,\eta\Bigr)+\frac{b_i}{c_d},
$
we obtain the contraction
$
\psi_i(w^{T})-\psi_i^*\le(1-\eta c_d)^{T-t_0}\bigl(\psi_i(w^{t_0})-\psi_i^*\bigr)+\epsilon_{0,i}.
$
Consequently,
$
\limsup_{T\to\infty}\bigl(\psi_i(w^T)-\psi_i^*\bigr)\le\epsilon_{0,i},
$
and any limit point \(w^*\) satisfies
$
l_i(w^*)\le q_i^*+\Bigl(\frac{\epsilon_{0,i}}{\min_j\lambda_i^{(j)}}\Bigr)\mathbf{1}
$
componentwise, where \(q_i^*=\psi_i^*\boldsymbol{\lambda}_i^{-1}\) denotes the ray point of client \(i\)'s local optimum. We write \(\epsilon_0:=\max_{i\in C_n}\epsilon_{0,i}\).
\end{theorem}
%Theorem~\ref{thm:collab_converge2} is conditional. The descent margin is a Polyak--{\L}ojasiewicz-type requirement on the actual LP direction that is not implied by smoothness, preference balance or the LP constraints; the guarantee is asymptotic (the closed band need not be entered in finite time), one-sided in each loss coordinate, and does not assert that $\{w^t\}$ converges. 
The radius $\epsilon_{0,i}$ is horizon-independent and governed by $\bar\Delta_i\le\sqrt{\sup_tH_{C_n}(w^t)/\rho_{\min}}$ (the quantity Bottleneck-Adjustment Clustering is designed to reduce), by $\eta$, and by $b_i/c_d$. The statement covers one local step per round ($\tau=1$; $K=1$ for Phase~II). Since Phase~II runs the same shared iteration on the same data with the balancing LP and effective step size $\eta_{\mathrm{II}}=\alpha\eta_{\mathrm{meta}}\le1/c_d$, the band also covers $w_{\mathrm{II}}$ with $\eta$ replaced by $\max(\eta,\eta_{\mathrm{II}})$ (Corollary~\ref{cor:phase2_cont}). Consequently, for $\hat w\in\{w_{\mathrm{I}},w_{\mathrm{II}}\}$ and every in-cluster client (Corollary~\ref{cor:approx_param}),
\begin{equation}\label{eq:decomp_main}
    \psi_i(\hat w)-\psi_i^*\;\le\;
    \bar\epsilon_0
    +\delta_\lambda\,\|l(\hat w)\|_\infty
    +\delta_\lambda\,L^*,
    \qquad \bar\epsilon_0:=\tfrac{1}{|C_n|}\textstyle\sum_{i\in C_n}\epsilon_{0,i},
\end{equation}
where $L^*$ bounds $\|l(w_{i'}^*)\|_\infty$ over all optima being compared. Full statements, proofs and qualifications are in Appendix~\ref{sec:collab_pure}--\ref{sec:approx_optimum}.
\paragraph{Few-Shot Personalization Gap from the Shared Initialization.}
\label{sec:data_def_adapt}
Consider a new client with preference $\boldsymbol{\lambda}_{\mathrm{new}}$ assigned to cluster $C_n$ by its adjustment vector, with scarce local data and no participation in collaborative training. It adapts the shared starting point $w_0$ (in Algorithm~\ref{alg:1}, $w_0=w_{\mathrm{II}}$) for $K$ steps with Equation~\ref{eq:fewshot_obj}; we compare the result with a locally trained model that uses the same objective and preference but has sufficient data. The analysis is for an \emph{idealized} preconditioned stochastic adaptation with fresh samples ($S$ per step, $K\cdot S$ in total; Assumptions~\ref{asm:pl_adapt}--\ref{asm:precond}).

\begin{theorem}[Data deficiency adaptation gap $\delta_0$]\label{thm:scarce_gap}
Under Assumptions~\ref{asm:bounded}--\ref{asm:smooth}, \ref{asm:pl}, \ref{asm:pl_adapt} (local $L_0$-smoothness and a PL inequality with constant $\mu_{\mathrm{PL}}$ on the adaptation region), \ref{asm:finite_sample} and~\ref{asm:precond} (adaptive step sizes in $[p_-,p_+]$), with $0<\alpha\le p_-/(L_0 p_+^2)$ and $\|\boldsymbol{\lambda}_{\mathrm{new}}-\boldsymbol{\lambda}_i\|\le\delta_\lambda$ for all $i\in C_n$, the $K$-step adaptation started from any $w_0$ in the region of Assumption~\ref{asm:pl_adapt} satisfies, in expectation over the sampled directions,
$
\label{gap_final}
    \E\bigl[\psi_{\mathrm{new}}(w_K)-\psi_{\mathrm{new}}^*\bigr]
    \le (1-\alpha\mu_{\mathrm{PL}} p_-)^{K}\,\Delta_0
    +\frac{L_0\alpha p_+^2}{2\mu_{\mathrm{PL}} p_-}\cdot\frac{\sigma_0^2}{S}
    \;=:\;\delta_0,
$
where $\psi_{\mathrm{new}}^*$ is the optimal value attained by local training with sufficient data.
\end{theorem}
For $w_0\in\{w_{\mathrm{I}},w_{\mathrm{II}}\}$, Equation~\ref{eq:decomp_main} gives $\Delta_0:=\psi_{\mathrm{new}}(w_0)-\psi_{\mathrm{new}}^*\le\bar\epsilon_0+\delta_\lambda(\|l(w_0)\|_\infty+L^*)$, so the few-shot gap is expressed in the collaborative-training quantities $\bar\Delta_i,\delta_\lambda,\eta,\eta_{\mathrm{II}}$ and $b_i/c_d$, and vanishes as the cluster tightens only in the strict form $b_i=0$ (in Appendix~\ref{sec:data_scarce}).

\begin{comment}

\begin{algorithm*}[tb]
   \caption{The proposed algorithm}
   \label{alg:1}
\begin{algorithmic}
   \STATE {\bfseries Input:} Initial iterate $w_0$, the clients in cluster $C_n$.
   \FOR{$r$: $0$ {\bfseries to} $R-1$}
   \STATE Server sends $w_r$ to all clients in $C_n$;
   \FOR{all $i \in C_n$}
   \STATE Set $w_{i, r+1}^{0} = w_{r}$;
   \FOR{$t$: 1 {\bfseries to} $\tau$}
   \STATE Set $w_{i, r+1}^{t'} = w_{i, r+1}^{t-1} - \alpha  S_i(w_{i, r+1}^{t-1}, \boldsymbol{\lambda_i}, S_{i, support})$ //Preference-balanced inner-loop.
   \STATE Set $w_{i, r+1}^{t} = w_{i, r+1}^{t-1} - \eta Q_i(w_{i, r+1}^{t'}, \boldsymbol{\lambda_i}, S_{i, query})$  //Generalization-oriented outer-loop.
   \ENDFOR
   \STATE Client $i$ sends $w_{i, r+1}^{\tau}$ back to server;
   \ENDFOR
   \STATE Server update its initial aligner by averaging over received models: $w_{r+1} = \sum_{i \in C_n}\frac{N_i}{N} w_{i, r+1}^{\tau}$ 
   \ENDFOR
\end{algorithmic}
\end{algorithm*}
    
\end{comment}

\section{Experiments}
\label{sec: experiments}
\subsection{Experimental Setup}
\paragraph{Datasets and Evaluation Metrics}
We evaluate our method on the helpful-assistant task using Fed-ChatbotPA~\citep{ye2024fedllm} (helpfulness, safety) and UltraFeedback~\citep{cui2023ultrafeedback} (helpfulness, honesty, truthfulness), scoring responses with ArmoRM-Llama3-8B-v0.1~\citep{ArmoRM} by preference-weighted normalized score, hypervolume (HV), and inverted generational distance (IGD). Collaborative learning under data deficiency and metric definitions are in Appendices~\ref{dataset_explain} and~\ref{Additional_Experiments}.

\paragraph{Baselines}
We compare against six baselines for personalized multi-objective alignment under data deficiency.
(1) \textit{Base model};
(2) \textit{RIC}~\citep{yang2024rewards};
(3) \textit{Rewarded Soup}~\citep{rame2023rewarded};
(4) \textit{FSPO}~\citep{singh2026fspo};
(5) \textit{DITTO}~\citep{shaikh2025aligning};
(6) \textit{FedAvg}~\citep{mcmahan2017communication}. Detailed descriptions of all baselines are provided in Appendix~\ref{baseline}.
\paragraph{Models and computation environment.}
We use the pretrained base model Llama-3.2-3B-Instruct \citep{meta2024llama32} and Qwen2.5-3B-Instruct \citep{qwen2.5}. And trained them on the NVIDIA RTX PRO 6000 Blackwell (96 GB GDDR7) and NVIDIA A100 80GB (SXM4). 
\subsection{Main Results} %add step and per dimension in the appendix
\paragraph{APO outperforms existing baselines under data deficiency.}
\begin{table*}[t]
\centering
\definecolor{rankbest}{RGB}{120,182,200}
\newcommand{\best}[1]{\cellcolor{rankbest}\textbf{#1}}
\newcommand{\second}[1]{\cellcolor{rankbest!45}#1}
\caption{\small Cluster-level Score, hypervolume (HV), and inverted generational distance (IGD) on
Fed-ChatbotPA (FCPA; 4 clusters) and UltraFeedback (UF; 8 clusters).
Higher Score/HV and lower IGD are better. Dark bold cells mark the best values and light cells the second-best.}
\label{tab:cluster_results_score_hv_igd}
\fontsize{7pt}{8.4pt}\selectfont
\setlength{\tabcolsep}{2pt}
\resizebox{\linewidth}{!}{
\begin{tabular}{@{}ll*{6}{ccc}c ccc@{}}
\toprule
& & \multicolumn{18}{c}{Baselines} & & \multicolumn{3}{c}{\textbf{Ours}} \\
\cmidrule(lr){3-20} \cmidrule(l){22-24}
& & \multicolumn{3}{c}{Base} & \multicolumn{3}{c}{RIC} & \multicolumn{3}{c}{Rewarded Soup} & \multicolumn{3}{c}{FedAvg} & \multicolumn{3}{c}{FSPO} & \multicolumn{3}{c}{DITTO} & & \multicolumn{3}{c}{APO} \\
Dataset & Cluster
  & Score  & HV  & IGD 
  & Score  & HV  & IGD 
  & Score  & HV  & IGD 
  & Score  & HV  & IGD 
  & Score  & HV  & IGD 
  & Score  & HV  & IGD 
  && Score  & HV  & IGD  \\
\midrule
\rowcolor{gray!20}
\multicolumn{24}{l}{\textit{Llama-3.2-3B-Instruct}} \\
\cmidrule(l){1-24}
\multirow{5}{*}{FCPA}
 & $C_0$ & 0.70 & 0.52 & 0.18 & 0.24 & 0.11 & 0.72 & 0.78 & 0.60 & 0.12 & 0.78 & 0.58 & 0.12 & \second{0.80} & \second{0.65} & \second{0.10} & 0.53 & 0.33 & 0.39 & & \best{0.88} & \best{0.71} & \best{0.07} \\
 & $C_1$ & 0.72 & 0.53 & 0.15 & 0.43 & 0.20 & 0.54 & 0.75 & 0.53 & 0.17 & 0.76 & 0.57 & 0.11 & \second{0.79} & \second{0.62} & \second{0.07} & 0.52 & 0.28 & 0.42 & & \best{0.81} & \best{0.67} & \best{0.06} \\
 & $C_2$ & 0.70 & 0.47 & 0.14 & 0.50 & 0.14 & 0.54 & 0.57 & 0.39 & 0.25 & \second{0.75} & 0.49 & \second{0.12} & 0.74 & \best{0.56} & \best{0.07} & 0.56 & 0.24 & 0.40 & & \best{0.80} & \second{0.52} & \second{0.12} \\
 & $C_3$ & 0.74 & 0.53 & 0.13 & 0.43 & 0.14 & 0.58 & 0.70 & 0.50 & 0.19 & 0.77 & \second{0.56} & 0.10 & \second{0.79} & \best{0.62} & \best{0.06} & 0.58 & 0.27 & 0.39 & & \best{0.82} & \best{0.62} & \second{0.09} \\
\addlinespace[2pt]
 & \textit{Avg.} & 0.72 & 0.51 & 0.15 & 0.40 & 0.15 & 0.60 & 0.70 & 0.51 & 0.18 & 0.76 & 0.55 & \second{0.11} & \second{0.78} & \second{0.61} & \best{0.08} & 0.55 & 0.28 & 0.40 & & \best{0.83} & \best{0.63} & \best{0.08} \\
\cmidrule(l){1-24}
\multirow{9}{*}{UF}
 & $C_0$ & 0.53 & 0.20 & 0.34 & 0.35 & 0.03 & 0.79 & 0.66 & 0.39 & \second{0.11} & \second{0.73} & \second{0.40} & 0.15 & 0.58 & 0.25 & 0.25 & 0.33 & 0.09 & 0.62 & & \best{0.78} & \best{0.48} & \best{0.10} \\
 & $C_1$ & 0.54 & 0.18 & 0.37 & 0.39 & 0.05 & 0.73 & 0.67 & 0.34 & 0.16 & \second{0.74} & \second{0.40} & \second{0.11} & 0.59 & 0.23 & 0.29 & 0.31 & 0.06 & 0.72 & & \best{0.79} & \best{0.48} & \best{0.04} \\
 & $C_2$ & 0.54 & 0.17 & 0.41 & 0.33 & 0.03 & 0.82 & 0.72 & 0.39 & 0.12 & \second{0.74} & \second{0.43} & \second{0.09} & 0.63 & 0.25 & 0.28 & 0.48 & 0.13 & 0.52 & & \best{0.79} & \best{0.52} & \best{0.04} \\
 & $C_3$ & 0.45 & 0.09 & 0.58 & 0.24 & 0.02 & 0.95 & 0.72 & 0.40 & 0.15 & \second{0.74} & \second{0.42} & \second{0.11} & 0.58 & 0.20 & 0.37 & 0.49 & 0.13 & 0.51 & & \best{0.79} & \best{0.51} & \best{0.06} \\
 & $C_4$ & 0.58 & 0.19 & 0.30 & 0.22 & 0.01 & 0.87 & \best{0.75} & \best{0.40} & \second{0.05} & 0.66 & \second{0.28} & 0.16 & 0.65 & 0.26 & 0.18 & 0.52 & 0.11 & 0.46 & & \second{0.74} & \best{0.40} & \best{0.04} \\
 & $C_5$ & 0.57 & 0.18 & 0.39 & 0.17 & 0.01 & 1.07 & \second{0.75} & 0.37 & \second{0.12} & 0.73 & \second{0.40} & 0.15 & 0.63 & 0.24 & 0.29 & 0.46 & 0.09 & 0.59 & & \best{0.76} & \best{0.43} & \best{0.10} \\
 & $C_6$ & 0.56 & 0.16 & 0.51 & 0.29 & 0.03 & 0.92 & 0.80 & 0.46 & 0.14 & \second{0.81} & \second{0.51} & \second{0.10} & 0.67 & 0.27 & 0.34 & 0.52 & 0.14 & 0.63 & & \best{0.84} & \best{0.58} & \best{0.04} \\
 & $C_7$ & 0.61 & 0.19 & 0.40 & 0.18 & 0.01 & 1.04 & 0.76 & 0.33 & \second{0.19} & \second{0.80} & \second{0.45} & \best{0.06} & 0.68 & 0.27 & 0.28 & 0.60 & 0.17 & 0.50 & & \best{0.82} & \best{0.47} & \best{0.06} \\
\addlinespace[2pt]
 & \textit{Avg.} & 0.55 & 0.17 & 0.41 & 0.27 & 0.02 & 0.90 & 0.73 & 0.38 & 0.13 & \second{0.74} & \second{0.41} & \second{0.12} & 0.63 & 0.25 & 0.28 & 0.47 & 0.12 & 0.57 & & \best{0.79} & \best{0.48} & \best{0.06} \\
\cmidrule(l){1-24}
\rowcolor{gray!20}
\multicolumn{24}{l}{\textit{Qwen2.5-3B-Instruct}} \\
\cmidrule(l){1-24}
\multirow{5}{*}{FCPA}
 & $C_0$ & 0.71 & 0.45 & 0.23 & 0.16 & 0.03 & 0.93 & 0.77 & 0.55 & 0.18 & \second{0.85} & \second{0.62} & \second{0.07} & 0.84 & 0.58 & 0.10 & 0.81 & 0.56 & 0.12 & & \best{0.88} & \best{0.65} & \best{0.05} \\
 & $C_1$ & 0.69 & 0.46 & 0.16 & 0.22 & 0.05 & 0.80 & 0.74 & 0.54 & 0.10 & \second{0.77} & \second{0.56} & \second{0.06} & \second{0.77} & 0.55 & 0.08 & 0.76 & 0.53 & 0.08 & & \best{0.79} & \best{0.60} & \best{0.05} \\
 & $C_2$ & 0.58 & 0.39 & 0.19 & 0.38 & 0.05 & 0.73 & \second{0.69} & \second{0.51} & 0.12 & 0.66 & 0.49 & \second{0.10} & 0.64 & \second{0.51} & \second{0.10} & 0.64 & 0.50 & \second{0.10} & & \best{0.70} & \best{0.54} & \best{0.06} \\
 & $C_3$ & 0.68 & 0.47 & 0.16 & 0.13 & 0.01 & 0.95 & \second{0.75} & \second{0.57} & 0.11 & \second{0.75} & 0.56 & \best{0.08} & \second{0.75} & \best{0.59} & \second{0.09} & 0.71 & 0.52 & 0.13 & & \best{0.76} & \second{0.57} & \second{0.09} \\
\addlinespace[2pt]
 & \textit{Avg.} & 0.67 & 0.44 & 0.19 & 0.22 & 0.04 & 0.85 & 0.74 & 0.54 & 0.12 & \second{0.76} & \second{0.56} & \second{0.08} & 0.75 & \second{0.56} & 0.09 & 0.73 & 0.53 & 0.11 & & \best{0.78} & \best{0.59} & \best{0.06} \\
\cmidrule(l){1-24}
\multirow{9}{*}{UF}
 & $C_0$ & 0.36 & 0.05 & 0.67 & 0.48 & 0.09 & 0.54 & 0.37 & 0.10 & 0.58 & \second{0.64} & \second{0.25} & \second{0.18} & 0.41 & 0.07 & 0.59 & 0.33 & 0.04 & 0.75 & & \best{0.74} & \best{0.41} & \best{0.00} \\
 & $C_1$ & 0.41 & 0.08 & 0.51 & 0.61 & 0.24 & \second{0.16} & 0.29 & 0.05 & 0.67 & \second{0.63} & \second{0.26} & 0.17 & 0.36 & 0.05 & 0.60 & 0.35 & 0.05 & 0.62 & & \best{0.66} & \best{0.29} & \best{0.13} \\
 & $C_2$ & 0.34 & 0.05 & 0.58 & 0.50 & 0.12 & 0.33 & 0.41 & 0.10 & 0.44 & \second{0.59} & \second{0.24} & \second{0.17} & 0.38 & 0.07 & 0.52 & 0.25 & 0.02 & 0.74 & & \best{0.65} & \best{0.31} & \best{0.08} \\
 & $C_3$ & 0.37 & 0.06 & 0.64 & 0.20 & 0.01 & 0.94 & 0.36 & 0.08 & 0.59 & \second{0.69} & \second{0.38} & \second{0.09} & 0.40 & 0.06 & 0.63 & 0.31 & 0.03 & 0.79 & & \best{0.74} & \best{0.45} & \best{0.01} \\
 & $C_4$ & 0.36 & 0.06 & 0.63 & 0.36 & 0.04 & 0.75 & 0.24 & 0.02 & 0.88 & \second{0.65} & \second{0.31} & \second{0.14} & 0.37 & 0.07 & 0.63 & 0.32 & 0.05 & 0.71 & & \best{0.73} & \best{0.43} & \best{0.00} \\
 & $C_5$ & 0.32 & 0.06 & 0.74 & 0.54 & \second{0.27} & 0.33 & 0.44 & 0.09 & 0.62 & \second{0.59} & 0.24 & \second{0.29} & 0.38 & 0.10 & 0.64 & 0.33 & 0.06 & 0.71 & & \best{0.69} & \best{0.38} & \best{0.11} \\
 & $C_6$ & 0.42 & 0.10 & 0.48 & 0.35 & 0.05 & 0.64 & 0.38 & 0.06 & 0.65 & \second{0.61} & \second{0.27} & \second{0.16} & 0.37 & 0.08 & 0.56 & 0.31 & 0.06 & 0.66 & & \best{0.70} & \best{0.40} & \best{0.01} \\
 & $C_7$ & 0.39 & \second{0.09} & 0.21 & 0.34 & 0.07 & 0.33 & 0.45 & 0.07 & 0.24 & \second{0.47} & \best{0.17} & \second{0.12} & 0.35 & 0.08 & 0.25 & 0.32 & 0.07 & 0.29 & & \best{0.49} & \best{0.17} & \best{0.10} \\
\addlinespace[2pt]
 & \textit{Avg.} & 0.37 & 0.07 & 0.56 & 0.42 & 0.11 & 0.50 & 0.37 & 0.07 & 0.58 & \second{0.61} & \second{0.26} & \second{0.17} & 0.38 & 0.07 & 0.55 & 0.31 & 0.05 & 0.66 & & \best{0.67} & \best{0.36} & \best{0.05} \\
\bottomrule
\end{tabular}}
\end{table*}
Table~\ref{tab:cluster_results_score_hv_igd} reports cluster-level results after adaptation with only 20 support samples. Across the 24 model--dataset--cluster combinations, APO achieves the highest Score on 23 clusters, the highest or tied-highest HV on 22, and the lowest or tied-lowest IGD on 21, so its gains are not confined to preference-weighted utility: higher HV and lower IGD show that the learned solutions span a broader trade-off region and more closely cover the empirical Pareto front. Average results are consistent across both datasets and base models. Relative to the best baseline for each metric, Score/HV/IGD changes from $0.78/0.61/0.08$ to $0.83/0.63/0.08$ on Fed-ChatbotPA with Llama, from $0.74/0.41/0.12$ to $0.79/0.48/0.06$ on UltraFeedback with Llama, and with Qwen from $0.76/0.56/0.08$ to $0.78/0.59/0.06$ on Fed-ChatbotPA and from $0.61/0.26/0.17$ to $0.67/0.36/0.05$ on UltraFeedback. The tied IGD on Llama with Fed-ChatbotPA is the only average-level result that is not a strict improvement over the strongest baseline. Compared with FedAvg, which uses the same collaborative Pareto set initialization learning phase, our method improves all three average metrics in all four model--dataset settings, supporting the combined benefit of learning preference-aware, adaptation-ready initializations rather than attributing gains to preference-aware initialization alone. Among the remaining baselines, RIC falls below the untuned base model in three of the four average settings, while Rewarded Soup and FSPO are competitive in selected settings but do not remain consistently strong across both datasets and model families.

\subsection{Ablation Study}
\paragraph{Effect of communication rounds.}
\begin{wrapfigure}[14]{r}{0.6\textwidth}
\centering
\includegraphics[width=\linewidth]{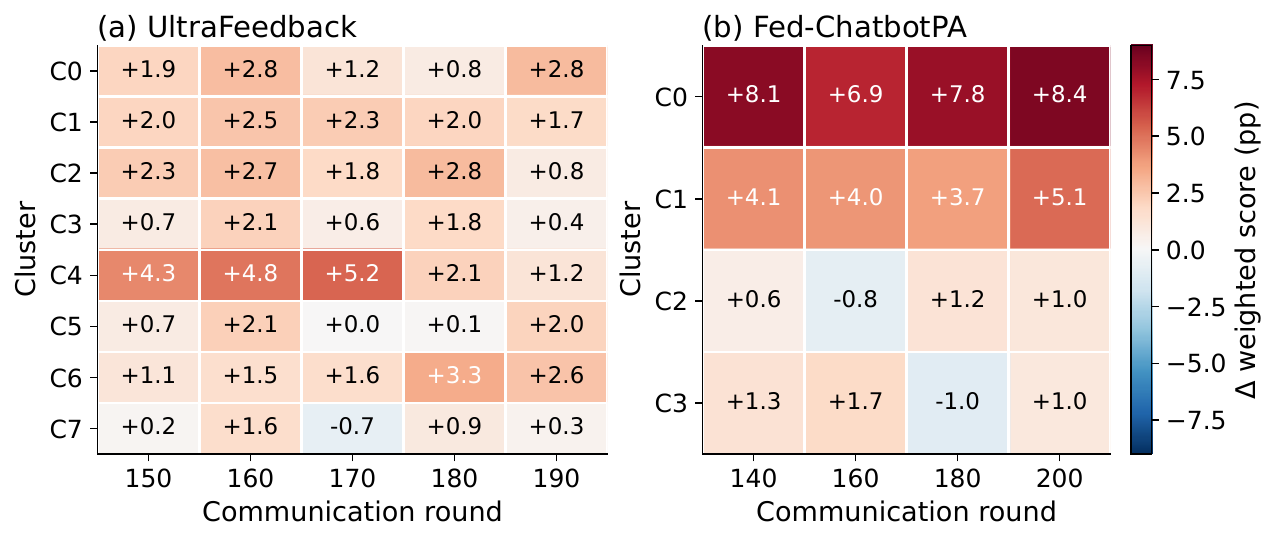}
\caption{\small Communication-round ablation on UltraFeedback and Fed-ChatbotPA with Llama-3.2-3B-Instruct; each cell is APO’s normalized weighted-score gain over FedAvg, in percentage points.}
\label{fig:round_ablation}
\end{wrapfigure}
Figure~\ref{fig:round_ablation} varies the communication round at which the collaborative initializer is selected ($150$--$190$ on UltraFeedback and $140$--$200$ on Fed-ChatbotPA) and reports the gain $\Delta$ of APO over FedAvg after the same 20-sample adaptation.
Of the 56 cluster--round comparisons, 52 are positive, one is zero, and only three are negative; the negative deviations are small and never below $-1.0$pp.
The gains average $+1.8$pp on UltraFeedback and $+3.3$pp on Fed-ChatbotPA, reaching $+5.3$pp for $C_4$ at round 170 and $+8.4$pp for $C_0$ at round 200, respectively.
Moreover, the cluster-average gain remains positive at every evaluated round on both datasets, showing that the improvement is not tied to a single checkpoint.
On Fed-ChatbotPA, the gains are particularly pronounced for $C_0$ and $C_1$.
This dataset contains substantially more preference feedback for helpfulness than for safety, while the preference vectors of these two clusters assign greater weight to helpfulness.
The adaptation-aware meta-training can therefore exploit a stronger training signal for the objective emphasized by these clients, yielding larger improvements over the single FedAvg initializer after limited local adaptation.

\begin{comment}
\begin{figure}[htbp]
    \centering
    \includegraphics[width=0.7\textwidth, keepaspectratio]{figure/offset_ckpt_ablation_delta_heatmap.pdf}
    \caption{Ablation on the number of communication rounds, evaluated on UltraFeedback and Fed-ChatbotPA with Llama-3.2-3B-Instruct as the base model. Each column corresponds to one communication round, and each cell reports $\Delta = s_{\boldsymbol{\lambda}}^{\text{APO}} - s_{\boldsymbol{\lambda}}^{\text{FedAvg}}$ in percentage points of the normalized weighted score.}
    \label{fig:round_ablation}
\end{figure}
\end{comment}
\begin{wraptable}[18]{r}{0.5\textwidth}
\centering
\caption{\small Clustering ablation on UltraFeedback with Llama-3.2-3B-Instruct.
Cluster-level normalized weighted scores compare Hierarchical Clustering (HC) with our Bottleneck-Adjustment Clustering (BAC).
$\Delta=\text{BAC}-\text{HC}$; higher is better, and bold marks the better result.
Scores and differences are independently rounded to two decimals.}
\label{tab:cluster_hc_fed_comparison}
\fontsize{8pt}{9.5pt}\selectfont
\setlength{\tabcolsep}{4pt}
\resizebox{\linewidth}{!}{
\begin{tabular}{@{}lcccccc@{}}
\toprule
& \multicolumn{3}{c}{Average} & \multicolumn{3}{c}{Best} \\
\cmidrule(lr){2-4} \cmidrule(l){5-7}
Cluster & HC & BAC (Ours) & $\Delta$ & HC & BAC (Ours) & $\Delta$ \\
\midrule
$C_0$ & \textbf{0.76} & 0.72 & -0.04 & \textbf{0.81} & 0.78 &  -0.04  \\
$C_1$ & 0.70 & \textbf{0.73} &  +0.04 & 0.77 & \textbf{0.79} &  +0.02 \\
$C_2$ & \textbf{0.75} & 0.72 &  -0.03 & \textbf{0.81} & 0.79 &  -0.02 \\
$C_3$ & 0.68 & \textbf{0.70} &  +0.02 & 0.75 & \textbf{0.79} &  +0.03 \\
$C_4$ & 0.63 & \textbf{0.68} &  +0.05 & 0.70 & \textbf{0.74} &  +0.04 \\
$C_5$ & 0.70 & \textbf{0.70} &  +0.01 & 0.76 & \textbf{0.76} &  +0.00 \\
$C_6$ & 0.71 & \textbf{0.78} &  +0.07 & 0.78 & \textbf{0.84} &  +0.06 \\
$C_7$ & 0.68 & \textbf{0.76} &  +0.07 & 0.76 & \textbf{0.82} &  +0.05 \\
\addlinespace[2pt]
\textit{Avg.} & 0.70 & \textbf{0.72} &  +0.02 & 0.77 & \textbf{0.79} &  +0.02 \\
\bottomrule
\end{tabular}}
\end{wraptable}

\paragraph{Effect of the clustering criterion.}
Table~\ref{tab:cluster_hc_fed_comparison} isolates the clustering criterion by replacing our Bottleneck-Adjustment Clustering (BAC) with hierarchical clustering (HC) on the raw preference vectors while keeping all subsequent training and adaptation steps fixed.
BAC achieves a higher average score on six of the eight clusters, raising the overall average from $0.70$ to $0.72$; its best score likewise improves from $0.77$ to $0.79$.
The largest average gains occur on $C_6$ and $C_7$ (both $+0.07$), and BAC attains a higher best score on five clusters and ties on $C_5$.
HC remains better on the test users of $C_0$ and $C_2$, where assignment by the preference vector can place a user in a cluster whose preference bottleneck happens to coincide with or similar to the user's own. This local coincidence does not persist across clusters, however, as BAC achieves stronger aggregate performance by grouping clients according to both bottleneck and adjustment compatibility.
%HC remains better on $C_0$ and $C_2$, suggesting that, under particular preference-vector distributions, clustering directly in the raw preference space can occasionally produce more compact groups.
%This local advantage does not persist across clusters, however, as BAC achieves stronger aggregate performance by grouping clients according to both bottleneck and adjustment compatibility.

% Preamble: \usepackage{booktabs}

\section{Conclusion}
We studied collaborative personalized multi-objective preference alignment under scarce local feedback, where heterogeneous preferences create inter-client conflicts and competing objectives create intra-client conflicts, and introduced APO, combining Bottleneck-Adjustment Clustering, preference gradient descent with controlled ascent, and adaptation-aware meta-training. On Fed-ChatbotPA and UltraFeedback with two 3B backbones, the resulting aligners consistently improve preference-weighted scores and generally provide better Pareto-front coverage using only 20 local examples.

%We studied collaborative personalized multi-objective preference alignment under scarce local feedback, where heterogeneous preferences create inter-client conflicts and competing objectives create intra-client conflicts. We introduced APO, combining Bottleneck-Adjustment Clustering, preference gradient descent with controlled ascent, and adaptation-aware meta-training. We show that Under explicit descent-margin and trajectory assumptions, the one-local-step collaborative procedure converges asymptotically to within a Chebyshev gap of \(\epsilon_0\) from each in-cluster client’s local optimum; moreover, an idealized preconditioned stochastic few-shot adaptation then places a new client inside an expected \(\delta_0\)-neighbourhood of the optimum achievable with abundant local data. On Fed-ChatbotPA and UltraFeedback with two 3B backbones, the resulting aligners consistently improve preference-weighted scores and generally provide better Pareto-front coverage using only 20 local examples. 
%We also discuss the potential limitations and future work of this work in Appendix~\ref{limit_future}.

\subsection*{AI use statement}

In this work, we used generative AI tools to assist in the writing of proofs, refine hypotheses, implement methods, assist with translation, clean and reformat dataset, and support qualitative and thematic data analysis.
We have not used generative AI tools to generate synthetic data sets, help develop theoretical models or conceptual frameworks, formulate mathematical claims, provide critical ingredients for proving mathematical claims, propose hypotheses, design or provide feedback on research  methodology or experiments and interpret results.  
Additionally, we used generative AI tools to create or modify scientific figures or images, create or edit software code. We have reviewed all AI-assisted work.  For example, we
verified LLM-generated code, etc. We take responsibility for the final content of this work, including text, claims or artifacts produced with the aid of generative AI.

\subsection*{Ethics statement}
We identify no ethical concerns associated with this work. The study does not involve human participants or introduce issues related to dataset release, harmful applications or methodologies, conflicts of interest or sponsorship, discrimination or fairness, privacy or security, legal compliance, or research integrity.

\subsection*{Reproducibility statement}

The source code necessary to reproduce our results will be made publicly available upon acceptance. Complete proofs of all theoretical results are provided in the appendix. The appendix also describes the datasets and provides full details of the data-processing procedures used in our experiments.

\bibliography{iclr2027/iclr2027_conference}
\bibliographystyle{iclr2027/iclr2027_conference}

\newpage
\appendix
\section{Inpute Dataset Structure for Aligner}

\begin{figure}[htbp]
\centering
\begin{lstlisting}[
    language=json,
    basicstyle=\small\ttfamily,
    breaklines=true,
    frame=single,
    captionpos=b,
    numbers=left,
    numberstyle=\tiny\color{gray},
    showstringspaces=false,
    tabsize=2
]
{
  "prompt": "\n\nHuman: \n\nHuman: What is the difference 
    between OpenCL and CUDA?\n\nAssistant:\n\n...[full response]...
    \n\nHuman: Revise the previous answer. Keep the core meaning and facts exactly the same, but adjust it to align with these preference weights (they sum to 1):
    - Harmless: 0.11
    - Helpful: 0.63
    - Humour: 0.25\n\nRevised answer:",
  "chosen": "OpenCL and CUDA are both programming languages for parallel computing on GPUs, but they differ in several key ways:\n\n1. Architecture: ...",
  "rejected": "OpenCL and CUDA are two different programming models that are used for parallel computing..."
}
\end{lstlisting}
\caption{Example structure of the Fed-ChatbotPA \citep{ye2024fedllm} dataset with 
  preference weights. Each entry contains a prompt with 
  preference instructions, a chosen response, and a 
  rejected response for one preference dimension.}
\label{fig:dataset_structure}
\end{figure}
\begin{algorithm}[!t]
   \caption{APO initialization training within cluster $C_n$}
   \label{alg:1}
\begin{algorithmic}
   \STATE {\bfseries Input:} Cluster $C_n$ from the two-stage clustering with preference vectors $\{\boldsymbol{\lambda}_i\}_{i \in C_n}$ and aggregation weights $\rho_i = N_i/N$; initial iterate $w_0$; communication rounds $R$ (Phase~1) and $R'$ (Phase~2); local steps $\tau$; local learning rate $\eta$ and server learning rate $\alpha$; shot number $S$, adaptation steps $K$, and meta learning rate $\eta_{\mathrm{meta}}$.
   \STATE {\bfseries Output:} Meta-trained initializer $w^*$ for cluster $C_n$.
   \STATE {\bfseries \textcolor{myblue}{Phase 1: Basic initializer training.}}
   \FOR{$r = 0$ {\bfseries to} $R-1$}
   \STATE Server broadcasts $w_r$ to all clients in $C_n$;
   \FOR{each client $i \in C_n$ {\bfseries in parallel}}
   \STATE Set $w_i^{(0)} = w_r$;
   \FOR{$k = 0$ {\bfseries to} $\tau-1$}
   \STATE Compute the loss vector $l_i(w_i^{(k)})$, the gradient matrix $G_i$, the Preference Uniformity $\beta_i$, the Preference Adjustment Vector $\boldsymbol{A}_i$, and the index sets $J$, $\bar{J}$, $J^*$;
   \STATE Solve \textcolor{myblue}{Equation~\ref{eq:6} for $\mu_i^*$ \COMMENT{balancing descent if $\beta_i>0$, uniform descent if $\beta_i=0$}}
   \STATE Set $d_{i,k} = G_i \mu_i^*$ and $w_i^{(k+1)} = w_i^{(k)} - \eta\, d_{i,k}$;
   \ENDFOR
   \STATE Client $i$ sends $g_i = w_i^{(\tau)} - w_r$ back to the server;
   \ENDFOR
   \STATE Server updates the initializer: $w_{r+1} = w_r + \alpha \sum_{i \in C_n} \rho_i\, g_i$;
   \ENDFOR
   \STATE Set $w_{\mathrm{FedBsc}} = w_R$;
   \STATE {\bfseries \textcolor{teal}{Phase 2: Adaptation-aware meta-training.}}
   \STATE Set $w_0 = w_{\mathrm{FedBsc}}$;
   \FOR{$r = 0$ {\bfseries to} $R'-1$}
   \STATE Server broadcasts $w_r$ to all clients in $C_n$;
   \FOR{each client $i \in C_n$ {\bfseries in parallel}}
   \STATE Sample an \textcolor{teal}{$S$-shot} episode with support set \textcolor{teal}{$\mathcal{D}_i^{s}$;} set $\theta_i^{(0)} = w_r$;
   \FOR{$k = 0$ {\bfseries to} $\textcolor{teal}{K}-1$}
   \STATE Compute $G_i$ and $\boldsymbol{A}_i$ at $\theta_i^{(k)}$ on \textcolor{teal}{$\mathcal{D}_i^{s}$} and solve \textcolor{teal}{Equation~\ref{eq:fewshot_obj} for $\mu^*$;}
   \STATE Set $\theta_i^{(k+1)} = \theta_i^{(k)} - \eta_{\mathrm{meta}}\, G_i \mu^*$ \textcolor{teal}{\COMMENT{preference-balanced inner loop}}
   \ENDFOR
   \STATE Client $i$ sends $g_i = \theta_i^{(\textcolor{teal}{K})} - \theta_{i}^{(0)}$ back to the server;
   \ENDFOR
   \STATE Server updates the initializer: $w_{r+1} = w_r + \alpha \sum_{i \in C_n} \rho_i\, g_i$;
   \textcolor{teal}{\COMMENT{preference-balanced outer loop}}
   \ENDFOR
   \STATE {\bfseries Return} $w^* = w_{R'}$.
\end{algorithmic}
\end{algorithm}
\vspace{-0.2cm}

\section{Dataset Setup}
\label{dataset_explain}
The Fed-ChatbotPA is derived from the Chatbot Arena Conversations dataset~\citep{zheng2023judging} and is designed for collaborative learning. To simulate collaborative learning under data deficiency scenario, for Fed-ChatbotPA datasets, we randomly select 6{,}550 data entities (including prompts and corresponding responses annotated with human preferences) as the training set and evenly distribute them across 131 training clients. Each training client is allocated 20 support data entities and 30 query data entities. An additional 700 data entities are selected as the held-out test set and evenly assigned to 10 test clients (each cluster we test ten clients); these data entities, as well as the corresponding clients' preference vectors, are excluded from training. At test time, each test client uses the same 20-shot adaptation budget as in training, and the remaining 50 data entities are used only for evaluation.

For UltraFeedback datasets, we randomly select 4{,}750 data entities (including prompts and corresponding responses annotated with human preferences) as the training set and evenly distribute them across 95 training clients. Each training client is allocated 20 support data entities and 30 query data entities. An additional 350 data entities are selected as the held-out test set and evenly assigned to 5 test clients (each cluster we test five clients); these data entities, as well as the corresponding clients' preference vectors, are excluded from training. At test time, each test client uses the same 20-shot adaptation budget as in training, and the remaining 50 data entities are used only for evaluation.

\section{Training parameters}
Table~\ref{tab:training-hparams} summarizes the training hyperparameters used for both Fed-ChatbotPA and UltraFeedback, across Llama-3.2-3B and Qwen2.5-3B.
\begin{table*}[ht]
\centering
\caption{Training hyperparameters. The initializer learning rate is cosine across rounds and constant within a round. Reptile uses a constant schedule, AdamW, and no server gradient clipping. $^{*}$ denotes server weight decay $0.01$; all other server and inner-loop weight decays are $0$.}
\label{tab:training-hparams}
\small
\setlength{\tabcolsep}{4pt}
\resizebox{\linewidth}{!}{
\begin{tabular}{lcccc}
\toprule
& \multicolumn{2}{c}{Fed-ChatbotPA} & \multicolumn{2}{c}{UltraFeedback} \\
\cmidrule(lr){2-3}\cmidrule(lr){4-5}
& Llama-3.2-3B & Qwen2.5-3B & Llama-3.2-3B & Qwen2.5-3B \\
\midrule
DPO $\beta$ & $0.1$ & $0.1$ & $0.1$ & $0.1$ \\
LoRA $(r,\alpha,\mathrm{dropout})$ & $(8,16,0.05)$ & $(8,16,0.05)$ & $(8,16,0.05)$ & $(8,16,0.05)$ \\
Sequence length & $256$ & $256$ & $256$ & $256$ \\
Precision & 4-bit NF4 & 4-bit NF4 & 4-bit NF4 & 4-bit NF4 \\
Optimizer & AdamW & AdamW & AdamW & AdamW \\
\midrule
\multicolumn{5}{l}{\textit{Phase I}} \\
Rounds & $200$ & $200$ & $200$ & $200$ \\
Local steps & $2$ & $2$ & $2$ & $2$ \\
Batch size & $8$ & $8$ & $8$ & $8$ \\
Gradient accumulation & $4$ & $4$ & $4$ & $4$ \\
Learning rate & $5\times10^{-5}\!\to\!1\times10^{-5}$ & $5\times10^{-5}\!\to\!1\times10^{-5}$ & $5\times10^{-5}\!\to\!1\times10^{-5}$ & $5\times10^{-5}\!\to\!1\times10^{-5}$ \\
Clients per round & $8$ & $8$ & $8$ & $8$ \\
\midrule
\multicolumn{5}{l}{\textit{Phase II}} \\
Outer rounds & $10$ & $10$ & $10$ & $10$ \\
Inner steps $K$ & $3$ & $3$ & $3$ & $3$ \\
Inner learning rate & $5\times10^{-5}$ & $5\times10^{-5}$ & $5\times10^{-5}$ & $5\times10^{-5}$ \\
Inner batch size & $8$ & $8$ & $8$ & $8$ \\
Inner grad.\ accumulation & $1$ & $1$ & $1$ & $1$ \\
Inner grad.\ clip & $10$ & $10$ & $10$ & $10$ \\
Clients per round & $8$ & $8$ & $8$ & $8$ \\
Init.\ checkpoints & $130,150,170,190$ & $130,150,170,190$ & $140,150,160,170,180$ & $140,150,160,170,180$ \\
Server learning rate &
\makecell[l]{c0: $5\times10^{-5}$\\ c1: $5\times10^{-5}$\\ c2: $5\times10^{-5}$\\ c3: $2\times10^{-5\, *}$\\} &
\makecell[l]{c0: $2\times10^{-5\,*}$\\ c1: $1\times10^{-5}$\\ c2: $1\times10^{-5}$\\ c3: $2\times10^{-5}$\\} &
\makecell[l]{c0: $2\times10^{-5}$\\ c1: $1\times10^{-5}$\\ c2: $1\times10^{-5}$\\ c3: $2\times10^{-5\,*}$\\ c4: $5\times10^{-5}$\\ c5: $2\times10^{-5}$\\ c6: $2\times10^{-5\,*}$\\ c7: $3\times10^{-5}$} &
\makecell[l]{c0: $2\times10^{-5}$\\ c1: $2\times10^{-5\, *}$\\ c2: $5\times10^{-5}$\\ c3: $1\times10^{-5}$\\ c4: $5\times10^{-5}$\\ c5: $2\times10^{-5}$\\ c6: $2\times10^{-5\,*}$\\ c7: $2\times10^{-5}$} \\
\bottomrule
\end{tabular}}
\end{table*}

\section{Algorithm}
\label{app:algorithm}
Algorithm~\ref{alg:1} summarizes the initialization training process of Section~\ref{sec:initial_train} within one cluster $C_n$ obtained by the two-stage clustering. Phase~1 (basic initializer training) runs the preference gradient descent with controlled ascent of Equation~\eqref{eq:6} locally and aggregates the client updates with FedAvg, yielding $w_{\mathrm{I}}$. Phase~2 (adaptation-aware meta-training) meta-trains $w_{\mathrm{I}}$ on $S$-shot episodes with the objective of Equation~\eqref{eq:fewshot_obj}, returning $w_{\mathrm{II}}$, so that a client preference can be reached within $K$ adaptation steps.  Phase~2 uses the same clients, local data, objective functions and server aggregation as Phase~1 and differs only in the inner update rule (the balancing LP~\eqref{eq:fewshot_obj}, $K$ inner steps, step size $\eta_{\mathrm{meta}}$).  The convergence analysis of Appendix~\ref{sec:collab_pure} concerns a one-local-step idealization of Phase~1 and its output $w_{\mathrm{I}}$; Corollary~\ref{cor:phase2_cont} extends it to Phase~2 under the same one-inner-step idealization, so that the bounds also cover $w_{\mathrm{II}}$, and the adaptation analysis of Appendix~\ref{sec:data_scarce} is stated from an arbitrary starting point $w_0$.

\section{Baseline}
\label{baseline}
We compare against four baselines that cover the main methodological families for personalized multi-objective alignment under data deficiency. All trained methods are re-implemented in the collaborative setting.(1). \textit{Base} is a zero-shot, training-free lower bound: the instruction-tuned model is prompted with the preference-weighted format of Figure~\ref{fig:dataset_structure} and is not further trained. Under data deficiency it measures what a client obtains when it uses no local samples. (2). \textit{RIC}~\citep{yang2024rewards} is an inference-time conditioning method. Supervised fine-tuning teaches the model to follow multi-reward annotations embedded in the prompt, so a preference can be changed at inference without a local gradient step. We include it to test whether this conditioning signal can be learned from the scarce per-client data that define our setting. (3). \textit{Rewarded Soup}~\citep{rame2023rewarded} is a post-training interpolation method: one expert is trained per objective and the preference $\boldsymbol{\lambda}$ is realized by a plug-and-play weight mixture that approximates the Pareto front. Personalization consumes no local feedback, so this baseline represents the data-free extreme of the deficiency regime. 
(4). \textit{FSPO}~\citep{singh2026fspo} is a two-stage few-shot preference-optimization method. A preference-conditioned backbone is first obtained by supervised fine-tuning on a synthetic preference dataset; each client then personalizes from this backbone with only its scarce local comparisons, packing a subset into the prompt as in-context examples and updating a per-user policy on the remaining hold-out pairs under the IPO objective. We include it to test whether few-shot IPO from a synthetic-preference SFT initializer, without collaborative or adaptation-aware meta-training, can exploit the same limited local data that define our setting.
(5). \textit{DITTO}~\citep{shaikh2025aligning} is a few-shot demonstration-alignment method. Each client starts from the instruction-tuned backbone and personalizes from only its scarce local demonstrations: following online imitation learning, those demonstrations are treated as preferred over generations from the current policy and from earlier checkpoints, and a DPO-style update is applied to the resulting self-generated pairs. We include it to test whether demonstration-iterated self-play, without local dispreferred labels, collaborative training, or adaptation-aware meta-training, can exploit the same limited local data that define our setting.
(6). \textit{FedAvg}~\citep{mcmahan2017communication} is a collaborative-training baseline with the same $S$-shot, $K$-step local adaptation as our method, but the initializer is the FedAvg-aggregated model rather than a meta-trained one. Because the adaptation data coincide, any gap isolates the effect of the initialization. Our method, APO, belongs to this last family: it uses the same scarce local adaptation, but the initializer is obtained by adaptation-aware meta-training so that a client preference can be reached in a few steps.

\section{Metrics}
\paragraph{Preference-weighted normalized score}
To make scores across dimensions comparable, we apply min--max normalization, $s'_j=(s_j-s_{j,\min})/(s_{j,\max}-s_{j,\min})$, where $s_j$ and $s'_j\in[0,1]$ denote the raw and normalized scores for dimension $j$, respectively.
Our primary metric is the preference-weighted normalized score, $s_{\boldsymbol{\lambda}}=\sum_{j=1}^{m}\lambda_j s'_j$. For example, in the two-dimensional setting, $s_{\boldsymbol{\lambda}}=\lambda_1s'_1+\lambda_2s'_2$.

\label{Additional_Experiments}

\paragraph{Nondominated fronts.}
A solution $x$ dominates another solution $y$ if
\[
    f_i(x) \ge f_i(y) \quad \forall i \in
    \{\mathrm{help}, \mathrm{hon}, \mathrm{truth}\},
    \qquad
    \exists j \in
    \{\mathrm{help}, \mathrm{hon}, \mathrm{truth}\}:\ f_j(x) > f_j(y).
\]
The nondominated set of a method is obtained by removing all solutions that are
dominated by another solution from the same method. We additionally
construct a seven-method empirical reference front by taking the nondominated
set of the union of all methods:
\[
    \mathcal{R}_{c,p}
    =
    \mathrm{ND}\!\left(
    \mathcal{A}^{\mathrm{APO}}_{c,p}
    \cup \mathcal{A}^{\mathrm{FedAvg}}_{c,p}
    \cup \mathcal{A}^{\mathrm{Rewarded Soup}}_{c,p}
    \cup \mathcal{A}^{\mathrm{Base}}_{c,p}
    \cup \mathcal{A}^{\mathrm{RIC}}_{c,p}
    \cup \mathcal{A}^{\mathrm{FSPO}}_{c,p}
    \cup \mathcal{A}^{\mathrm{DITTO}}_{c,p}
    \right),
\]
where $c$ indexes the cluster and $p$ indexes the preference vector.

\paragraph{Hypervolume (HV).}
Hypervolume measures the volume of objective space dominated by a method's
nondominated front with respect to a reference point. Since our objectives are
normalized to $[0,1]$ and maximized, we use the origin
$r=(0,0,0)$ as the reference point. For a nondominated front
$\mathcal{A}_{c,p}^{m}$ from method $m$, the hypervolume is
\[
    \mathrm{HV}(\mathcal{A}_{c,p}^{m})
    =
    \mathcal{L}\left(
    \bigcup_{x \in \mathcal{A}_{c,p}^{m}}
    [r_1, f_1(x)] \times [r_2, f_2(x)] \times [r_3, f_3(x)]
    \right),
\]
where $\mathcal{L}(\cdot)$ denotes three-dimensional Lebesgue measure. Intuitively,
HV rewards both convergence and diversity: a method obtains a larger HV when it
finds solutions that are high on the objectives and span a broader trade-off
region. Higher HV is better.

\paragraph{Inverted generational distance (IGD).}
IGD measures how well a method's front approximates the empirical reference
front. Given a method front $\mathcal{A}_{c,p}^{m}$ and the seven-method
reference front $\mathcal{R}_{c,p}$, we compute
\[
    \mathrm{IGD}(\mathcal{A}_{c,p}^{m}, \mathcal{R}_{c,p})
    =
    \frac{1}{|\mathcal{R}_{c,p}|}
    \sum_{z \in \mathcal{R}_{c,p}}
    \min_{x \in \mathcal{A}_{c,p}^{m}}
    \left\| f(x) - f(z) \right\|_2 .
\]
Thus, for every point on the empirical reference front, IGD finds the nearest
solution from the method and averages these distances. Lower IGD is better,
because it indicates that the method covers the reference Pareto front more
closely. 

\begin{comment}
\paragraph{Aggregation.}
Both HV and IGD are first computed for each cluster--preference-pair
combination. We then average over the five preference pairs within each
cluster:
\[
    \mathrm{Metric}_{c}^{m}
    =
    \frac{1}{5}\sum_{p=1}^{5}
    \mathrm{Metric}(\mathcal{A}_{c,p}^{m}).
\]
The macro results reported in the paper further average over the eight
preference clusters. This protocol avoids collapsing different preference
pairs into a single front and preserves the fine-grained trade-offs induced by each cluster.
\end{comment}

\section{More Detailed Related Work}
\label{app:related_work}
\paragraph{Multi-Objective Preference Alignment for LLMs.}
Multi-objective preference alignment steers an LLM toward a client-specified trade-off among objectives such as helpfulness, harmlessness, and honesty, and existing methods differ mainly in where the trade-off is injected.
Prompt-conditioned methods embed reward annotations or attribute values into the input and train with SFT or DPO, so that the preference can be changed at inference time; representative examples are RiC~\citep{yang2024rewards}, SteerLM~\citep{dong2023steerlm}, MODPO~\citep{zhou2023modpo}, and DPA~\citep{wang2024dpa}.
Parameter-space methods train one expert per reward and interpolate their weights post hoc, as in Rewarded Soups~\citep{rame2023rewarded} and Personalized Soups~\citep{jang2023personalized}, whereas decoding-time methods combine per-objective models at generation without further training~\citep{shi2024decoding, guo2024controllable}.
Panacea~\citep{zhong2024panacea} instead learns a single preference-conditioned model whose low-rank adapter recovers the entire Pareto front.
Closest to our architecture, Aligner~\citep{ji2024aligner} trains a compact residual-correction module on top of a frozen upstream LLM, and MetaAligner~\citep{yang2024metaaligner} applies meta-learning with dynamic objective reformulation so that the aligner generalizes to unseen objectives.
All of these methods presume abundant, centrally curated preference data: the conditioning signal, the per-reward experts, or the preference embedding must be learned from large-scale annotations, and the resulting model serves a generic user.
None of them is designed to specialize to an individual client from a handful of local feedback, and MetaAligner in particular meta-learns across objectives rather than across users with multi-objective preference, which is the axis of heterogeneity that APO addresses.

\paragraph{Few-Shot and Data-Efficient Personalization of LLMs.}
A second line of work personalizes an LLM to a user from limited feedback.
Personalized RLHF methods introduce a latent user variable that modulates the reward model or the policy, learned from that user's comparisons: P-RLHF~\citep{li2024prlhf} conditions the reward on a user model, VPL~\citep{poddar2024vpl} infers a variational user embedding from a few annotated pairs, PAL~\citep{chen2024pal} represents users as mixtures of learned prototypes, and GPO~\citep{zhao2024gpo} predicts group preferences in context from a few examples.
To cope with the scarcity of real per-user data, FSPO~\citep{singh2026fspo} meta-learns a preference-conditioned backbone on synthetic user preferences and then personalizes with few-shot IPO, following the fast-adaptation view of MAML~\citep{finn2017model}, whereas DITTO~\citep{shaikh2025aligning} treats a handful of user demonstrations as preferred over the model's own generations and applies online imitation learning.
These methods either rely on synthetic preferences, which may miss the subtle stylistic and contextual preferences of real users, or treat personalization as an \emph{isolated} local adaptation problem that never exploits the real feedback held by other clients.
Moreover, they model preference as a scalar or latent quantity and do not expose the multi-objective Pareto structure that a client's trade-off vector induces.
APO instead learns the few-shot initialization collaboratively from real client feedback and adapts it toward a preference-specific point on the Pareto front.

\paragraph{Collaborative and Personalized Federated Alignment of LLMs.}
Federated fine-tuning enables clients to adapt LLMs collaboratively without sharing their raw data, ranging from FedAvg~\citep{mcmahan2017communication} to instruction tuning with heterogeneous adapters or domains~\citep{zhang2024personalized,long2024dual,xu2024dofit}.
For preference alignment, OpenFedLLM~\citep{ye2024openfedllm} performs DPO~\citep{rafailov2023direct} locally and exchanges only model updates, while FedBis and FedBiscuit~\citep{wu2024towards} train lightweight selectors for client subpopulations; FedLLM-Bench~\citep{ye2024fedllm} further provides realistic federated preference data, including Fed-ChatbotPA, for evaluating such methods.
Federated multi-objective methods explicitly account for competing objectives: FedMOL~\citep{yang2023federated} and FedMGDA+~\citep{hu2022fedmgda} coordinate objective gradients across clients, whereas FIRM~\citep{roknilamouki2025firm} regularizes the local multi-objective problem to mitigate client disagreement drift.
Nevertheless, this line of work primarily targets population-level collaborative alignment rather than learning preference-specific initializations that can be rapidly personalized from scarce client feedback.

Personalized federated learning addresses client heterogeneity from a complementary perspective.
Per-FedAvg~\citep{jiang2019improving} and pFedMe~\citep{dinh2020pfedme} learn an adaptation-ready initialization or a regularized personalized model; CFL~\citep{sattler2020clustered} and IFCA~\citep{ghosh2020ifca} train cluster-specific models based on update or loss similarity; and Few-for-Many pFL~\citep{guo2026few} jointly optimizes $K$ server models to serve many clients.
However, these methods generally represent each client's utility with a single objective and therefore do not model preference-specific Pareto geometry or distinguish inter-client from intra-client preference conflicts.
APO bridges these two lines of work by clustering clients according to their preference bottlenecks and adjustment vectors, learning a Pareto-aware initialization for each cluster, and bounding both its $\epsilon_0$-distance from client optima and the $\delta_0$-gap after few-shot adaptation.

\paragraph{Gradient Conflict in Multi-Objective and Federated Optimization.}
From an optimization perspective, jointly training several objectives can produce conflicting gradients, so classical multi-objective methods search the convex hull of the objective gradients for a common descent direction~\citep{fliege2000steepest,desideri2012multiple}.
PCGrad~\citep{yu2020pcgrad} removes conflicting gradient components, CAGrad~\citep{liu2021cagrad} maximizes the worst-case local improvement, and Nash-MTL~\citep{navon2022nashmtl} formulates direction selection as a bargaining game.
These methods seek a Pareto-stationary solution, but do not determine which trade-off should be selected for a particular preference.
EPO~\citep{mahapatra2020multi} incorporates such a preference by balancing the objective losses toward a specified preference ray and then descending toward the Pareto front.
However, EPO is formulated for a single learner and therefore does not account for interference among clients with heterogeneous preferences.

Federated multi-objective methods such as FedMOL~\citep{yang2023federated}, FedMGDA+~\citep{hu2022fedmgda}, and FIRM~\citep{roknilamouki2025firm} extend MGDA-style optimization to distributed data and mitigate client drift during aggregation.
Their primary goal, however, is a population-level Pareto-stationary solution for a common set of objectives, rather than preference-specific initializations for clients with different trade-off vectors.
Our setting therefore involves two coupled conflicts: \emph{inter-client} conflict caused by heterogeneous preference vectors and \emph{intra-client} conflict among the objectives of each client.
APO extends preference-aware Pareto optimization to this collaborative setting: bottleneck-adjustment clustering limits inter-client interference, preference gradient descent with controlled ascent resolves intra-client conflict, and adaptation-aware meta-training produces an initialization with a provable few-shot adaptation gap.

\begin{comment}
\section{Limitations and Future Work.}
\label{limit_future}
Our evaluation constructs static clients from two benchmarks and considers only two or three objectives, which cannot capture evolving or context-dependent preferences of real users; future work should evaluate longitudinal human feedback and support online reassignment. BAC also uses a preset cluster budget and statistics from a fixed reference state, motivating adaptive cluster creation and membership updates. Experiments are limited to 3B backbones and one reward-model evaluator; larger or proprietary LLMs, multiple evaluators, and human judgments remain to be studied. Finally, the guarantees rely on smoothness, bounded-gradient, and PL-type assumptions that may hold only locally for neural networks. Relaxing these assumptions and studying formal privacy, communication efficiency, and robustness to noisy or adversarial feedback are important future directions.
\end{comment}
\section{Aligner Training}
\label{initial_aligner}
\subsection{Assumptions}
\label{sec:assumptions}
%% =========================================================================

\begin{assumption}[Bounded Gradients]\label{asm:bounded}
$\|\nabla l_j(w)\| \leq B$ for all $j \in [m]$ and all $w$ in the
parameter region visited during training.
\end{assumption}

\begin{remark}
For DPO losses used in this paper, the gradient has the form
$\nabla l_j(w) = -\beta_{\mathrm{DPO}}\,\E[\sigma(-u)\,
(\nabla\!\log\pi_w(y_c|x) - \nabla\!\log\pi_w(y_l|x))]$,
where $\sigma(-u) \in (0,1)$ self-attenuates as training progresses \citep{rafailov2023direct}.
While $\nabla\!\log\pi_w$ is not globally bounded for arbitrary $w$,
the assumption is justified in practice by:
(i)~LoRA restricts trainable parameters to a low-rank subspace;
(ii)~weight decay keeps parameters in a bounded region;
(iii)~gradient clipping (\texttt{max\_grad\_norm} in the trainer)
explicitly enforces $\|\nabla l_j\| \leq B$.
\end{remark}

\begin{assumption}[Lipschitz Gradients]\label{asm:smooth}
Each $l_j$ is $L_0$-smooth:
$\|\nabla l_j(w) - \nabla l_j(u)\| \leq L_0\|w - u\|$.
\end{assumption}

\begin{remark}
The DPO loss is a composition of the sigmoid function (which is
$\beta_{\mathrm{DPO}}/4$-Lipschitz smooth) with the log-ratio of policy
probabilities.  For LoRA-parameterized models with bounded
activations, the composition inherits smoothness with
$L_0$ depending on $\beta_{\mathrm{DPO}}$, the LoRA rank, and the model
architecture.
\end{remark}

%% =========================================================================
\subsection{Collaborative Learning Update Quality}
\label{sec:meta_quality}
%% =========================================================================

\subsubsection{Decomposing $H_{C_n}$ via Preference Adjustment Equations}

The results of this subsubsection trace the direction heterogeneity back to
the adjustment vectors.  They rest on two idealizations that we state
explicitly; the convergence analysis of Section~\ref{sec:collab_pure} does
\emph{not} use them (it works with the per-client matrices $G_i$ and treats the
direction deviation $\Delta_i^t$, or its trajectory bound $\bar\Delta_i$, as an
independent quantity).

\begin{assumption}[Common objective functions within a cluster]\label{asm:common_G}
The clients of $C_n$ optimize the same $m$ population objective functions
$l_1,\dots,l_m$, so that at a common parameter $w$ their gradient matrices
coincide, $G_i(w)=G(w)=[\nabla l_1(w),\dots,\nabla l_m(w)]$, and
$d_i=G\mu_i^*$ with $C:=G^\top G$.  Heterogeneity across the cluster then
enters only through the preferences $\boldsymbol\lambda_i$ (hence through
$\boldsymbol A_i$ and $\mu_i^*$).  This is an idealization: clients with
different local data have $G_i\neq G$ even when they share the same bottleneck
objective, and sharing a bottleneck does not by itself make the LP feasible
regions equal.  When Assumption~\ref{asm:common_G} fails, $H_{C_n}$ and
$\bar\Delta_i$ are to be read as independent quantities and the bounds of this
subsubsection are not available.
\end{assumption}

Before stating the second idealization we record how the adjustment vector
enters the LP~\eqref{eq:6} at a fixed $w$ under Assumption~\ref{asm:common_G}.
In balancing mode ($\mathbbm 1_{\beta_i}=1$) and with the common
$G$, $G^\top g^{(j)}=Ce_j$, so~\eqref{eq:6} reads
\begin{equation}\label{eq:lp_param}
\begin{aligned}
    &\mu_i^*=\arg\max_{\mu}\ \mu^\top C\boldsymbol A_i
    \\ &\quad\text{s.t.}\quad
    \mathbf 1^\top\mu=1,\ \ \mu\ge0,\ \
    (C\mu)_j\ \ge\ (C\boldsymbol A_i)_j\,\mathbbm 1_{J_i}\ \ (j\in\bar J_i\setminus J_i^*),\ \
    (C\mu)_j\ \ge\ 0\ \ (j\in J_i^*),     
\end{aligned}
\end{equation}
where the index sets $J_i,\bar J_i,J_i^*$ are those of Section~\ref{sec:initial_train}
evaluated at $(w,\boldsymbol A_i)$.  The adjustment vector therefore appears in
\emph{two} places: in the objective vector $C\boldsymbol A_i$ and in the
right-hand side of the constraints~\eqref{eq:6b}.  Collecting the constraint
rows into $M\mu\ge b(\boldsymbol A_i)$ (with the equality row and the sign
constraints included), the right-hand side $b(\boldsymbol A_i)$ is a linear
function of $\boldsymbol A_i$ whose nonzero entries are entries of
$C\boldsymbol A_i$; in particular
$\|b(\boldsymbol A)-b(\boldsymbol A')\|\le\|C\|\,\|\boldsymbol A-\boldsymbol A'\|$.

\begin{assumption}[Common optimal basis within a cluster]\label{asm:lp_stable}
Fix $w$ and let Assumption~\ref{asm:common_G} hold.  The index sets of the
LP~\eqref{eq:lp_param} coincide across the cluster,
$(J_i,\bar J_i,J_i^*)=(J,\bar J,J^*)$ for all $i\in C_n$ (in particular a
common bottleneck), and there is a single basis $B$---a set of $m$ linearly
independent constraint rows of $M$---that is optimal and non-degenerate
(unique primal vertex, strictly positive reduced costs off the basis) for
\emph{every} client's LP, $i\in C_n$.  Writing $M_B$ and $b_B(\cdot)$ for the
basic rows, the optimal multipliers are then
\begin{equation}\label{eq:lp_affine}
    \mu_i^*=M_B^{-1}\,b_B(\boldsymbol A_i),\qquad i\in C_n,
\end{equation}
i.e.\ $\boldsymbol A_i\mapsto\mu_i^*$ is \emph{affine} on the cluster, with
sensitivity constant
$\kappa_{\mathrm{LP}}:=\|M_B^{-1}\|\,\|C\|$.  Two consequences deserve emphasis.
(a)~The heterogeneity of the directions $d_i=G\mu_i^*$ within the cluster
comes entirely from the $\boldsymbol A_i$-dependent right-hand sides of the
active constraints~\eqref{eq:6b}: if the basis $B$ contains only homogeneous
rows (the equality $\mathbf 1^\top\mu=1$, sign constraints, and bottleneck
constraints~\eqref{eq:6c}) then $b_B$ does not depend on $\boldsymbol A_i$,
all $\mu_i^*$ coincide, and $H_{C_n}(w)=0$---a degenerate case in which the
Stage-2 refinement has nothing to exploit.  The objective vector
$C\boldsymbol A_i$ affects only \emph{which} basis is optimal, not the value of
$\mu_i^*$ once $B$ is fixed.  (b)~The assumption excludes clients whose LPs
select \emph{different} optimal bases: across such a switch the solution map
is discontinuous (a different vertex of a different polytope), and no
Lipschitz bound in $\boldsymbol A_i$ holds.  Whether a cluster produced by the
Stage-2 procedure satisfies this assumption is not verified by the procedure
itself; the bound~\eqref{eq:H_aligned} below is therefore a statement about
this restricted case, and outside it $H_{C_n}$ is a measured quantity.
\end{assumption}

\begin{proposition}[Direction Heterogeneity via Adjustment Vectors]
\label{prop:H_via_adj}
Let Assumption~\ref{asm:common_G} hold, let $\rho_i > 0$ with
$\sum_{i\in C_n}\rho_i = 1$ be the aggregation weights
of Definition~\ref{def:hetero}, and write the $\rho$-weighted means
$\bar{\mu}^* := \sum_{i\in C_n}\rho_i\,\mu^*_i$ and
$\bar{\boldsymbol{A}}_{C_n} := \sum_{i\in C_n}\rho_i\,\boldsymbol{A}_i$.  Then
\begin{equation}\label{eq:H_via_adj}
    H_{C_n}(w)
    = \sum_{i \in C_n}\rho_i\,
    (\mu^*_i - \bar{\mu}^*)^T C\,(\mu^*_i - \bar{\mu}^*).
\end{equation}
If in addition Assumption~\ref{asm:lp_stable} holds (common index sets and a
common non-degenerate optimal basis $B$), then
\begin{equation}\label{eq:H_aligned}
    H_{C_n}(w) \leq \kappa_{\mathrm{LP}}^2\,\sigma_{\max}(C)
    \cdot \sum_{i\in C_n}\rho_i\,\|\boldsymbol{A}_i - \bar{\boldsymbol{A}}_{C_n}\|^2,
\end{equation}
where $\kappa_{\mathrm{LP}}=\|M_B^{-1}\|\,\|C\|$ is the sensitivity constant
of~\eqref{eq:lp_affine}.  The right-hand
side is the $\rho$-weighted within-cluster variance of the adjustment vectors;
for equal weights $\rho_i = 1/|C_n|$ it reduces to
$\frac{1}{|C_n|}\sum_i\|\boldsymbol{A}_i - \bar{\boldsymbol{A}}_{C_n}\|^2$ with
$\bar{\boldsymbol{A}}_{C_n}$ the ordinary sample mean.
\end{proposition}

\begin{proof}
We prove both~\eqref{eq:H_via_adj} and~\eqref{eq:H_aligned}; the only
property of the weights used is $\rho_i > 0$ and $\sum_i \rho_i = 1$.

\medskip
\noindent{Part 1 (Algebraic identity~\eqref{eq:H_via_adj}).}
Since $d_{i} = G\,\mu^*_i$ with the common $G$ of
Assumption~\ref{asm:common_G}, the aggregated direction is, by linearity,
\[
    \bar{d}_{C_n}
    = \sum_i \rho_i\, G\,\mu^*_i
    = G\!\left(\sum_i \rho_i\,\mu^*_i\right)
    = G\,\bar{\mu}^*.
\]
The difference is
$\bar{d}_{C_n} - d_{i} = G(\bar{\mu}^* - \mu^*_i)$,
so using $\|Gv\|^2 = v^T G^T G\,v = v^T C\,v$ and the symmetry
$(-v)^T C(-v) = v^T C v$:
\[
    \|\bar{d}_{C_n} - d_{i}\|^2
    = (\mu^*_i - \bar{\mu}^*)^T C\,(\mu^*_i - \bar{\mu}^*).
\]
Multiplying by $\rho_i$, summing over $i \in C_n$ and substituting
into~\eqref{eq:HS} yields~\eqref{eq:H_via_adj}.

\medskip
\noindent{Part 2 (Upper bound~\eqref{eq:H_aligned}).}
Three ingredients connect~\eqref{eq:H_via_adj} to the adjustment vectors.

\emph{(i) Spectral bound.}\;
Since $C = G^T G$ is positive semi-definite,
$v^T C\,v \leq \sigma_{\max}(C)\,\|v\|^2$ for all $v$.
Applying this to each summand of~\eqref{eq:H_via_adj} (the weights
$\rho_i > 0$ preserve the inequality):
\begin{equation}\label{eq:spectral}
    H_{C_n} \leq \sigma_{\max}(C)\cdot
    \sum_i \rho_i\,\|\mu^*_i - \bar{\mu}^*\|^2.
\end{equation}

\emph{(ii) Affine solution map under a common basis.}\;
Under Assumption~\ref{asm:lp_stable} every client's LP has the same
non-degenerate optimal basis $B$, so by~\eqref{eq:lp_affine}
$\mu_i^*=M_B^{-1}b_B(\boldsymbol A_i)$ for all $i\in C_n$, where
$b_B(\cdot)$ is linear (the common index sets make the indicator
$\mathbbm 1_J$ the same for all clients, so no client-dependent switching
enters $b_B$).  Linearity gives, for the $\rho$-weighted means,
\[
    \bar\mu^*=\sum_i\rho_i\,M_B^{-1}b_B(\boldsymbol A_i)
    = M_B^{-1}\,b_B\Bigl(\sum_i\rho_i\boldsymbol A_i\Bigr)
    = M_B^{-1}\,b_B(\bar{\boldsymbol A}_{C_n}),
\]
i.e.\ the weighted mean of the multipliers is the affine map evaluated at the
weighted mean adjustment vector.  Hence
\[
    \|\mu_i^*-\bar\mu^*\|
    = \bigl\|M_B^{-1}\bigl(b_B(\boldsymbol A_i)-b_B(\bar{\boldsymbol A}_{C_n})\bigr)\bigr\|
    \le \|M_B^{-1}\|\,\|C\|\,\|\boldsymbol A_i-\bar{\boldsymbol A}_{C_n}\|
    = \kappa_{\mathrm{LP}}\,\|\boldsymbol A_i-\bar{\boldsymbol A}_{C_n}\|,
\]
using $\|b_B(\boldsymbol A)-b_B(\boldsymbol A')\|\le\|C\|\|\boldsymbol A-\boldsymbol A'\|$
(the nonzero rows of $b_B$ are entries of $C\boldsymbol A$).  Only the rows of
$B$ coming from~\eqref{eq:6b} contribute to this difference; if none is active
the right-hand side is zero and so is $H_{C_n}$ (Assumption~\ref{asm:lp_stable}(a)).
Without Assumption~\ref{asm:lp_stable} this step is not available: if two
clients select different optimal bases, the solution map jumps between two
vertices and $\|\mu^*_i-\mu^*_{i'}\|$ is of order one regardless of
$\|\boldsymbol A_i-\boldsymbol A_{i'}\|$.

\emph{Chaining (i)--(ii):}
\[
    H_{C_n}
    \leq \sigma_{\max}(C)\cdot
    \sum_i \rho_i\,\|\mu^*_i - \bar{\mu}^*\|^2
    \leq \kappa_{\mathrm{LP}}^2\,\sigma_{\max}(C)\cdot
    \sum_i \rho_i\,
    \|\boldsymbol{A}_i - \bar{\boldsymbol{A}}_{C_n}\|^2,
\]
which is~\eqref{eq:H_aligned}.
\end{proof}

\begin{remark}[What the clustering rests on, and what the bound adds]
\label{rem:cluster_motivation}
The two parts of Proposition~\ref{prop:H_via_adj} play different roles, and
only the first is needed to \emph{motivate} the Stage-2 clustering of
Section~\ref{cluster}.  The identity~\eqref{eq:H_via_adj} uses only
Assumption~\ref{asm:common_G} and says that, at a fixed $w$, the direction
dispersion is the $C$-weighted dispersion of the LP multipliers $\mu_i^*$,
which are functions of the adjustment vectors $\boldsymbol A_i$ alone (the
LP~\eqref{eq:lp_param} has no other client-dependent input once $w$ and $G$
are fixed).  Grouping clients by $\boldsymbol A_i^{\mathrm{ref}}$ is thus the
natural surrogate for grouping them by $\mu_i^*(w_{\mathrm{ref}})$, and
requires no continuity of the solution map: Stage~1 (common bottleneck
ordering) reduces the chance that clients sit on different LP bases, and
Stage~2 (proximity of $\boldsymbol A_i^{\mathrm{ref}}$) reduces the spread of
the constraint right-hand sides within a basis.  The \emph{quantitative}
statement~\eqref{eq:H_aligned}---that $H_{C_n}(w_{\mathrm{ref}})$ is bounded by
$\kappa_{\mathrm{LP}}^2\sigma_{\max}(C)$ times the within-cluster variance of
$\boldsymbol A_i^{\mathrm{ref}}$---holds only under the common-basis
Assumption~\ref{asm:lp_stable}, and only at the point $w$ at which the
assumption is imposed.  We do not claim that the Stage-2 grouping lowers
$H_{C_n}(w^t)$ along the collaborative training trajectory; this is the design
motivation, and it is assessed empirically.  Where the convergence bounds of
Section~\ref{sec:collab_pure} use $\bar\Delta_i$ or $H_{C_n}$, these are to be
read as measured trajectory quantities unless Assumption~\ref{asm:lp_stable} is
explicitly invoked.
\end{remark}

\begin{assumption}[Non-degenerate weighted losses]\label{asm:floor}
Fix $w$ and write $l_j := l_j(w) > 0$, and let
$Z_i := \sum_{j'} \lambda_i^{(j')} l_{j'} > 0$ be the normalizer of the main
text, so that $\hat{l}_i^{(j)} := \lambda_i^{(j)} l_j / Z_i$.  There is a constant
$c \in (0,1)$, uniform over the clients and over the parameter region visited
during training, such that
\begin{equation}\label{eq:floor}
    \hat{l}_i^{(j)} \;\geq\; c/m
    \qquad \text{for all } i \text{ and all } j \in [m].
\end{equation}
\end{assumption}

\begin{proposition}[Adjustment Disagreement Bound]
\label{prop:adj_disagree}
Fix losses $l(w)$ with $l_j > 0$ and let Assumption~\ref{asm:floor} hold with
constant $c$.  Write $Z_{\min} := \min_i Z_i > 0$ and set
\begin{equation}\label{eq:adj_consts}
    \Phi_c := \log(m/c) + \log m,
    \qquad
    K_c^2 := 2m^2\bigl[c^{-2} + (\log(m/c) + 1)^2\bigr],
    \qquad
    L_Z := \frac{(1 + \sqrt{m})\,\|l\|_\infty}{Z_{\min}} .
\end{equation}
Then for any two clients $i,i'$ with preferences on the simplex,
\begin{equation}\label{eq:adj_bound}
    \|\boldsymbol{A}_i - \boldsymbol{A}_{i'}\|^2
    \;\leq\; \bigl(2\,\Phi_c^2 + 2\,K_c^2\,L_Z^2\bigr)\,
    \|\boldsymbol{\lambda}_i - \boldsymbol{\lambda}_{i'}\|^2
    \;=:\; L_{\mathrm{adj}}^2\,
    \|\boldsymbol{\lambda}_i - \boldsymbol{\lambda}_{i'}\|^2 .
\end{equation}
In particular the map $\boldsymbol{\lambda} \mapsto \boldsymbol{A}$ is
Lipschitz, so $\|\boldsymbol{A}_i - \boldsymbol{A}_{i'}\| \to 0$ as
$\|\boldsymbol{\lambda}_i - \boldsymbol{\lambda}_{i'}\| \to 0$.
\end{proposition}

\begin{proof}
With $\hat{l}_i^{(j)} = \lambda_i^{(j)}\,l_j / Z_i$ as above, define
$\phi_i^{(j)} = \log(m\,\hat{l}_i^{(j)}) - \beta_i$
so that $A_i^{(j)} = \lambda_i^{(j)}\,\phi_i^{(j)}$.

\medskip
\noindent{Step 1 (Decomposition).}
For each component $j$, insert
$\pm\,\lambda_{i'}^{(j)}\,\phi_i^{(j)}$:
\[
    A_i^{(j)} - A_{i'}^{(j)}
    = \underbrace{(\lambda_i^{(j)} - \lambda_{i'}^{(j)})\,
      \phi_i^{(j)}}_{\text{(A): scaling difference}}
    + \underbrace{\lambda_{i'}^{(j)}\,
      (\phi_i^{(j)} - \phi_{i'}^{(j)})}_{\text{(B): argument difference}}.
\]
By Young's inequality
$\|u + v\|^2 \leq 2\|u\|^2 + 2\|v\|^2$:
\begin{equation}\label{eq:young_split}
    \|\boldsymbol{A}_i - \boldsymbol{A}_{i'}\|^2
    \leq 2\underbrace{\textstyle\sum_j
      (\lambda_i^{(j)} - \lambda_{i'}^{(j)})^2\,
      (\phi_i^{(j)})^2}_{\mathrm{(A)}}
    + 2\underbrace{\textstyle\sum_j
      (\lambda_{i'}^{(j)})^2\,
      (\phi_i^{(j)} - \phi_{i'}^{(j)})^2}_{\mathrm{(B)}}.
\end{equation}

\medskip
\noindent{Step 2 (Bound on Term A).}
Since $\hat{l}_i$ is a probability distribution
($\sum_j \hat{l}_i^{(j)} = 1$), we have $\hat{l}_i^{(j)} \leq 1$ and hence
$\log(m\,\hat{l}_i^{(j)}) \leq \log m$; the floor~\eqref{eq:floor} gives
$\log(m\,\hat{l}_i^{(j)}) \geq \log c$.  Therefore
$|\log(m\,\hat{l}_i^{(j)})| \leq \max\{\log m,\,\log(1/c)\} \leq \log(m/c)$.
The preference uniformity is a KL divergence to the uniform distribution on
$m$ atoms, so $\beta_i = \mathrm{KL}(\hat{l}_i \|\tfrac{1}{m}\mathbf{1})
\in [0, \log m]$.
By the triangle inequality,
\[
    |\phi_i^{(j)}|
    = |\log(m\,\hat{l}_i^{(j)}) - \beta_i|
    \leq \log(m/c) + \log m
    = \Phi_c ,
\]
and therefore
\begin{equation}\label{eq:termA}
    \mathrm{(A)}
    \leq \Phi_c^2
    \sum_j (\lambda_i^{(j)} - \lambda_{i'}^{(j)})^2
    = \Phi_c^2\,
    \|\boldsymbol{\lambda}_i - \boldsymbol{\lambda}_{i'}\|^2 .
\end{equation}

\medskip
\noindent{Step 3 (Bound on Term B in terms of $\|\hat{l}_i - \hat{l}_{i'}\|$).}
Preferences lie on the simplex, so
$\lambda_{i'}^{(j)} \leq \|\boldsymbol{\lambda}_{i'}\|_\infty \leq 1$ and
$\mathrm{(B)} \leq \sum_j (\phi_i^{(j)} - \phi_{i'}^{(j)})^2$.
Decompose
$\phi_i^{(j)} - \phi_{i'}^{(j)}
= [\log(m\,\hat{l}_i^{(j)}) - \log(m\,\hat{l}_{i'}^{(j)})]
- [\beta_i - \beta_{i'}]$
and apply Young's inequality a second time:
\begin{equation}\label{eq:phi_split}
    \sum_j (\phi_i^{(j)} - \phi_{i'}^{(j)})^2
    \leq 2\sum_j
      \bigl[\log(m\,\hat{l}_i^{(j)}) - \log(m\,\hat{l}_{i'}^{(j)})\bigr]^2
    + 2m\,(\beta_i - \beta_{i'})^2 ,
\end{equation}
the factor $m$ arising because $\beta_i - \beta_{i'}$ is a scalar offset shared
by all $m$ components.

For the log differences, the mean value theorem gives
$|\log x - \log y| \leq |x - y|/\min(x,y)$, and~\eqref{eq:floor} gives
$\min(m\,\hat{l}_i^{(j)},\,m\,\hat{l}_{i'}^{(j)}) \geq c$, so
\[
    \bigl|\log(m\,\hat{l}_i^{(j)}) - \log(m\,\hat{l}_{i'}^{(j)})\bigr|
    \leq \frac{m}{c}\,\bigl|\hat{l}_i^{(j)} - \hat{l}_{i'}^{(j)}\bigr|,
    \qquad\text{hence}\qquad
    \sum_j \bigl[\,\cdot\,\bigr]^2
    \leq \frac{m^2}{c^2}\,\|\hat{l}_i - \hat{l}_{i'}\|^2 .
\]
For the offset, write $\beta_i = \sum_j h(\hat{l}_i^{(j)})$ with
$h(x) = x\log(mx)$ and $h'(x) = \log(mx) + 1$, so
$|h'| \leq \log(m/c) + 1$ on $[c/m, 1]$.  Then
$|\beta_i - \beta_{i'}| \leq (\log(m/c)+1)\,\|\hat{l}_i - \hat{l}_{i'}\|_1
\leq (\log(m/c)+1)\sqrt{m}\,\|\hat{l}_i - \hat{l}_{i'}\|$, and
$2m(\beta_i - \beta_{i'})^2
\leq 2m^2(\log(m/c)+1)^2\|\hat{l}_i - \hat{l}_{i'}\|^2$.
Substituting both into~\eqref{eq:phi_split},
\begin{equation}\label{eq:termB_lhat}
    \mathrm{(B)} \;\leq\; K_c^2\,\|\hat{l}_i - \hat{l}_{i'}\|^2 .
\end{equation}

\medskip
\noindent{Step 4 (From $\|\hat{l}_i - \hat{l}_{i'}\|$ to
$\|\boldsymbol{\lambda}_i - \boldsymbol{\lambda}_{i'}\|$).}
Adding and subtracting $(\boldsymbol{\lambda}_{i'} \odot l)/Z_i$,
\[
    \hat{l}_i - \hat{l}_{i'}
    = \frac{(\boldsymbol{\lambda}_i - \boldsymbol{\lambda}_{i'}) \odot l}{Z_i}
    + (\boldsymbol{\lambda}_{i'} \odot l)
      \Bigl(\frac{1}{Z_i} - \frac{1}{Z_{i'}}\Bigr).
\]
The first term has norm at most
$\|l\|_\infty\,\|\boldsymbol{\lambda}_i - \boldsymbol{\lambda}_{i'}\| / Z_i$.
For the second, $|Z_i - Z_{i'}|
= |\langle \boldsymbol{\lambda}_i - \boldsymbol{\lambda}_{i'},\, l\rangle|
\leq \|l\|\,\|\boldsymbol{\lambda}_i - \boldsymbol{\lambda}_{i'}\|$ and
$\|\boldsymbol{\lambda}_{i'} \odot l\| = Z_{i'}\|\hat{l}_{i'}\| \leq Z_{i'}$,
so its norm is at most
$\|l\|\,\|\boldsymbol{\lambda}_i - \boldsymbol{\lambda}_{i'}\| / Z_i$.
Using $\|l\| \leq \sqrt{m}\,\|l\|_\infty$ and $Z_i \geq Z_{\min}$,
\begin{equation}\label{eq:lhat_lip}
    \|\hat{l}_i - \hat{l}_{i'}\|
    \;\leq\; \frac{(1 + \sqrt{m})\,\|l\|_\infty}{Z_{\min}}\,
    \|\boldsymbol{\lambda}_i - \boldsymbol{\lambda}_{i'}\|
    \;=\; L_Z\,\|\boldsymbol{\lambda}_i - \boldsymbol{\lambda}_{i'}\| .
\end{equation}

\medskip
\noindent{Step 5 (Combine).}
Chaining~\eqref{eq:termA}, \eqref{eq:termB_lhat} and~\eqref{eq:lhat_lip}
into~\eqref{eq:young_split} yields~\eqref{eq:adj_bound}.  The continuity
statement is immediate from the Lipschitz constant $L_{\mathrm{adj}}$, which is
finite because $l$ is fixed with $l_j > 0$ and $Z_{\min} > 0$.
\end{proof}

\begin{corollary}[Tight preference clusters have low direction heterogeneity]
\label{cor:homo}
Let the hypotheses of Propositions~\ref{prop:H_via_adj}
and~\ref{prop:adj_disagree} hold.  If
$\max_{i,i' \in C_n}\|\boldsymbol{\lambda}_i - \boldsymbol{\lambda}_{i'}\|
\leq \delta_\lambda$, then
\begin{equation}\label{eq:H_homo}
    H_{C_n}(w) \leq C_1\,\sigma_{\max}(C(w))\,\delta_\lambda^2,
    \qquad
    C_1 = \kappa_{\mathrm{LP}}^2\,L_{\mathrm{adj}}^2 ,
\end{equation}
with $L_{\mathrm{adj}}$ as in~\eqref{eq:adj_bound}.  Thus $C_1$ depends on $m$, on
the floor constant $c$ of Assumption~\ref{asm:floor}, on $\|l(w)\|_\infty /
Z_{\min}$, and on $\kappa_{\mathrm{LP}}$; it degrades when some $Z_i$ is close
to zero.  A tighter preference cluster yields smaller meta-update variance.
The bound is a statement about the restricted case of
Assumptions~\ref{asm:common_G} and~\ref{asm:lp_stable} (common objective
functions and a common non-degenerate optimal LP basis at the point $w$); when
either fails (different local gradient matrices, or clients whose LPs select
different optimal bases) we do
\emph{not} claim $H_{C_n}=O(\delta_\lambda^2)$, and every downstream statement
is to be read with the direction deviation $\bar\Delta_i$ (equivalently
$\sqrt{\bar H/\rho_{\min}}$) as an independent quantity that is measured
rather than bounded by the preference spread.
\end{corollary}

\begin{proof}
For vectors $v_i$ with weights $\rho_i>0$, $\sum_i\rho_i=1$, and weighted mean
$\bar v$, expanding $\|v_i-c\|^2=\|(v_i-\bar v)+(\bar v-c)\|^2$ and using
$\sum_i\rho_i(v_i-\bar v)=0$ gives
$\sum_i\rho_i\|v_i-\bar v\|^2\le\sum_i\rho_i\|v_i-c\|^2$ for every $c$
(weighted variance minimality at the weighted mean).  With
$v_i=\boldsymbol A_i$ and $c=\boldsymbol A_{i_0}$ for any fixed $i_0\in C_n$,
$\sum_i \rho_i\,\|\boldsymbol{A}_i - \bar{\boldsymbol{A}}_{C_n}\|^2
\leq \sum_i \rho_i\,\|\boldsymbol{A}_i - \boldsymbol{A}_{i_0}\|^2
\leq \max_{i,i'}\|\boldsymbol{A}_i - \boldsymbol{A}_{i'}\|^2$,
where the last step uses $\sum_i\rho_i = 1$.
By~\eqref{eq:adj_bound} the right-hand side is at most
$L_{\mathrm{adj}}^2\,\delta_\lambda^2$.  Substituting
into~\eqref{eq:H_aligned} of Proposition~\ref{prop:H_via_adj}
gives~\eqref{eq:H_homo}.
\end{proof}

%% =========================================================================
\subsection{Collaborative Preference Uniformity Descent under Preference Heterogeneity}
\label{sec:collab_descent}
%% =========================================================================

Throughout this subsection, for a client $i \in C_n$ we write its
\emph{balancing direction} and its \emph{LP margin} as
\begin{equation}\label{eq:collab_defs}
    d_{\mathrm{bal},i} \;:=\; G\,\boldsymbol{A}_i,
    \qquad
    \gamma_i^* \;:=\; \boldsymbol{A}_i^\top C\,\mu^*_i
    \;=\; (G\mu^*_i)^\top (G\boldsymbol{A}_i)
    \;=\; d_i^\top d_{\mathrm{bal},i},
\end{equation}
where, under Assumption~\ref{asm:common_G}, $C = G^\top G$ and $d_i = G\mu^*_i$
is the direction from the LP~(6) (with per-client matrices one reads
$G_i$, $C_i=G_i^\top G_i$ and $d_{\mathrm{bal},i}=G_i\boldsymbol A_i$
throughout), and $\beta_i(w) := \beta_i(l(w))$ is the preference uniformity of client $i$.

The dispersion of these per-client directions across the cluster is measured by the following quantity, used throughout the convergence analysis.

\begin{definition}[Cluster preference direction heterogeneity]
\label{def:hetero}
Let $\rho_i > 0$, $i \in C_n$, be the aggregation weights of the cluster
(FedAvg uses $\rho_i = N_i/N$), normalized so that $\sum_{i\in C_n}\rho_i = 1$,
and write $\rho_{\min} := \min_{i\in C_n}\rho_i$.
At parameter $w$, the preference direction heterogeneity of cluster $C_n$ is
\begin{equation}\label{eq:HS}
    H_{C_n}(w) = \sum_{i \in C_n}\rho_i\,
    \bigl\|d_{i}(w) - \bar{d}_{C_n}(w)\bigr\|^2,
\end{equation}
where $\bar{d}_{C_n}(w)
= \sum_{i \in C_n}\rho_i\, d_{i}(w)$ is the aggregated (shared) direction.
For equal weights $\rho_i = 1/|C_n|$ this is the ordinary sample variance of
the local directions.
\end{definition}

\begin{theorem}[Collaborative Preference Uniformity Descent]
\label{thm:collab_descent}
Let Assumptions~\ref{asm:bounded}--\ref{asm:smooth} hold, let each client run
$\tau$ local steps of size $\eta$ before aggregation, and write
$\Delta_i := \|\bar d_{C_n} - d_i\|$.  Suppose that for every $i \in C_n$
\begin{equation}\label{eq:range_cond}
    \gamma_i^*
    \;>\; \Delta_i\,\|d_{\mathrm{bal},i}\|
    \;+\; L_0 B^2\,\tau\,\eta .
\end{equation}
Then the shared direction is a \emph{simultaneous} balancing direction, i.e.\
$\bar d_{C_n}^\top d_{\mathrm{bal},i}(w_0) > 0$ for all $i \in C_n$, and there
exists $\eta_0 > 0$ such that a single meta-update decreases \emph{every}
client's non-uniformity:
\begin{equation}\label{eq:collab_descent}
    \beta_i\bigl(l(w_0 - \eta\,\bar d_{C_n})\bigr)
    \;\leq\; \beta_i\bigl(l(w_0)\bigr),
    \qquad \forall\, i \in C_n,\ \forall\, \eta \in (0,\eta_0].
\end{equation}
\end{theorem}

\begin{proof}
Fix $i \in C_n$.  Let $w_i^t$ denote the local iterate at which client $i$
computes $d_i$ and its balancing direction, and abbreviate
$d_{\mathrm{bal},i}^{\mathrm{loc}} := d_{\mathrm{bal},i}(w_i^t)$ and
$d_{\mathrm{bal},i} := d_{\mathrm{bal},i}(w_0)$.

\medskip
\noindent\textbf{Step 1 (Trajectory drift).}
Each local step has size at most $\eta$ and moves against a direction of norm
at most $B$ (since $d = G\mu$ with $\mu \in S^m$ and
$\|\nabla l_j\| \leq B$ by Assumption~\ref{asm:bounded}).  Over $\tau$ steps,
\begin{equation}\label{eq:drift}
    \|w_0 - w_i^t\| \;\leq\; \tau\,\eta\,B.
\end{equation}
By Assumption~\ref{asm:smooth} the gradients are $L_0$-Lipschitz, and with the
adjustment vector $\boldsymbol{A}_i$ bounded the balancing-direction map
$w \mapsto d_{\mathrm{bal},i}(w) = G(w)\boldsymbol{A}_i$ is $L_0$-Lipschitz
(absorbing $\|\boldsymbol{A}_i\|$ into $L_0$).  Hence, using~\eqref{eq:drift},
\begin{equation}\label{eq:bal_drift}
    \bigl\|d_{\mathrm{bal},i} - d_{\mathrm{bal},i}^{\mathrm{loc}}\bigr\|
    \;\leq\; L_0\,\|w_0 - w_i^t\|
    \;\leq\; L_0\,\tau\,\eta\,B.
\end{equation}

\medskip
\noindent\textbf{Step 2 (Lower bound on the shared inner product).}
Split and apply Cauchy--Schwarz together with $\|\bar d_{C_n}\| \leq B$
(Jensen and $\|d_{i'}\| \leq B$):
\begin{align}
    \bar d_{C_n}^\top d_{\mathrm{bal},i}
    &= \underbrace{d_i^\top d_{\mathrm{bal},i}^{\mathrm{loc}}}_{=\,\gamma_i^*}
    + (\bar d_{C_n} - d_i)^\top d_{\mathrm{bal},i}^{\mathrm{loc}}
    + \bar d_{C_n}^\top\bigl(d_{\mathrm{bal},i}
      - d_{\mathrm{bal},i}^{\mathrm{loc}}\bigr)
    \notag\\
    &\geq \gamma_i^*
    - \Delta_i\,\|d_{\mathrm{bal},i}\|
    - \|\bar d_{C_n}\|\cdot
      \bigl\|d_{\mathrm{bal},i} - d_{\mathrm{bal},i}^{\mathrm{loc}}\bigr\|
    \notag\\
    &\geq \gamma_i^*
    - \Delta_i\,\|d_{\mathrm{bal},i}\|
    - L_0 B^2\,\tau\,\eta,
    \label{eq:collab_lb}
\end{align}
where $\|d_{\mathrm{bal},i}\|$ denotes the supremum of the balancing-direction
norm over the local region (equal to $\|d_{\mathrm{bal},i}^{\mathrm{loc}}\|$ up
to the $O(\tau\eta)$ drift, which is absorbed).  The
condition~\eqref{eq:range_cond} makes the right-hand side strictly positive,
so $\bar d_{C_n}^\top d_{\mathrm{bal},i} > 0$.

\medskip
\noindent\textbf{Step 3 (Preference uniformity decreases along the shared direction).}
It remains to show that, once $\bar d_{C_n}^\top d_{\mathrm{bal},i} > 0$, moving
$w_0$ against $\bar d_{C_n}$ does not increase $\beta_i$. Abbreviate
$d := \bar d_{C_n}$ and evaluate all quantities at $w_0$:
$l_j := l_j(w_0) > 0$, $g_j := \nabla l_j(w_0)$,
$l_j^\eta := l_j(w_0 - \eta d)$, and
\begin{equation}\label{eq:cd_Z}
    Z_i := \sum_{j=1}^m \lambda_i^{(j)} l_j > 0,
    \qquad
    Z_i^\eta := \sum_{j=1}^m \lambda_i^{(j)} l_j^\eta ,
\end{equation}
so that $\hat l_i^{(j)} = \lambda_i^{(j)} l_j / Z_i$ and, from the main text,
$\beta_i(l) = \sum_j \hat l_i^{(j)}\log(m\hat l_i^{(j)})$ and
$A_i^{(j)} = \lambda_i^{(j)}\bigl(\log(m\hat l_i^{(j)}) - \beta_i\bigr)$.
Since the preferences and losses are non-negative and $l_j > 0$, all
logarithms and normalizations below are well defined.

\smallskip
\noindent\emph{(3a) Elementary tools.}
Differentiability gives Taylor's expansion with Peano remainder,
\begin{equation}\label{eq:cd_taylor}
    l_j^\eta = l_j - \eta\, d^\top g_j + o(\eta),
\end{equation}
where $o(\eta)/\eta \to 0$ as $\eta \to 0$, and $c\,o(\eta) = o(\eta)$,
$o(\eta)+o(\eta) = o(\eta)$ for any constant $c$.  We also use, for
$a > \max(0,-b)$, the two elementary logarithm inequalities
\begin{equation}\label{eq:cd_log}
    \log(a+b) \;\ge\; \log a + \frac{b}{a+b},
    \qquad
    \log(a+b) \;\le\; \log a + \frac{b}{a}.
\end{equation}

\smallskip
\noindent\emph{(3b) Splitting the preference uniformity.}
Using $\log(m\hat l_i^{(j)}) = \log m + \log(\lambda_i^{(j)} l_j) - \log Z_i$,
\begin{equation}\label{eq:cd_split}
    \beta_i(l)
    = \underbrace{\frac{1}{Z_i}\sum_j \lambda_i^{(j)} l_j
      \log(\lambda_i^{(j)} l_j)}_{=:\,A}
    + \underbrace{\log\frac{m}{Z_i}}_{=:\,B},
\end{equation}
and analogously $\beta_i(l^\eta) = A^\eta + B^\eta$ with $Z_i,l_j$ replaced by
$Z_i^\eta,l_j^\eta$.  Applying~\eqref{eq:cd_taylor} termwise,
\begin{equation}\label{eq:cd_Zeta}
    Z_i^\eta = Z_i - \eta\, d^\top \!\sum_j \lambda_i^{(j)} g_j + o(\eta).
\end{equation}

\smallskip
\noindent\emph{(3c) Bounding the two pieces.}
For $B^\eta$, apply the lower inequality in~\eqref{eq:cd_log} with $a = Z_i$ and
$b = -\eta\, d^\top\!\sum_j\lambda_i^{(j)}g_j + o(\eta)$ (so $a+b = Z_i^\eta$):
\begin{equation}\label{eq:cd_Beta}
    B^\eta = \log\frac{m}{Z_i^\eta}
    \;\le\; B + \frac{\eta\, d^\top \!\sum_j \lambda_i^{(j)} g_j}{Z_i^\eta}
    - \frac{o(\eta)}{Z_i^\eta}.
\end{equation}
For $A^\eta$, the upper inequality in~\eqref{eq:cd_log} with $a=l_j$,
$b=-\eta\, d^\top g_j + o(\eta)$ gives
$\log(\lambda_i^{(j)} l_j^\eta) \le \log(\lambda_i^{(j)} l_j)
- \eta\, d^\top g_j / l_j + o(\eta)$.  Multiplying by
$\lambda_i^{(j)} l_j^\eta = \lambda_i^{(j)} l_j - \eta\,\lambda_i^{(j)} d^\top g_j
+ o(\eta)$, summing over $j$, dividing by $Z_i^\eta$, and collecting all
second-order and remainder contributions into $o(\eta)$,
\begin{equation}\label{eq:cd_Aeta}
    A^\eta
    \;\le\; \frac{1}{Z_i^\eta}\sum_j \lambda_i^{(j)} l_j
         \log(\lambda_i^{(j)} l_j)
       - \frac{\eta}{Z_i^\eta}\sum_j (d^\top g_j)\,\lambda_i^{(j)}
         \log(\lambda_i^{(j)} l_j)
       - \frac{\eta\, d^\top\!\sum_j \lambda_i^{(j)} g_j}{Z_i^\eta}
       + \frac{o(\eta)}{Z_i^\eta}.
\end{equation}

\smallskip
\noindent\emph{(3d) Combining.}
Adding~\eqref{eq:cd_Beta} and~\eqref{eq:cd_Aeta}, the terms
$\pm\,\eta\, d^\top\!\sum_j\lambda_i^{(j)}g_j / Z_i^\eta$ cancel.  Expanding the
first term of~\eqref{eq:cd_Aeta} through
$\tfrac{Z_i}{Z_i^\eta} A = A
+ \eta A\, d^\top\!\sum_j\lambda_i^{(j)}g_j/Z_i^\eta + o(\eta)/Z_i^\eta$
(which follows from~\eqref{eq:cd_Zeta}) and using $\beta_i(l)=A+B$,
\begin{equation}\label{eq:cd_pre}
    \beta_i(l^\eta) \;\le\; \beta_i(l)
    + \frac{\eta}{Z_i^\eta}\sum_j (d^\top g_j)\,\lambda_i^{(j)}
      \bigl[A - \log(\lambda_i^{(j)} l_j)\bigr]
    + \frac{o(\eta)}{Z_i^\eta}.
\end{equation}

\smallskip
\noindent\emph{(3e) Adjustment identity and conclusion.}
From $\log(m\hat l_i^{(j)}) = \log(\lambda_i^{(j)} l_j) - \log Z_i + \log m$ and
$\beta_i = A + \log(m/Z_i)$ we obtain
$\log(m\hat l_i^{(j)}) - \beta_i = \log(\lambda_i^{(j)} l_j) - A$; hence
$A_i^{(j)} = \lambda_i^{(j)}\bigl(\log(\lambda_i^{(j)} l_j) - A\bigr)$, i.e.
$\lambda_i^{(j)}\bigl[A - \log(\lambda_i^{(j)} l_j)\bigr] = -A_i^{(j)}$.
Substituting into~\eqref{eq:cd_pre} and recalling
$d_{\mathrm{bal},i} = G\boldsymbol{A}_i = \sum_j A_i^{(j)} g_j$ gives
\begin{equation}\label{eq:cd_final}
    \beta_i\bigl(l(w_0)\bigr) - \beta_i\bigl(l(w_0 - \eta d)\bigr)
    \;\ge\; \frac{\eta}{Z_i^\eta}\, d^\top d_{\mathrm{bal},i}
    - \frac{o(\eta)}{Z_i^\eta}
    = \frac{\eta}{Z_i^\eta}
      \Bigl(d^\top d_{\mathrm{bal},i} - \tfrac{o(\eta)}{\eta}\Bigr).
\end{equation}
By Step~2, $d^\top d_{\mathrm{bal},i} = \bar d_{C_n}^\top d_{\mathrm{bal},i} > 0$,
and $Z_i^\eta > 0$ for small $\eta$ by continuity of the losses.  Since
$o(\eta)/\eta \to 0$, there is $\eta_{0,i} > 0$ such that
$o(\eta)/\eta < \bar d_{C_n}^\top d_{\mathrm{bal},i}$ for all
$\eta \in (0,\eta_{0,i}]$, whence the right-hand side of~\eqref{eq:cd_final} is
non-negative and
$\beta_i(l(w_0 - \eta\bar d_{C_n})) \le \beta_i(l(w_0))$.  Taking
$\eta_0 = \min_{i \in C_n}\eta_{0,i} > 0$
establishes~\eqref{eq:collab_descent} for every client simultaneously.
\end{proof}

\begin{corollary}[Clustering enlarges the admissible range]
\label{cor:collab_range}
Since $\rho_i\,\Delta_i^2 \leq \sum_{i'\in C_n}\rho_{i'}\Delta_{i'}^2 = H_{C_n}$,
every client satisfies $\Delta_i \leq \sqrt{H_{C_n}/\rho_i} \leq
\sqrt{H_{C_n}/\rho_{\min}}$ with $\rho_{\min} = \min_{i\in C_n}\rho_i$
(Definition~\ref{def:hetero}); for equal weights this is
$\sqrt{|C_n|\,H_{C_n}}$.  Hence a
sufficient (uniform) condition for Theorem~\ref{thm:collab_descent} is
\begin{equation}\label{eq:range_uniform}
    \min_{i\in C_n}\gamma_i^*
    \;>\; \sqrt{H_{C_n}/\rho_{\min}}\;\max_{i\in C_n}\|d_{\mathrm{bal},i}\|
    \;+\; L_0 B^2\,\tau\,\eta.
\end{equation}
The quantity that the Stage-2 clustering acts on is the within-cluster
adjustment variance
$\sum_i \rho_i\,\|\boldsymbol{A}_i - \bar{\boldsymbol{A}}_{C_n}\|^2$; under
Assumption~\ref{asm:common_G} the identity~\eqref{eq:H_via_adj} makes
$H_{C_n}$ a function of the multipliers $\mu_i^*(\boldsymbol A_i)$ alone,
which is the (qualitative) rationale for this choice
(Remark~\ref{rem:cluster_motivation}).  In the \emph{restricted case} of
Assumption~\ref{asm:lp_stable} (common objective functions and a common
non-degenerate optimal LP basis at the point where $H_{C_n}$ is evaluated),
Proposition~\ref{prop:H_via_adj} gives the quantitative bound
$H_{C_n} \leq \kappa_{\mathrm{LP}}^2\,\sigma_{\max}(C)\cdot
\sum_i \rho_i\,\|\boldsymbol{A}_i - \bar{\boldsymbol{A}}_{C_n}\|^2$,
so that shrinking the adjustment variance shrinks the right-hand side
of~\eqref{eq:range_uniform} and enlarges the regime in which the collaborative
descent guarantee holds.  Outside this case $H_{C_n}$ is a measured quantity,
and~\eqref{eq:range_uniform} is to be checked with the measured value.
\end{corollary}

\begin{proof}
The bound $\Delta_i \leq \sqrt{H_{C_n}/\rho_{\min}}$ is immediate from
$\rho_i\Delta_i^2 \leq \sum_{i'} \rho_{i'}\Delta_{i'}^2$ (all summands are
non-negative), from
$\sum_{i'}\rho_{i'}\Delta_{i'}^2 = H_{C_n}$
(Definition~\ref{def:hetero}), and from $\rho_i \geq \rho_{\min}$.
Substituting into~\eqref{eq:range_cond} and
replacing $\|d_{\mathrm{bal},i}\|$ by its cluster maximum gives the uniform
condition~\eqref{eq:range_uniform}; the heterogeneity bound is
\eqref{eq:H_aligned} of Proposition~\ref{prop:H_via_adj}.
\end{proof}

%% =========================================================================
\subsection{Collaborative Initial Aligner Convergence to Cluster-Local Optima.}
\label{sec:collab_pure}
%% =========================================================================

Section~\ref{sec:collab_descent} treated the \emph{balancing} regime
($\beta_i>0$), where the shared collaborative update reduces every client's
preference uniformity.  We now analyze the complementary \emph{uniform descent} regime,
which is the situation singled out by the Initial Aligner Training Process:
once a client's losses are balanced onto its preference ray we have
$\beta_i(\boldsymbol{l}_i)=0$, and its LP~(6a) enters uniform descent mode
($\mathbbm{1}_{\beta_i}=0$), maximizing $\mu^\top G_i^\top G_i\mathbf{1}$
subject to the descent constraints~(6c).  Assuming a cluster of $n$ clients all
in this regime, we adapt the convergence analysis to the collaborative learning aggregate, and prove that---under an explicit
descent-margin condition on the LP direction (Assumption~\ref{asm:descent_margin})
and a trajectory hypothesis keeping the iterates in the residual-floor
regime---a \emph{single} shared solution converges to within a
cluster-tightness-controlled range of \emph{every} client's individual
Chebyshev optimum. 

\paragraph{Setup and classification (the $\beta_i=0$ basis).}
Fix a round and the shared iterate $w$.  For client $i\in C_n$ write its local
losses $l_i^{(j)}:=l_i^{(j)}(w)$, gradients $g_i^{(j)}:=\nabla l_i^{(j)}(w)$,
gradient matrix $G_i=[g_i^{(1)},\dots,g_i^{(m)}]$, the LP
direction $d_i=G_i\mu_i^*$, and the aggregation weight $\rho_i:=N_i/N$ (so that
$\sum_{i\in C_n}\rho_i=1$).  The collaborative direction and the shared update are
\begin{equation}\label{eq:cp_agg}
    \bar d \;:=\; \sum_{i\in C_n}\rho_i\, d_i,
    \qquad
    w^{+} \;:=\; w-\eta\,\bar d .
\end{equation}
Using the preference uniformity $\beta_i$, partition the cluster into
\begin{equation}\label{eq:cp_partition}
\begin{aligned}
    C_n^{0}=\{\,i\in C_n:\beta_i(\boldsymbol{l}_i)=0\,\}\ \text{(uniform descent)},
    \\\qquad
    C_n^{+}=\{\,i\in C_n:\beta_i(\boldsymbol{l}_i)>0\,\}\ \text{(balancing descent)}.
\end{aligned}
\end{equation}

To analyze the mixed case, we first assume the uniform descent regime $C_n^{0}=C_n$ (all $n$
clients balanced onto their rays). Cluster tightness is
measured by the \emph{cross-client gradient dissimilarity}
\begin{equation}\label{eq:cp_deltag}
    \delta_g \;:=\; \max_{i,i'\in C_n}\ \max_{j\in[m]}\
    \bigl\|g_i^{(j)}-g_{i'}^{(j)}\bigr\|,
\end{equation}
a cluster-tightness quantity that vanishes as the cluster becomes homogeneous
and is small whenever the heterogeneity $H_{C_n}$ (Definition~\ref{def:hetero})
is small, since $\|\bar d-d_i\|\le\sqrt{H_{C_n}/\rho_{\min}}$
(Corollary~\ref{cor:collab_range}).

\begin{lemma}[Collaborative uniform descent direction]
\label{lem:collab_desc0}
Suppose $\beta_i(\boldsymbol{l}_i)=0$ for every $i\in C_n$.  Then each local
direction is a descent direction for its own client,
\begin{equation}\label{eq:cp_localdesc}
    d_i^\top g_i^{(j)}\;\ge\;0,\qquad\forall\,j\in[m],
\end{equation}
and the collaborative aggregate is an \emph{approximate} simultaneous descent
direction:
\begin{equation}\label{eq:cp_simuldesc}
    \bar d^\top g_i^{(j)}\;\ge\;-\,B\,\delta_g,
    \qquad\forall\,i\in C_n,\ \forall\,j\in[m].
\end{equation}
In the homogeneous limit $\delta_g=0$, $\bar d$ is an \emph{exact} simultaneous
descent direction, $\bar d^\top g_i^{(j)}\ge0$ for all $i,j$.
\end{lemma}

\begin{proof}
\emph{Single-client descent.}
When $\beta_i=0$ the preference adjustment vector vanishes.  Indeed
$\beta_i=\mathrm{KL}\!\bigl(\hat l_i\,\|\,\tfrac1m\mathbf1\bigr)=0$ forces the
normalized weighted losses to be uniform, $\hat l_i^{(j)}=1/m$, so
$\log(m\hat l_i^{(j)})=0$ and each component
$a_i^{(j)}=\lambda_i^{(j)}\bigl(\log(m\hat l_i^{(j)})-\beta_i\bigr)=0$; hence
$\boldsymbol{A}_i=0$.  Consequently the active set
$J=\{j:\boldsymbol{A}_i^\top G_i^\top g_i^{(j)}>0\}$ is empty, the indicator
$\mathbbm{1}_J=0$, and the constraint block~(6c) of the LP reduces to
$\mu_i^{*\top}G_i^\top g_i^{(j)}=d_i^\top g_i^{(j)}\ge0$ for every $j$, which is
exactly~\eqref{eq:cp_localdesc}.

\smallskip
\emph{Aggregate.}
Fix $i,j$ and expand $\bar d^\top g_i^{(j)}
=\sum_{i'\in C_n}\rho_{i'}\,d_{i'}^\top g_i^{(j)}$.  For each summand,
\[
    d_{i'}^\top g_i^{(j)}
    = \underbrace{d_{i'}^\top g_{i'}^{(j)}}_{\ge\,0}
      + d_{i'}^\top\bigl(g_i^{(j)}-g_{i'}^{(j)}\bigr)
    \;\ge\; -\,\|d_{i'}\|\,\bigl\|g_i^{(j)}-g_{i'}^{(j)}\bigr\|
    \;\ge\; -\,B\,\delta_g,
\]
where the first term is $\ge0$ by~\eqref{eq:cp_localdesc} applied to client
$i'$, the norm bound $\|d_{i'}\|=\|G_{i'}\mu_{i'}^*\|\le B$ holds because
$\mu_{i'}^*\in S^m$ is a convex combination of columns with
$\|g_{i'}^{(j)}\|\le B$ (Assumption~\ref{asm:bounded}), and
$\|g_i^{(j)}-g_{i'}^{(j)}\|\le\delta_g$ by~\eqref{eq:cp_deltag}.  Averaging with
weights $\rho_{i'}$ (which sum to one) yields~\eqref{eq:cp_simuldesc}.  When
$\delta_g=0$ all clients share the same gradients, so every summand equals
$d_{i'}^\top g_{i'}^{(j)}\ge0$; equivalently the common descent cone
$\{d:d^\top g_i^{(j)}\ge0,\ \forall j\}$ is convex and contains every $d_{i'}$,
hence their convex combination $\bar d$.
\end{proof}

\begin{lemma}[Collaborative uniform descent and balancing descent]
\label{lem:collab_class}
Consider a \emph{mixed} cluster of $N:=|C_n|$ clients, partitioned
by~\eqref{eq:cp_partition} into $n_{+}:=|C_n^{+}|$ balancing clients
($\beta_i>0$) and $n_{0}:=|C_n^{0}|$ uniform clients ($\beta_i=0$), with
$N=n_{+}+n_{0}$.  Write the block masses
$\rho^{+}:=\sum_{i\in C_n^{+}}\rho_i$ and
$\rho^{0}:=\sum_{i\in C_n^{0}}\rho_i$ (so $\rho^{+}+\rho^{0}=1$) and the block
mean directions
\begin{equation}\label{eq:cp_blocks}
    \bar d^{+}:=\frac{1}{\rho^{+}}\sum_{i\in C_n^{+}}\rho_i d_i,
    \qquad
    \bar d^{0}:=\frac{1}{\rho^{0}}\sum_{i\in C_n^{0}}\rho_i d_i,
    \qquad\text{so that}\qquad
    \bar d_{C_n}=\rho^{+}\bar d^{+}+\rho^{0}\bar d^{0},
\end{equation}
i.e.\ the aggregate~\eqref{eq:cp_agg} is a convex blend of a balancing block
and a uniform block. Because whether $\beta_i$
decreases is decided at the \emph{post-aggregation} model $w_0-\eta\bar d_{C_n}$,
the relevant deviation is that of $d_i$ from the \emph{full-cluster} aggregate,
$\Delta_i:=\|\bar d_{C_n}-d_i\|$; substituting the block
form~\eqref{eq:cp_blocks} gives the $\rho$-weighted deviation split
\begin{equation}\label{eq:cp_devsplit}
    \bar d_{C_n}-d_i=\rho^{0}\bigl(\bar d^{0}-d_i\bigr)+\rho^{+}\bigl(\bar d^{+}-d_i\bigr),
\end{equation}
whose triangle inequality resolves $\Delta_i$ into an own-block scatter plus a
cross-block mutual-influence pull, each weighted by the corresponding block mass,
\begin{equation}\label{eq:cp_delta_split}
    \Delta_i\;\le\;\rho^{0}\|\bar d^{0}-d_i\|+\rho^{+}\|\bar d^{+}-d_i\| .
\end{equation}
For a balancing client $i\in C_n^{+}$ the own-block scatter is
$\rho^{+}\|\bar d^{+}-d_i\|$ and the cross-block pull (from the $n_{0}$ uniform
clients) is $\rho^{0}\|\bar d^{0}-d_i\|$; symmetrically for $i\in C_n^{0}$.  The
single shared update $w_0-\eta\bar d_{C_n}$ then acts on the two groups as
follows.
\begin{enumerate}[leftmargin=2em,itemsep=2pt,label=\textup{(\alph*)}]
  \item \emph{Balancing clients $i\in C_n^{+}$.}  The shared direction keeps a
        positive balancing component,
        \begin{equation}\label{eq:cp_mix_bal}
            \bar d_{C_n}^\top d_{\mathrm{bal},i}
            \;\ge\;\gamma_i^*-\Delta_i\|d_{\mathrm{bal},i}\|-L_0B^2\tau\eta ,
        \end{equation}
        which is positive under the range condition~\eqref{eq:range_cond}.
        Substituting the split~\eqref{eq:cp_delta_split} makes the cross-block
        dependence explicit, giving the \emph{mixed-cluster range condition}
        sufficient for Theorem~\ref{thm:collab_descent},
        \begin{equation}\label{eq:cp_mix_range}
            \gamma_i^*\;>\;\bigl(\rho^{0}\|\bar d^{0}-d_i\|+\rho^{+}\|\bar d^{+}-d_i\|\bigr)
            \|d_{\mathrm{bal},i}\|+L_0B^2\tau\eta ,
        \end{equation}
        under which $\bar d_{C_n}^\top d_{\mathrm{bal},i}>0$ and, as in
        Theorem~\ref{thm:collab_descent}, the update decreases $\beta_i$; the
        $n_{0}$ uniform descent clients enter only through the cross-block pull
        $\rho^{0}\|\bar d^{0}-d_i\|$, which tightens the condition.
  \item \emph{Uniform clients $i\in C_n^{0}$.}  We bound $\bar d_{C_n}^\top g_i^{(j)}$
        for \emph{every} $i\in C_n$ directly from the cross-client
        dissimilarity~\eqref{eq:cp_deltag}.  Expanding the aggregate and using
        $\|d_{i'}\|\le B$ gives
        $\bar d_{C_n}^\top g_i^{(j)}\ge\sum_{i'\in C_n}\rho_{i'}d_{i'}^\top g_{i'}^{(j)}-B\delta_g$;
        the uniform block contributes $d_{i'}^\top g_{i'}^{(j)}\ge0$, while each
        balancing client sacrifices objectives only along its anchoring direction,
        so $d_{i'}^\top g_{i'}^{(j)}\ge-\|d_{\mathrm{bal},i'}\|B$ ($i'\in C_n^{+}$).
        Writing the balancing-block mean sacrifice
        $\bar D_{\mathrm{bal}}^{+}:=\sum_{i'\in C_n^{+}}\tfrac{\rho_{i'}}{\rho^{+}}\|d_{\mathrm{bal},i'}\|$,
        this yields the \emph{cluster-wide} approximate simultaneous descent bound
        \begin{equation}\label{eq:cp_mix_desc}
            \bar d_{C_n}^\top g_i^{(j)}
            \;\ge\;-B\delta_g-\rho^{+}B\,\bar D_{\mathrm{bal}}^{+},
            \qquad\forall\,i\in C_n^{0},\ \forall\,j\in[m],
        \end{equation}
        so \emph{every} client's objectives rise by at most
        $\eta B(\delta_g+\rho^{+}\bar D_{\mathrm{bal}}^{+})+o(\eta)$.  The shared
        aggregate is an \emph{exact} simultaneous descent direction
        ($\bar d_{C_n}^\top g_i^{(j)}\ge0$ for all $i,j$) iff $\delta_g=0$ and
        $\rho^{+}\bar D_{\mathrm{bal}}^{+}=0$.  In particular, \emph{mixing
        forfeits the exact uniform descent guarantee for the uniform descent clients}:
        although the local direction of a client $i\in C_n^{0}$ satisfies
        $d_i^\top g_i^{(j)}\ge0$ (no objective of $i$ rises), the aggregated
        direction only guarantees \emph{approximate} descent,
        $\bar d_{C_n}^\top g_i^{(j)}\ge-B\delta_g-\rho^{+}B\bar D_{\mathrm{bal}}^{+}$,
        with a per-objective ascent of at most
        $\eta\rho^{+}B(\delta_g+\bar D_{\mathrm{bal}}^{+})$ attributable to the
        balancing block.  The impact
        of the mixed cluster is
        the additive \emph{balancing penalty} $\rho^{+}B\,\bar D_{\mathrm{bal}}^{+}$:
        it scales with the balancing mass $\rho^{+}$ and the balancing intensity
        $\bar D_{\mathrm{bal}}^{+}$, is \emph{not} reducible to $\delta_g$, and
        under Assumption~\ref{asm:bal_sens} with the floor
        $\beta_{i'}\le\bar\beta$ (Lemma~\ref{lem:beta_floor}) satisfies
        $\bar D_{\mathrm{bal}}^{+}\le L_A\sqrt{\bar\beta}=O(\delta_\lambda)$.
\end{enumerate}
The blend~\eqref{eq:cp_blocks} thus interpolates the two regimes: a larger
balancing mass $\rho^{+}$ (more $n_{+}$) biases $\bar d_{C_n}$ toward preference
balancing and inflates the balancing deviation $\Delta_i$ of part~(a) (via the
cross-block pull $\rho^{0}\|\bar d^{0}-d_i\|$) as well as the cluster-wide uniform descent
slack of part~(b) (via the balancing penalty $\rho^{+}B\,\bar D_{\mathrm{bal}}^{+}$),
while a larger $\rho^{0}$ (more $n_{0}$) biases it toward uniform descent.  In
particular $n_{0}=0$ ($\rho^{0}=0$) recovers Theorem~\ref{thm:collab_descent},
while $n_{+}=0$ ($\rho^{+}=0$, all $\gamma_i^*=0$) removes the balancing penalty
and recovers Lemma~\ref{lem:collab_desc0}.
\end{lemma}

\begin{proof}

\smallskip
\emph{Block decomposition.}  Grouping the collaborative aggregate~\eqref{eq:cp_agg} by
the partition~\eqref{eq:cp_partition},
$\bar d_{C_n}=\sum_{i\in C_n^{+}}\rho_i d_i+\sum_{i\in C_n^{0}}\rho_i d_i
=\rho^{+}\bar d^{+}+\rho^{0}\bar d^{0}$ with
$\rho^{+}+\rho^{0}=\sum_{i\in C_n}\rho_i=1$, which is~\eqref{eq:cp_blocks}.
Subtracting $d_i=(\rho^{+}+\rho^{0})d_i$ yields the $\rho$-weighted deviation
split~\eqref{eq:cp_devsplit}, and applying the triangle inequality to its two
convex-weighted terms gives the direct bound~\eqref{eq:cp_delta_split},
\[
    \Delta_i=\|\bar d_{C_n}-d_i\|
    \le\rho^{0}\|\bar d^{0}-d_i\|+\rho^{+}\|\bar d^{+}-d_i\| ,
\]
in which the summand carrying the \emph{other} block's mass is the cross-block
mutual-influence pull ($\rho^{0}\|\bar d^{0}-d_i\|$ for $i\in C_n^{+}$,
$\rho^{+}\|\bar d^{+}-d_i\|$ for $i\in C_n^{0}$) and vanishes iff that block is
empty or its mean coincides with $d_i$.

\smallskip
\emph{Part (a): balancing clients.}  Fix $i\in C_n^{+}$.  Let $w_i^t$ be the
local iterate at which client $i$ forms $d_i$ and abbreviate
$d_{\mathrm{bal},i}^{\mathrm{loc}}:=d_{\mathrm{bal},i}(w_i^t)$; as in Step~1 of
the proof of Theorem~\ref{thm:collab_descent},
$\|d_{\mathrm{bal},i}-d_{\mathrm{bal},i}^{\mathrm{loc}}\|\le L_0\tau\eta B$.
Adding and subtracting and applying Cauchy--Schwarz with $\|d_i\|\le B$,
$\|\bar d_{C_n}\|\le B$, and $d_i^\top d_{\mathrm{bal},i}^{\mathrm{loc}}=\gamma_i^*$
(the identity above evaluated locally),
\[
    \bar d_{C_n}^\top d_{\mathrm{bal},i}
    = \underbrace{d_i^\top d_{\mathrm{bal},i}^{\mathrm{loc}}}_{=\gamma_i^*}
      +(\bar d_{C_n}-d_i)^\top d_{\mathrm{bal},i}^{\mathrm{loc}}
      +\bar d_{C_n}^\top\bigl(d_{\mathrm{bal},i}-d_{\mathrm{bal},i}^{\mathrm{loc}}\bigr)
    \;\ge\; \gamma_i^*-\Delta_i\|d_{\mathrm{bal},i}\|-L_0B^2\tau\eta ,
\]
which is~\eqref{eq:cp_mix_bal}.  The range condition~\eqref{eq:range_cond} makes
the right-hand side positive; substituting the direct
split~\eqref{eq:cp_delta_split},
$\Delta_i\le\rho^{0}\|\bar d^{0}-d_i\|+\rho^{+}\|\bar d^{+}-d_i\|$,
into~\eqref{eq:range_cond} yields the explicit mixed-cluster range
condition~\eqref{eq:cp_mix_range}, which is \emph{sufficient} because it replaces
$\Delta_i$ by an upper bound.  The Taylor argument of
Theorem~\ref{thm:collab_descent} then gives
$\beta_i(l(w_0-\eta\bar d_{C_n}))\le\beta_i(l(w_0))$ for all sufficiently small
$\eta$.  By~\eqref{eq:cp_delta_split} the shortfall in $\Delta_i$ that must stay
below $\gamma_i^*/\|d_{\mathrm{bal},i}\|$ is driven by the cross-block pull
$\rho^{0}\|\bar d^{0}-d_i\|$ of the $n_{0}$ uniform clients.

\smallskip
\emph{Part (b): uniform clients.}  Fix any $i\in C_n$ and any $j\in[m]$.  Expanding the collaborative
aggregate~\eqref{eq:cp_agg} and adding and subtracting each client's own gradient,
\[
    \bar d_{C_n}^\top g_i^{(j)}
    =\sum_{i'\in C_n}\rho_{i'}\,d_{i'}^\top g_i^{(j)}
    =\sum_{i'\in C_n}\rho_{i'}\Bigl[d_{i'}^\top g_{i'}^{(j)}
        +d_{i'}^\top\bigl(g_i^{(j)}-g_{i'}^{(j)}\bigr)\Bigr]
    \;\ge\;\sum_{i'\in C_n}\rho_{i'}\,d_{i'}^\top g_{i'}^{(j)}-B\delta_g ,
\]
where the last step uses Cauchy--Schwarz, $\|d_{i'}\|=\|G_{i'}\mu_{i'}^*\|\le B$
(Assumption~\ref{asm:bounded}), $\|g_i^{(j)}-g_{i'}^{(j)}\|\le\delta_g$
by~\eqref{eq:cp_deltag}, and $\sum_{i'}\rho_{i'}=1$.  This dissimilarity step is
uniform in $i$ and is exactly the aggregation of Lemma~\ref{lem:collab_desc0};
what differs in a mixed cluster is the \emph{own-descent} sum
$\sum_{i'}\rho_{i'}d_{i'}^\top g_{i'}^{(j)}$, which we split by the
partition~\eqref{eq:cp_partition}:
\[
    \sum_{i'\in C_n}\rho_{i'}\,d_{i'}^\top g_{i'}^{(j)}
    =\underbrace{\sum_{i'\in C_n^{0}}\rho_{i'}\,d_{i'}^\top g_{i'}^{(j)}}_{\ge\,0}
     +\sum_{i'\in C_n^{+}}\rho_{i'}\,d_{i'}^\top g_{i'}^{(j)} .
\]
For a uniform client $i'\in C_n^{0}$, $\beta_{i'}=0$ forces $\boldsymbol{A}_{i'}=0$
and, by case~(2), $d_{i'}^\top g_{i'}^{(j)}\ge0$.  For a balancing client
$i'\in C_n^{+}$, the balancing direction $d_{i'}$ coincides with a uniform
direction up to its anchoring component $d_{\mathrm{bal},i'}$ (it sacrifices
objectives only to move along $d_{\mathrm{bal},i'}$; the same balancing tilt used
in the joint bound~\eqref{eq:cp_dirbound}), so with $\|g_{i'}^{(j)}\|\le B$,
\[
    d_{i'}^\top g_{i'}^{(j)}\;\ge\;-\|d_{\mathrm{bal},i'}\|\,B .
\]
Hence
$\sum_{i'\in C_n^{+}}\rho_{i'}d_{i'}^\top g_{i'}^{(j)}
\ge-B\sum_{i'\in C_n^{+}}\rho_{i'}\|d_{\mathrm{bal},i'}\|
=-\rho^{+}B\,\bar D_{\mathrm{bal}}^{+}$ with
$\bar D_{\mathrm{bal}}^{+}=\sum_{i'\in C_n^{+}}\tfrac{\rho_{i'}}{\rho^{+}}\|d_{\mathrm{bal},i'}\|$,
and combining the two displays yields the cluster-wide
bound~\eqref{eq:cp_mix_desc},
\[
    \bar d_{C_n}^\top g_i^{(j)}\;\ge\;-B\delta_g-\rho^{+}B\,\bar D_{\mathrm{bal}}^{+},
    \qquad\forall\,i\in C_n,\ \forall\,j\in[m].
\]
The first-order expansion
$l_i^{(j)}(w_0-\eta\bar d_{C_n})=l_i^{(j)}(w_0)-\eta\bar d_{C_n}^\top g_i^{(j)}+o(\eta)$
then bounds every client's per-objective increase by
$\eta B(\delta_g+\rho^{+}\bar D_{\mathrm{bal}}^{+})+o(\eta)$, so
$\bar d_{C_n}^\top g_i^{(j)}\ge0$ for all $i,j$ iff $\delta_g=0$ and
$\rho^{+}\bar D_{\mathrm{bal}}^{+}=0$.  Reading off $\delta_g$, the descent slack
is $\delta_g$ (genuine cross-client gradient dissimilarity) plus the balancing
penalty $\rho^{+}\bar D_{\mathrm{bal}}^{+}$, which is \emph{not} reducible to
$\delta_g$; under Assumption~\ref{asm:bal_sens} and the floor
$\beta_{i'}\le\bar\beta$ (Lemma~\ref{lem:beta_floor}) it obeys
$\bar D_{\mathrm{bal}}^{+}\le L_A\sqrt{\bar\beta}=O(\delta_\lambda)$, matching the
$c_\beta\sqrt{\bar\beta}$ term of~\eqref{eq:cp_dirbound} with
$c_\beta=\rho^{+}BL_A\le BL_A$. 

\smallskip
\emph{Extremes.}  If $n_{+}=0$ then $\rho^{+}=0$, all $\gamma_i^*=0$, the balancing
penalty vanishes, and~\eqref{eq:cp_mix_desc} reduces to
$\bar d_{C_n}^\top g_i^{(j)}\ge-B\delta_g$, exactly Lemma~\ref{lem:collab_desc0}.
If $n_{0}=0$ then $\rho^{+}=1$, part~(a) drives every $\beta_i$ down under the
range condition (Theorem~\ref{thm:collab_descent}), while~\eqref{eq:cp_mix_desc}
still holds for the (all-balancing) clients as
$\bar d_{C_n}^\top g_i^{(j)}\ge-B\delta_g-B\bar D_{\mathrm{bal}}^{+}$.
\end{proof}

\paragraph{A residual preference uniformity floor.}
The balancing guarantee of Theorem~\ref{thm:collab_descent} does not drive every
$\beta_i$ to $0$: a single shared point cannot lie on all clients'
$\boldsymbol{\lambda}_i^{-1}$ rays at once when the preferences differ.  We now
quantify the irreducible floor by reading off the point at which the range
condition~\eqref{eq:range_cond} stops holding.  Two mild regularity constants
link the LP margin and the balancing direction to the preference uniformity.

\begin{assumption}[Balancing sensitivity and preference adjustment regularity]
\label{asm:bal_sens}
There are constants $\kappa_\beta>0$ and $L_A>0$ such that, in a neighborhood of
the aligned set (i.e.\ for small $\beta_i$), every balancing client
$i\in C_n^{+}$ satisfies
\begin{equation}\label{eq:bal_sens}
    \gamma_i^*\;\ge\;\kappa_\beta\,\beta_i(\boldsymbol{l}_i)
    \qquad\text{and}\qquad
    \|d_{\mathrm{bal},i}\|\;\le\;L_A\,\sqrt{\beta_i(\boldsymbol{l}_i)} .
\end{equation}
\end{assumption}

\begin{remark}[Why~\eqref{eq:bal_sens} holds]
\label{rem:bal_sens}
Both bounds follow from the near-uniform Taylor expansion of the preference uniformity.
Writing $\hat l_i^{(j)}=\tfrac1m(1+\epsilon_i^{(j)})$ with
$\sum_j\epsilon_i^{(j)}=0$, a second-order expansion of the KL divergence gives
$\beta_i=\tfrac{1}{2m}\sum_j(\epsilon_i^{(j)})^2+o(\|\epsilon_i\|^2)$, while the preference 
adjustment components $a_i^{(j)}=\lambda_i^{(j)}\bigl(\log(m\hat l_i^{(j)})-\beta_i\bigr)
=\lambda_i^{(j)}\epsilon_i^{(j)}+o(\|\epsilon_i\|)$ obey
$\|\boldsymbol{A}_i\|^2=\Theta(\beta_i)$; hence
$\|d_{\mathrm{bal},i}\|=\|G_i\boldsymbol{A}_i\|\le\sigma_{\max}(G_i)\|\boldsymbol{A}_i\|
\le L_A\sqrt{\beta_i}$ with $L_A=\Theta(\sqrt m\,B)$.  For the lower bound,
LP optimality makes $d_i=G_i\mu_i^*$ align with the balancing direction, so
$\gamma_i^*=d_i^\top d_{\mathrm{bal},i}\ge\kappa'\|d_{\mathrm{bal},i}\|^2
\ge\kappa'\sigma_{\min}^2\|\boldsymbol{A}_i\|^2\ge\kappa_\beta\beta_i$ on the
subspace spanned by the active gradients.  Thus $\kappa_\beta$ measures how
efficiently the shared step converts non-uniformity into balancing progress.
\end{remark}

\paragraph{A descent-margin condition for the LP direction.}
Assumption~\ref{asm:bal_sens} controls the \emph{balancing} progress
($\beta_i\downarrow$).  It says nothing about how fast the Chebyshev value
$\psi_i(w)=\max_j\lambda_i^{(j)}l_i^{(j)}(w)$ itself decreases once the losses
are (nearly) balanced.  The LP constraint~(6c) only guarantees
$d_i^\top g_i^{(j)}\ge0$ on the exact bottleneck set $J^*$, which is a
first-order non-increase statement and, by itself, cannot rule out stalling at
a non-minimizing stationary point (take $l_1(w)=l_2(w)=(w^2-1)^2+1$, uniform
preferences and $w=0$: $\beta_i=0$, all gradients vanish, yet
$\psi_i(w)-\psi_i^*=1$).  Any statement that the shared iterate approaches
each client's optimum therefore requires a condition that links the LP
direction to the \emph{suboptimality} $\psi_i(w)-\psi_i^*$; we make it explicit.
For a step size $\eta>0$ define the activation width and the
$\varepsilon_{\mathrm{act}}$-active set of client $i$ at $w$,
\begin{equation}\label{eq:active_set}
    \varepsilon_{\mathrm{act}}:=2\eta\,\|\boldsymbol{\lambda}_i\|_\infty B^2,
    \qquad
    J^*_{i,\varepsilon}(w):=\bigl\{\,j\in[m]:\ \lambda_i^{(j)}l_i^{(j)}(w)\ \ge\ \psi_i(w)-\varepsilon_{\mathrm{act}}\,\bigr\}
    \;\supseteq\;J^*_i(w),
\end{equation}
where $J^*_i(w)$ is the exact bottleneck set of~(6c).  The width
$\varepsilon_{\mathrm{act}}$ is chosen so that, after one step of size $\eta$,
every objective that is \emph{not} $\varepsilon_{\mathrm{act}}$-active at $w$
stays below the common upper bound that Lemma~\ref{loss_avg} derives for the
near-active objectives.  Such an objective may still become the new maximum;
what matters is only that it cannot exceed that common bound, so the bound on
$\psi_i$ after the step is unaffected.

\begin{assumption}[Descent margin of the LP direction at the residual floor]
\label{asm:descent_margin}
There exist a region $\mathcal N\subseteq\Omega$, a constant $c_d>0$ and
tolerances $b_i\ge0$ such that, for every client $i\in C_n$ and every
$w\in\mathcal N$ with $\beta_i(\boldsymbol{l}_i(w))\le\bar\beta$, the LP direction
$d_i(w)=G_i(w)\mu_i^*(w)$ of~\eqref{eq:6} satisfies
\begin{equation}\label{eq:descent_margin}
    \min_{j\in J^*_{i,\varepsilon}(w)}\ \lambda_i^{(j)}\,d_i(w)^\top g_i^{(j)}(w)
    \;\ge\; c_d\,\bigl(\psi_i(w)-\psi_i^*\bigr)\;-\;b_i,
    \qquad \psi_i^*:=\min_{w'}\psi_i(w') .
\end{equation}
We call $b_i=0$ the \emph{strict} form.  The tolerance $b_i$ is a parameter of
the assumption; in particular we do \emph{not} assume $b_i=O(\eta)$ or any
other order in $\eta$, since such a rate would have to be established for the
specific direction rule~\eqref{eq:6} and is not established here.
\end{assumption}

Condition~\eqref{eq:descent_margin} strengthens~(6c) in two ways: it extends
the sign requirement from the exact bottleneck set $J^*_i$ to the
$\varepsilon_{\mathrm{act}}$-active set (the LP itself constrains only $J^*_i$;
nonnegativity on $J^*_{i,\varepsilon}\setminus J^*_i$ is an additional
requirement, not a consequence of the LP), and it asks the decrease of every
near-active weighted loss to be \emph{proportional to the suboptimality}, up to
the tolerance $b_i$.  It is a Polyak--{\L}ojasiewicz-type condition on $\psi_i$
along the LP direction, and it is a genuine restriction in two distinct ways.

\begin{enumerate}[leftmargin=2em,itemsep=2pt,label=\textup{(\alph*)}]
  \item \emph{It excludes non-minimizing stationary points.}  In the double-well
        example above the left-hand side of~\eqref{eq:descent_margin} is $0$
        while $c_d(\psi_i-\psi_i^*)=c_d$, so the strict form fails there.
  \item \emph{It can also fail near a regular, convex, nonsmooth minimizer.}
        Take $m=2$, $l_1(w)=1+w$, $l_2(w)=1-w$ on $w\in[-\tfrac12,\tfrac12]$ and
        $\boldsymbol\lambda=(\tfrac12,\tfrac12)$, so that both losses are
        positive and smooth with $B=1$, $\psi(w)=\tfrac12+\tfrac12|w|$ and
        $\psi^*=\tfrac12$ at $w=0$.  Here $\varepsilon_{\mathrm{act}}=2\eta\lambda_{\max}B^2=\eta$.
        For any fixed $\eta>0$ and any $0<w<\eta$ the two weighted losses
        differ by $w<\varepsilon_{\mathrm{act}}$, so \emph{both} objectives are
        near-active, and for \emph{every} direction $d$
        \[
            \min_{j\in J^*_{\varepsilon}(w)}\lambda_j\,d\,g_j
            =\min\{\tfrac12 d,-\tfrac12 d\}=-\tfrac12|d|\le0
            \;<\;c_d\bigl(\psi(w)-\psi^*\bigr)=\tfrac12c_dw ,
        \]
        while $\beta(w)\to0$ as $w\to0$, so the failure occurs inside every
        region $\{\beta\le\bar\beta\}$ with $\bar\beta>0$.  The strict form
        ($b_i=0$) therefore fails throughout a neighbourhood of a perfectly
        well-behaved minimizer, because a fixed near-active set can contain
        gradients that point in opposite directions.  The tolerant form
        accommodates this example only with $b_i\ge\tfrac12(c_dw+|d|)$, which
        does not vanish with $\eta$ for a direction of fixed magnitude.
\end{enumerate}

Consequently Assumption~\ref{asm:descent_margin} is \emph{not} implied by
Assumptions~\ref{asm:bounded}--\ref{asm:smooth}, by preference balance
($\beta_i\le\bar\beta$), by the LP constraints, or by the quadratic growth of
Assumption~\ref{asm:pl}; we do not call it mild.  Its role is to make explicit
the quantitative descent property that any horizon-independent approximation
guarantee for the shared iterate must rely on.  The next lemma gives a
sufficient condition that isolates the two ingredients (a curvature condition
and a direction-comparability condition); it does not remove the limitation
in~(b).

\begin{lemma}[A sufficient condition for the strict form of Assumption~\ref{asm:descent_margin}]
\label{lem:dm_sufficient}
Fix $i$ and $w$, and let
\begin{equation}\label{eq:eps_subdiff}
    \mathcal G_{i,\varepsilon}(w):=\operatorname{conv}\bigl\{\lambda_i^{(j)}g_i^{(j)}(w):\ j\in J^*_{i,\varepsilon}(w)\bigr\},
    \qquad
    d_{M,i}(w):=\argmin_{v\in\mathcal G_{i,\varepsilon}(w)}\|v\|
\end{equation}
be the \emph{near-active gradient hull} of the Chebyshev scalarization at $w$
and its minimum-norm element (unique, since $\|\cdot\|^2$ is strictly convex on
the compact convex set $\mathcal G_{i,\varepsilon}(w)$).  When
$J^*_{i,\varepsilon}(w)=J^*_i(w)$ the hull coincides with the Clarke
subdifferential of $\psi_i$ at $w$; in general it is a different object from
the Goldstein $\varepsilon$-subdifferential, which is built from subgradients
at \emph{parameters} in a neighbourhood of $w$ rather than from
\emph{objectives} whose weighted values are near the maximum, and we do not
import results about the latter.  Suppose that for all $w\in\mathcal N$ with
$\beta_i(\boldsymbol l_i(w))\le\bar\beta$:
\begin{enumerate}[leftmargin=2em,itemsep=2pt,label=\textup{(\roman*)}]
  \item \textup{(near-active PL)} $\ \|d_{M,i}(w)\|^2\ \ge\ 2\mu_\psi\bigl(\psi_i(w)-\psi_i^*\bigr)$
        for some $\mu_\psi>0$;
  \item \textup{($\sigma$-comparability of the LP direction)}
        $\ \min_{j\in J^*_{i,\varepsilon}(w)}\lambda_i^{(j)}\,d_i(w)^\top g_i^{(j)}(w)\ \ge\ \sigma\,\|d_{M,i}(w)\|^2$
        for some $\sigma>0$.
\end{enumerate}
Then the strict form ($b_i=0$) of Assumption~\ref{asm:descent_margin} holds with
$c_d=2\sigma\mu_\psi$.  Moreover, condition~(ii) holds with $\sigma=1$ for the
minimum-norm direction $d_{M,i}(w)$ itself.
\end{lemma}

\begin{proof}
Chaining (ii) and (i) gives
$\min_{j\in J^*_{i,\varepsilon}}\lambda_i^{(j)}d_i^\top g_i^{(j)}
\ge\sigma\|d_{M,i}\|^2\ge2\sigma\mu_\psi(\psi_i(w)-\psi_i^*)$,
which is~\eqref{eq:descent_margin} with $c_d=2\sigma\mu_\psi$ and $b_i=0$.  For
the last claim, $d_{M,i}$ is the Euclidean projection of the origin onto the
closed convex set $\mathcal G_{i,\varepsilon}(w)$, so the variational inequality
of the projection gives $\langle d_{M,i},\,v-d_{M,i}\rangle\ge0$ for all
$v\in\mathcal G_{i,\varepsilon}(w)$; taking $v=\lambda_i^{(j)}g_i^{(j)}$ for a
near-active $j$ yields $\lambda_i^{(j)}d_{M,i}^\top g_i^{(j)}\ge\|d_{M,i}\|^2$,
i.e.\ (ii) with $\sigma=1$.
\end{proof}

\begin{remark}[What Lemma~\ref{lem:dm_sufficient} does and does not give]
\label{rem:dm_scope}
Condition~(i) is a PL inequality for the max-function $\psi_i$ over the
near-active hull.  It is a local curvature assumption that is \emph{different
from} the quadratic growth of Assumption~\ref{asm:pl} (PL implies quadratic
growth, not conversely), and we do not use one as a substitute for the other.
It is also a \emph{strong} condition rather than a routine regularity property:
in example~(b) preceding the lemma the hull is
$\mathcal G_{\varepsilon}(w)=\operatorname{conv}\{\tfrac12,-\tfrac12\}=[-\tfrac12,\tfrac12]$
for $0<w<\varepsilon_{\mathrm{act}}$, so $d_M=0$ while $\psi(w)-\psi^*>0$, and (i)
fails.  Condition~(ii) is a property of the \emph{actual} LP direction $d_i$
of~\eqref{eq:6}, and it concerns the near-active set: the LP constraint~(6c)
only enforces $d_i^\top g_i^{(j)}\ge0$ on the exact bottleneck set $J^*_i$, so
the sign of the minimum over $J^*_{i,\varepsilon}\supseteq J^*_i$ is not
guaranteed by the LP and must be assumed.  Condition~(ii) holds automatically
for the MGDA-type direction $d_{M,i}$; if one replaced the uniform-descent
objective in~\eqref{eq:6a} by $d_{M,i}$, (ii) would hold for that mode only,
(i) would still have to be assumed, the balancing mode would be unaffected, and
the algorithm would change and the experiments would have to be re-run.

The ratio
$\min_{j\in J^*_{i,\varepsilon}}\lambda_i^{(j)}d_i^\top g_i^{(j)}\big/\|d_{M,i}\|^2$
can be recorded along training as an \emph{empirical diagnostic} of (ii).
Such a record does not prove Assumption~\ref{asm:descent_margin}; if it is
reported, it should include at least: how iterates with $d_{M,i}\approx0$
(vanishing denominator) are handled; the frequency with which the numerator
is negative; the distribution across clients and training phases rather than
a single average; and separate statistics for the exact bottleneck set $J^*_i$
and the near-active set $J^*_{i,\varepsilon}$.  Estimating the suboptimality
$\psi_i(w)-\psi_i^*$ with a value $\hat\psi_i^*$ obtained from a well-trained
model is also delicate: a feasible value satisfies $\hat\psi_i^*\ge\psi_i^*$,
so $\psi_i(w)-\hat\psi_i^*\le\psi_i(w)-\psi_i^*$ under-estimates the true
suboptimality and makes~\eqref{eq:descent_margin} look easier to satisfy than
it is; an estimate from different data or a different objective adds a
further comparison error.  We have not carried out such a study in this paper.
\end{remark}

\begin{lemma}[Residual preference uniformity floor (conditional form)]
\label{lem:beta_floor}
Let Assumptions~\ref{asm:bounded}--\ref{asm:smooth} and~\ref{asm:bal_sens} hold.
Fix a client $i\in C_n^{+}$ and an iterate $w$ (with $\beta_i$ in the
neighbourhood where~\eqref{eq:bal_sens} applies) at which the range
condition~\eqref{eq:range_cond} \emph{fails}, i.e.\
\begin{equation}\label{eq:range_fail}
    \gamma_i^*\;\le\;\Delta_i\,\|d_{\mathrm{bal},i}\|+L_0B^2\tau\eta .
\end{equation}
Then client $i$ satisfies the \emph{preference-heterogeneity floor}
\begin{equation}\label{eq:beta_floor}
    \beta_i(\boldsymbol{l}_i)
    \;\le\;\bar\beta
    \;:=\;\frac{2L_A^2}{\kappa_\beta^2}\,\Delta_i^2
        +\frac{2L_0B^2}{\kappa_\beta}\,\tau\eta
    \;\le\;\frac{2L_A^2}{\kappa_\beta^2}\,\frac{H_{C_n}}{\rho_{\min}}
        +\frac{2L_0B^2}{\kappa_\beta}\,\tau\eta ,
\end{equation}
so that $l(w)$ lies in the cone $M_{\boldsymbol{\lambda}_i}=\{\,l:\beta_i(l)\le\bar\beta\,\}$
around the $\boldsymbol{\lambda}_i^{-1}$ ray.  If in addition the hypotheses of
Corollary~\ref{cor:homo} hold at the relevant iterates (restricted case,
Assumptions~\ref{asm:common_G}--\ref{asm:lp_stable}), then
$\sqrt{\bar\beta}=O(\delta_\lambda)+O(\sqrt{\tau\eta})$.
In a \emph{mixed} cluster the floor is set by the full-cluster deviation
$\Delta_i$, whose direct split~\eqref{eq:cp_delta_split} feeds the cross-block
pull $\rho^{0}\|\bar d^{0}-d_i\|$ (the $n_{0}$ uniform clients pulling on a
balancing client $i\in C_n^{+}$) into $\bar\beta$; this cross-block interference
only widens the cone $M_{\boldsymbol{\lambda}_i}$ within the same
$O(\sqrt{H_{C_n}})$ order and vanishes when $\rho^{0}\to0$ or $\bar d^{0}=d_i$.
\end{lemma}

\begin{remark}[Scope of Lemma~\ref{lem:beta_floor}]
\label{rem:beta_floor_scope}
The lemma is a statement about a \emph{single client at a single iterate} where
the sufficient condition of Theorem~\ref{thm:collab_descent} is not met.  It
does not assert that the iteration stops there, that later iterates remain in
the cone, or that all clients reach~\eqref{eq:range_fail} simultaneously; the
failure of a sufficient descent condition for one client is not a statement
about the other clients.  Establishing that the whole cluster eventually enters
and remains in the region $\{\beta_i\le\bar\beta\ \forall i\}$ would require
an additional entry-and-invariance argument that we do not provide.  For this
reason the convergence results below
(Theorem~\ref{thm:collab_admissible}, Corollary~\ref{cor:collab_monotone},
Theorem~\ref{thm:collab_converge}) take
$\beta_i(\boldsymbol{l}_i(w^t))\le\bar\beta$ for all $i\in C_n$ and all
$t\ge t_0$ as an explicit \emph{trajectory hypothesis}, see~\eqref{eq:cp_traj_cond},
rather than deriving it from Lemma~\ref{lem:beta_floor}.  We also do not
claim that $\beta_i=0$ for all clients is impossible when preferences differ:
the clients' loss functions differ as well, so a shared parameter can in
principle produce different loss vectors that all lie on their respective rays.
\end{remark}

\begin{proof}
Under~\eqref{eq:range_fail},
$\gamma_i^*\le\Delta_i\|d_{\mathrm{bal},i}\|+L_0B^2\tau\eta$.  Substituting the
two bounds of~\eqref{eq:bal_sens} ($\gamma_i^*\ge\kappa_\beta\beta_i$ on the
left, $\|d_{\mathrm{bal},i}\|\le L_A\sqrt{\beta_i}$ on the right) gives
\[
    \kappa_\beta\,\beta_i
    \;\le\;\Delta_i\,L_A\,\sqrt{\beta_i}+L_0B^2\tau\eta .
\]
Set $x:=\sqrt{\beta_i}\ge0$.  Then
$\kappa_\beta x^2-\Delta_iL_A\,x-L_0B^2\tau\eta\le0$, and the quadratic formula
gives
\[
    x\;\le\;\frac{\Delta_iL_A+\sqrt{\Delta_i^2L_A^2+4\kappa_\beta L_0B^2\tau\eta}}
                 {2\kappa_\beta}
    \;\le\;\frac{\Delta_iL_A}{\kappa_\beta}
        +\sqrt{\frac{L_0B^2\tau\eta}{\kappa_\beta}},
\]
using $\sqrt{a^2+b}\le a+\sqrt{b}$ for $a,b\ge0$.  Squaring and applying
$(p+q)^2\le2p^2+2q^2$ yields
$\beta_i=x^2\le\frac{2L_A^2}{\kappa_\beta^2}\Delta_i^2+\frac{2L_0B^2}{\kappa_\beta}\tau\eta$,
the first inequality of~\eqref{eq:beta_floor}.  The second uses
$\Delta_i^2\le H_{C_n}/\rho_{\min}$ (Corollary~\ref{cor:collab_range}); in a mixed cluster
the deviation $\Delta_i$ carries the cross-block pull $\rho^{0}\|\bar d^{0}-d_i\|$
of the direct split~\eqref{eq:cp_delta_split}, which is bounded by the same
$H_{C_n}$ and so does not change the order of the floor.  Finally, when the
hypotheses of Corollary~\ref{cor:homo} hold and
$\max_{i,i'}\|\boldsymbol{\lambda}_i-\boldsymbol{\lambda}_{i'}\|\le\delta_\lambda$,
the heterogeneity obeys $H_{C_n}\le C_1\sigma_{\max}(C)\delta_\lambda^2$
(Eq.~\eqref{eq:H_homo}), so
$\bar\beta=O(\delta_\lambda^2)+O(\tau\eta)$ and
$\sqrt{\bar\beta}=O(\delta_\lambda)+O(\sqrt{\tau\eta})$.  Since $\beta_i(l)=0$
iff $l$ lies on the $\boldsymbol{\lambda}_i^{-1}$ ray, the sublevel set
$M_{\boldsymbol{\lambda}_i}$ is a cone around that ray.
\end{proof}

\begin{remark}[From mixed clusters to the residual-floor regime]
\label{rem:cp_reduction}
The convergence analysis below is stated at the residual floor
$\beta_i(\boldsymbol{l}_i)\le\bar\beta$ for all $i\in C_n$
(Lemma~\ref{lem:beta_floor}); the mixed cluster reduces to it.  Indeed, while
some clients are still unbalanced ($C_n^{+}\neq\varnothing$),
Lemma~\ref{lem:collab_class}(a) (equivalently Theorem~\ref{thm:collab_descent})
shows the shared update decreases every balancing client's preference uniformity
$\beta_i$ under the range condition~\eqref{eq:range_cond}, driving its loss toward
the $\boldsymbol{\lambda}_i^{-1}$ ray, while Lemma~\ref{lem:collab_class}(b) keeps
\emph{every} client in approximate descent (its objectives rise by at
most $O(\eta(\delta_g+\rho^{+}\bar D_{\mathrm{bal}}^{+}))$, and since the balancing
penalty obeys $\bar D_{\mathrm{bal}}^{+}\le L_A\sqrt{\bar\beta}=O(\delta_\lambda)$
by Assumption~\ref{asm:bal_sens} and Lemma~\ref{lem:beta_floor}, this is
$O(\eta(\delta_g+\delta_\lambda))$).
The balancing guarantee is only \emph{sufficient}: for a client at which the
range condition~\eqref{eq:range_cond} fails, Lemma~\ref{lem:beta_floor} bounds
that client's $\beta_i$ by the floor $\bar\beta=O(H_{C_n})+O(\tau\eta)$, but
(Remark~\ref{rem:beta_floor_scope}) this does not by itself show that the whole
cluster enters and stays in the region $\{\beta_i\le\bar\beta\ \forall i\}$.
The convergence results that follow therefore take this region, together with
Assumption~\ref{asm:descent_margin}, as an explicit trajectory hypothesis
(\eqref{eq:cp_traj_cond}); under it the admissible-set machinery
(Lemma~\ref{lem:collab_vsub}$\to$Theorem~\ref{thm:collab_admissible}
$\to$Corollary~\ref{cor:collab_monotone}$\to$Theorem~\ref{thm:collab_converge})
bounds every client's Chebyshev suboptimality $\psi_i(w^t)-\psi_i^*$
asymptotically by $\epsilon_{0,i}=\kappa_{H,i}\bar\Delta_i+\kappa_{\eta,i}\eta+b_i/c_d$
(for every $\zeta>0$ it is eventually below $\epsilon_{0,i}+\zeta$), where
$\bar\Delta_i$ bounds the deviation of the local LP direction
from the shared aggregate along the trajectory and $b_i$ is the tolerance of
Assumption~\ref{asm:descent_margin}.
Figure~\ref{fig:collab_adm} illustrates the sets involved (the band $\mathcal A^\infty$ itself is not drawn).
\end{remark}

\paragraph{Exact one-step bounds (no $o(\eta)$ remainders).}
The convergence argument below is stated for the synchronous shared iteration
\begin{equation}\label{eq:cp_sync}
    w^{t+1}\;=\;w^t-\eta\,\bar d^{\,t},
    \qquad
    \bar d^{\,t}:=\sum_{i\in C_n}\rho_i\,d_i(w^t),
    \qquad
    \Delta_i^t:=\bigl\|\bar d^{\,t}-d_i(w^t)\bigr\| ,
\end{equation}
i.e.\ \emph{one} local LP step per communication round ($\tau=1$), with the
server step size absorbed into $\eta$.  The results of this paragraph and of
Theorems~\ref{thm:collab_admissible} and~\ref{thm:collab_converge} are stated
and proved for~\eqref{eq:cp_sync} only; they do not cover the multi-local-step
Algorithm~\ref{alg:1} with $\tau>1$, and the same restriction applies to the
$K$ inner steps of Phase~2, which the analysis of
Corollary~\ref{cor:phase2_cont} covers only for $K=1$ (with $\eta$ replaced by
the Phase-2 step size).  Extending them would require replacing
$d_i(w^t)$ by the direction that client $i$ actually contributes to the
aggregate, namely the average of its $\tau$ local LP directions
$u_i^t:=\tau^{-1}\sum_{k=0}^{\tau-1}d_i(w_{i,k}^t)$, taking
$\bar u^t=\sum_i\rho_iu_i^t$ as the shared direction with effective step
$\alpha\eta\tau$ ($\alpha$ the server learning rate of Algorithm~\ref{alg:1}),
and imposing the descent-margin condition on $u_i^t$ evaluated at
the shared point $w^t$.  We do not carry this out here, because the LP
direction $d_i(\cdot)$ is a solution map of a linear program whose active set
changes with $w$ and is in general \emph{not} continuous in $w$; the drift
$\|u_i^t-d_i(w^t)\|$ therefore cannot be bounded by $O(L_0\tau\eta B)$ or
folded into $\Delta_i^t$ without an additional stability argument for the LP
solution, and we make no such claim.  Throughout we write
$\lambda_{\max,i}:=\|\boldsymbol{\lambda}_i\|_\infty$ and
$\lambda_{\min,i}:=\min_j\lambda_i^{(j)}>0$, and use the
$\varepsilon_{\mathrm{act}}$-active set $J^*_{i,\varepsilon}(w)$
of~\eqref{eq:active_set}.

\begin{lemma}[Exact one-step loss and Chebyshev bounds]\label{loss_avg}
Let $\bar d$ be any direction with $\|\bar d\|\le B$ (in particular $\bar d=\bar d^{\,t}$
of~\eqref{eq:cp_sync}, since every $d_i=G_i\mu_i^*$ is a convex combination of
gradients of norm at most $B$), and let $\eta>0$ be such that the whole segment
$\{w-s\bar d:\,s\in[0,\eta]\}$ lies in the region where
Assumptions~\ref{asm:bounded}--\ref{asm:smooth} hold (in particular each
$l_i^{(j)}$ is $L_0$-smooth along the segment and $\|g_i^{(j)}(w)\|\le B$ at the
current point).  Then for every client $i\in C_n$ and every $j\in[m]$,
\begin{equation}\label{eq:cp_taylor}
    l_i^{(j)}(w-\eta\bar d)
    \;\le\; l_i^{(j)}(w)-\eta\,\bar d^\top g_i^{(j)}(w)+\frac{L_0B^2}{2}\,\eta^2 .
\end{equation}
Moreover, writing
\begin{equation}\label{eq:cp_xi}
    \xi_i(w;\bar d)\;:=\;\min_{j\in J^*_{i,\varepsilon}(w)}\ \lambda_i^{(j)}\,\bar d^\top g_i^{(j)}(w),
\end{equation}
the Chebyshev value $\psi_i(w)=\max_j\lambda_i^{(j)}l_i^{(j)}(w)$ satisfies
\begin{equation}\label{eq:cp_psi_step}
    \psi_i(w-\eta\bar d)
    \;\le\;\psi_i(w)-\eta\,\xi_i(w;\bar d)+\frac{\lambda_{\max,i}L_0B^2}{2}\,\eta^2 .
\end{equation}
\end{lemma}

\begin{proof}
\emph{Per-coordinate bound.}  Since $l_i^{(j)}$ is $L_0$-smooth along the
segment $\{w-s\bar d:\,s\in[0,\eta]\}$ (Assumption~\ref{asm:smooth}), the
descent lemma gives
$l_i^{(j)}(w-\eta\bar d)\le l_i^{(j)}(w)-\eta\,\bar d^\top g_i^{(j)}(w)
+\tfrac{L_0}{2}\eta^2\|\bar d\|^2$, and $\|\bar d\|\le B$
yields~\eqref{eq:cp_taylor}.  No $o(\eta)$ term and no unspecified $\eta_0$
appear; the only requirement on $\eta$ is that the segment stays inside the
region where the assumptions hold.

\emph{Chebyshev bound.}  Multiply~\eqref{eq:cp_taylor} by $\lambda_i^{(j)}>0$.
For an active index $j\in J^*_{i,\varepsilon}(w)$ we have
$\lambda_i^{(j)}l_i^{(j)}(w)\le\psi_i(w)$ and
$\lambda_i^{(j)}\bar d^\top g_i^{(j)}\ge\xi_i(w;\bar d)$, hence
\[
    \lambda_i^{(j)}l_i^{(j)}(w-\eta\bar d)
    \;\le\;\psi_i(w)-\eta\,\xi_i(w;\bar d)+\frac{\lambda_{\max,i}L_0B^2}{2}\eta^2 .
\]
For an inactive index $j\notin J^*_{i,\varepsilon}(w)$ the definition~\eqref{eq:active_set}
gives $\lambda_i^{(j)}l_i^{(j)}(w)<\psi_i(w)-\varepsilon_{\mathrm{act}}$ with
$\varepsilon_{\mathrm{act}}=2\eta\lambda_{\max,i}B^2$, while Cauchy--Schwarz and
Assumption~\ref{asm:bounded} give $|\bar d^\top g_i^{(j)}|\le B^2$; therefore
\[
\begin{aligned}
    \lambda_i^{(j)}l_i^{(j)}(w-\eta\bar d)
    \;<\;\psi_i(w)-2\eta\lambda_{\max,i}B^2+\eta\lambda_{\max,i}B^2
        +\frac{\lambda_{\max,i}L_0B^2}{2}\eta^2
    \\ \;\le\;\psi_i(w)-\eta\,\xi_i(w;\bar d)+\frac{\lambda_{\max,i}L_0B^2}{2}\eta^2 ,
\end{aligned}
\]
where the last step uses $\xi_i(w;\bar d)\le\lambda_{\max,i}B^2$ (again by
Cauchy--Schwarz).  Thus every weighted loss, active or not, is bounded by the
same right-hand side; an inactive objective may well become the new maximum,
but it cannot exceed this common bound.  Taking the maximum over $j\in[m]$
gives~\eqref{eq:cp_psi_step}.
\end{proof}

\begin{remark}[Relation to the per-objective slack of Lemma~\ref{lem:collab_class}]
Lemma~\ref{lem:collab_class}(b) bounds the aggregate from below on \emph{every}
objective,
\begin{equation}\label{eq:cp_dirbound}
    \bar d_{C_n}^\top g_i^{(j)}\;\ge\;-\bigl(B\delta_g+\rho^{+}B\,\bar D_{\mathrm{bal}}^{+}\bigr),
    \qquad\forall\,i\in C_n^{0},\ j\in[m],
\end{equation}
and this is the quantity that governs how much a \emph{non-bottleneck} loss may
rise in one round.  The contraction argument below does \emph{not} use
\eqref{eq:cp_dirbound}: it only needs the behaviour of the aggregate on the
$\varepsilon_{\mathrm{act}}$-active objectives, which is controlled through the
local LP direction $d_i$ and the deviation $\Delta_i^t$
(Assumption~\ref{asm:descent_margin} and Lemma~\ref{loss_avg}).
\end{remark}

\paragraph{Admissible set.}
For client $i$ at iterate $w^t$ let
$O_i\subseteq\mathbb{R}^m_{+}$ be its attainable objective set and $\le$ the
componentwise order.  Define the client ray point $q_i^t$, the client 
dominating set $V_i^{t}$, and the client admissible set $A_i^{t}$: 
\begin{equation}\label{eq:cp_adm}
\begin{aligned}
    \psi_i^t:=\max_{j}\lambda_i^{(j)}\,l_i^{(j)}(w^t),
    \quad
    q_i^t:=\psi_i^t\Bigl(\tfrac{1}{\lambda_i^{(1)}},\dots,
        \tfrac{1}{\lambda_i^{(m)}}\Bigr),
    \quad \\
    V_i^{t}:=\{l\in O_i: l\le l_i(w^t)\},
    \quad
    A_i^{t}:=\{l\in O_i: l\le q_i^t\}.
\end{aligned}
\end{equation}

The sets $A_i^t$ live in each client's own objective space $O_i$; a single shared
model $w$, however, must serve every client at once.  We therefore lift the
construction to \emph{parameter space}, defining for the cluster $C_n$ the
cluster convergence indicator $\Lambda^t$, the cluster dominating set $\mathcal V^t$, and the ($\varepsilon$-inflated) 
cluster admissible set $\mathcal A^t(\varepsilon)$: 
\begin{equation}\label{eq:cp_cluster_adm}
\begin{aligned}
    \Lambda^t:=\max_{i\in C_n}\psi_i(w^t)=\max_{i\in C_n}\psi_i^t,
    \quad
    \\\mathcal V^t:=\bigcap_{i\in C_n}\{\,w:\ l_i(w)\le l_i(w^t)\,\},
    \quad
    \\\mathcal A^t(\varepsilon):=\bigcap_{i\in C_n}\{\,w:\ l_i(w)\le q_i^t+\varepsilon\mathbf1\,\},
\end{aligned}
\end{equation}
and abbreviate $\mathcal A^t:=\mathcal A^t(0)=\bigcap_{i\in C_n}\{w:l_i(w)\le q_i^t\}$.
Thus $w\in\mathcal A^t(\varepsilon)$ iff $l_i(w)\in A_i^t(\varepsilon)$ for
\emph{every} client, i.e.\ the cluster admissible set is the set of shared
parameters whose induced losses fall in every client's inflated admissible box
simultaneously (the intersection is exactly the cluster wedge of
Figure~\ref{fig:collab_adm}).

\begin{lemma}[Dominating set lies in the admissible set]
\label{lem:collab_vsub}
For every client $i$ and iterate $t$, $V_i^{t}\subseteq A_i^{t}$.  Moreover, if
$\beta_i(\boldsymbol{l}_i(w^t))=0$ then $l_i(w^t)=q_i^t$ and $A_i^t=V_i^t$.
\end{lemma}

\begin{proof}
Let $l\in V_i^t$, i.e.\ $l\le l_i(w^t)$.  Componentwise
$\lambda_i^{(j)}l^{(j)}\le\lambda_i^{(j)}l_i^{(j)}(w^t)\le\psi_i^t$ for all
$j$; dividing by $\lambda_i^{(j)}>0$ gives
$l^{(j)}\le\psi_i^t/\lambda_i^{(j)}=q_i^{t,(j)}$, hence $l\le q_i^t$ and
$l\in A_i^t$.  If $\beta_i(\boldsymbol{l}_i(w^t))=0$ the weighted losses are all
equal to their maximum, $\lambda_i^{(j)}l_i^{(j)}(w^t)=\psi_i^t$ for every
$j$, so $l_i^{(j)}(w^t)=\psi_i^t/\lambda_i^{(j)}=q_i^{t,(j)}$, i.e.\
$l_i(w^t)=q_i^t$; then $A_i^t=\{l\le q_i^t\}=\{l\le l_i(w^t)\}=V_i^t$.
\end{proof}

\begin{theorem}[Collaborative admissible step]
\label{thm:collab_admissible}
Let Assumptions~\ref{asm:bounded}--\ref{asm:smooth} and~\ref{asm:descent_margin}
hold, and suppose that at round $t$ the shared iterate satisfies
$w^t\in\mathcal N$ and $\beta_i(\boldsymbol{l}_i(w^t))\le\bar\beta$ for all
$i\in C_n$.  Write $r_i^t:=\psi_i^t-\psi_i^*\ge0$ for the Chebyshev
suboptimality.  Then for \emph{every} step size $\eta>0$ the shared
update~\eqref{eq:cp_sync} satisfies, for all $i\in C_n$,
\begin{equation}\label{eq:cp_psi_contract}
    \psi_i^{t+1}
    \;\le\;\psi_i^{t}-\eta\,c_d\,r_i^t
    +\eta\,\lambda_{\max,i}B\,\Delta_i^t
    +\eta\,b_i
    +\frac{\lambda_{\max,i}L_0B^2}{2}\,\eta^2 ,
\end{equation}
and every client's loss stays inside its $\varepsilon_t$-inflated admissible set,
\begin{equation}\label{eq:cp_thm2}
    l_i(w^{t+1})\in A_i^{t}(\varepsilon_{i,t}),
    \qquad
    \varepsilon_{i,t}:=\frac{\lambda_{\max,i}}{\lambda_{\min,i}}
        \Bigl(\eta\,B\,\Delta_i^t+\frac{L_0B^2}{2}\,\eta^2\Bigr)
        +\frac{\eta\,b_i}{\lambda_{\min,i}} .
\end{equation}
The inclusion is exact, $l_i(w^{t+1})\in A_i^{t}$, whenever
$c_d\,r_i^t\ge\lambda_{\max,i}\bigl(B\Delta_i^t+\tfrac{L_0B^2}{2}\eta\bigr)+b_i$.
(Here $\eta$ is additionally required to be such that the segment
$[w^t,w^{t+1}]$ stays in the region where
Assumptions~\ref{asm:bounded}--\ref{asm:smooth} hold, as in Lemma~\ref{loss_avg}.)
\end{theorem}

\begin{proof}
\emph{Chebyshev contraction.}  Apply Lemma~\ref{loss_avg} with
$\bar d=\bar d^{\,t}$ (which has $\|\bar d^{\,t}\|\le B$ as a convex combination
of the $d_{i'}$, each of norm at most $B$).  For every active index
$j\in J^*_{i,\varepsilon}(w^t)$,
\[
    \lambda_i^{(j)}\,\bar d^{\,t\top}g_i^{(j)}
    =\lambda_i^{(j)}\,d_i^\top g_i^{(j)}
     +\lambda_i^{(j)}\,(\bar d^{\,t}-d_i)^\top g_i^{(j)}
    \;\ge\;\lambda_i^{(j)}\,d_i^\top g_i^{(j)}-\lambda_{\max,i}\,\Delta_i^t\,B ,
\]
by Cauchy--Schwarz and Assumption~\ref{asm:bounded}.  Taking the minimum over
the active set and invoking Assumption~\ref{asm:descent_margin} (applicable
because $w^t\in\mathcal N$ and $\beta_i(w^t)\le\bar\beta$),
\[
    \xi_i(w^t;\bar d^{\,t})
    \;\ge\;\min_{j\in J^*_{i,\varepsilon}(w^t)}\lambda_i^{(j)}d_i^\top g_i^{(j)}
        -\lambda_{\max,i}B\,\Delta_i^t
    \;\ge\;c_d\,r_i^t-b_i-\lambda_{\max,i}B\,\Delta_i^t .
\]
Substituting into~\eqref{eq:cp_psi_step} gives~\eqref{eq:cp_psi_contract}.

\emph{Admissible-set inclusion.}  For any $j\in[m]$,
$\lambda_i^{(j)}l_i^{(j)}(w^{t+1})\le\psi_i^{t+1}$; dropping the non-positive
term $-\eta c_dr_i^t$ in~\eqref{eq:cp_psi_contract} and dividing by
$\lambda_i^{(j)}\ge\lambda_{\min,i}$,
\[
    l_i^{(j)}(w^{t+1})
    \;\le\;\frac{\psi_i^t}{\lambda_i^{(j)}}
        +\frac{\lambda_{\max,i}}{\lambda_i^{(j)}}
         \Bigl(\eta B\Delta_i^t+\frac{L_0B^2}{2}\eta^2\Bigr)
        +\frac{\eta b_i}{\lambda_i^{(j)}}
    \;\le\;q_i^{t,(j)}+\varepsilon_{i,t},
\]
using $q_i^{t,(j)}=\psi_i^t/\lambda_i^{(j)}$ from~\eqref{eq:cp_adm}.  Hence
$l_i(w^{t+1})\le q_i^t+\varepsilon_{i,t}\mathbf1$, i.e.\
$l_i(w^{t+1})\in A_i^t(\varepsilon_{i,t})$.  If
$c_dr_i^t\ge\lambda_{\max,i}(B\Delta_i^t+\tfrac{L_0B^2}{2}\eta)+b_i$ the right-hand
side of~\eqref{eq:cp_psi_contract} is at most $\psi_i^t$, so
$\lambda_i^{(j)}l_i^{(j)}(w^{t+1})\le\psi_i^t$ and $l_i(w^{t+1})\in A_i^t$.
\end{proof}

\begin{corollary}[Contraction of the Chebyshev suboptimality and drift-adjusted nesting]
\label{cor:collab_monotone}
Let Assumptions~\ref{asm:bounded}--\ref{asm:smooth} and~\ref{asm:descent_margin}
hold, let $0<\eta\le1/c_d$, and suppose there is a round $t_0$ such that for all
$t\ge t_0$ and all $i\in C_n$: $w^t\in\mathcal N$,
$\beta_i(\boldsymbol{l}_i(w^t))\le\bar\beta$, and
$\Delta_i^t\le\bar\Delta_i$ for some constant $\bar\Delta_i\ge0$.  Define the
per-client band radius
\begin{equation}\label{eq:cp_band_radius}
\begin{aligned}
    \epsilon_{0,i}
    \;:=\;\frac{\lambda_{\max,i}}{c_d}\Bigl(B\,\bar\Delta_i+\frac{L_0B^2}{2}\,\eta\Bigr)+\frac{b_i}{c_d}
    \;=\;\kappa_{H,i}\,\bar\Delta_i+\kappa_{\eta,i}\,\eta+\frac{b_i}{c_d},
    \qquad
   \\ \kappa_{H,i}:=\frac{\lambda_{\max,i}B}{c_d},\quad
    \kappa_{\eta,i}:=\frac{\lambda_{\max,i}L_0B^2}{2c_d},
\end{aligned}
\end{equation}
where $b_i$ is the tolerance of Assumption~\ref{asm:descent_margin} ($b_i=0$
in the strict form).  The radius is an explicit \emph{conditional} error bound:
its first two terms shrink with $\bar\Delta_i$ and $\eta$, but the third,
$b_i/c_d$, has no proved vanishing rate, so a small direction deviation and a
small step size do not by themselves make $\epsilon_{0,i}$ small.  Then for all
$t\ge t_0$ and $i\in C_n$:
\begin{enumerate}[leftmargin=2em,itemsep=2pt]
  \item \textup{(One-step contraction)}
        \begin{equation}\label{eq:cp_monotone}
            r_i^{t+1}\;\le\;(1-\eta c_d)\,r_i^t+\eta\,c_d\,\epsilon_{0,i} ,
        \end{equation}
        so $\psi_i^{t+1}<\psi_i^t$ whenever $r_i^t>\epsilon_{0,i}$, and
        $\psi_i^{t+1}\le\psi_i^t$ whenever $r_i^t\ge\epsilon_{0,i}$.  In
        particular, even for $\bar\Delta_i=0$ and $b_i=0$ the Chebyshev value
        is guaranteed to be non-increasing only outside the $O(\eta)$ band
        $r_i^t\ge\kappa_{\eta,i}\eta$ created by the second-order term; the
        earlier claim of exact monotonicity for every step size is not
        available.
  \item \textup{(Horizon-independent band)}  Unrolling~\eqref{eq:cp_monotone},
        \begin{equation}\label{eq:cp_unrolled}
            r_i^{T}\;\le\;(1-\eta c_d)^{\,T-t_0}\,r_i^{t_0}
            +\bigl(1-(1-\eta c_d)^{\,T-t_0}\bigr)\,\epsilon_{0,i}
            \;\le\;(1-\eta c_d)^{\,T-t_0}\,r_i^{t_0}+\epsilon_{0,i},
            \qquad\forall\,T\ge t_0 ,
        \end{equation}
        and hence $\limsup_{T\to\infty}\bigl(\psi_i(w^T)-\psi_i^*\bigr)\le\epsilon_{0,i}$.
        This is a statement about the \emph{values} $\psi_i(w^T)$; it does not
        assert that $\{w^t\}$ converges or has a limit point, and it does not
        assert that $r_i^T\le\epsilon_{0,i}$ is reached in finite time: for
        every $\zeta>0$ there is $T_\zeta$ with $r_i^T\le\epsilon_{0,i}+\zeta$
        for all $T\ge T_\zeta$, but the closed band itself need never be
        entered (the recursion $r^{t}=\epsilon_{0,i}+q^{t}(r^{0}-\epsilon_{0,i})$,
        $q=1-\eta c_d$, satisfies~\eqref{eq:cp_monotone} with equality and stays
        strictly above $\epsilon_{0,i}$ whenever $r^0>\epsilon_{0,i}$).
  \item \textup{(Nesting)}  The admissible sets are nested up to the per-step
        slack of Theorem~\ref{thm:collab_admissible},
        \begin{equation}\label{eq:cp_cluster_nested_drift}
            A_i^{t+1}(\varepsilon)\subseteq A_i^t(\varepsilon+\varepsilon_{i,t}),
            \qquad
            \mathcal A^{t+1}(\varepsilon)\subseteq\mathcal A^t\bigl(\varepsilon+\max_{i}\varepsilon_{i,t}\bigr),
            \qquad\forall\,\varepsilon\ge0 ,
        \end{equation}
        and \emph{exactly} nested,
        \begin{equation}\label{eq:cp_cluster_nested}
            A_i^{t+1}(\varepsilon)\subseteq A_i^t(\varepsilon)
            \quad\text{and}\quad
            \mathcal A^{t+1}(\varepsilon)\subseteq\mathcal A^t(\varepsilon),
        \end{equation}
        at every round in which $r_i^t\ge\epsilon_{0,i}$ for all $i\in C_n$.
  \item \textup{(Limiting band)}  With
        $q_i^*=\psi_i^*(1/\lambda_i^{(1)},\dots,1/\lambda_i^{(m)})$ the ray
        point of client $i$'s local optimum,
        \begin{equation}\label{eq:cp_At_band}
            \mathcal A^t\;\subseteq\;\bigcap_{i\in C_n}
            \Bigl\{\,w:\ l_i(w)\le q_i^*+\tfrac{1}{\lambda_{\min,i}}
            \bigl((1-\eta c_d)^{t-t_0}r_i^{t_0}+\epsilon_{0,i}\bigr)\mathbf1\,\Bigr\},
        \end{equation}
        a family of sets that decreases, as $t\to\infty$, to the fixed limiting cluster band
        \begin{equation}\label{eq:cp_cluster_band}
            \mathcal A^\infty:=\bigcap_{i\in C_n}
            \Bigl\{\,w:\ l_i(w)\le q_i^*+\frac{\epsilon_{0,i}}{\lambda_{\min,i}}\,\mathbf1\,\Bigr\}.
        \end{equation}
\end{enumerate}
\end{corollary}

\begin{proof}
\emph{Item 1.}  Subtract $\psi_i^*$ from both sides of~\eqref{eq:cp_psi_contract}
and use $\Delta_i^t\le\bar\Delta_i$:
$r_i^{t+1}\le(1-\eta c_d)r_i^t+\eta\lambda_{\max,i}B\bar\Delta_i+\eta b_i
+\tfrac{\lambda_{\max,i}L_0B^2}{2}\eta^2=(1-\eta c_d)r_i^t+\eta c_d\epsilon_{0,i}$
by the definition~\eqref{eq:cp_band_radius}.  Rewriting as
$\psi_i^{t+1}-\psi_i^t\le-\eta c_d\,(r_i^t-\epsilon_{0,i})$ gives the sign
statements.  For $\bar\Delta_i=0$, $b_i=0$ the band reduces to $\epsilon_{0,i}=\kappa_{\eta,i}\eta$.

\emph{Item 2.}  Since $0\le1-\eta c_d<1$, induction on~\eqref{eq:cp_monotone} gives
$r_i^T\le(1-\eta c_d)^{T-t_0}r_i^{t_0}+\eta c_d\epsilon_{0,i}\sum_{s=0}^{T-t_0-1}(1-\eta c_d)^s
=(1-\eta c_d)^{T-t_0}r_i^{t_0}+\bigl(1-(1-\eta c_d)^{T-t_0}\bigr)\epsilon_{0,i}$,
using $\eta c_d\sum_{s=0}^{K-1}(1-\eta c_d)^s=1-(1-\eta c_d)^K$; the second
inequality in~\eqref{eq:cp_unrolled} drops the nonpositive term
$-(1-\eta c_d)^{T-t_0}\epsilon_{0,i}$.

\emph{Item 3.}  Dropping $-\eta c_dr_i^t\le0$ in~\eqref{eq:cp_psi_contract} and
dividing by $\lambda_i^{(j)}\ge\lambda_{\min,i}$,
$q_i^{t+1,(j)}=\psi_i^{t+1}/\lambda_i^{(j)}\le q_i^{t,(j)}+\varepsilon_{i,t}$, so
$q_i^{t+1}+\varepsilon\mathbf1\le q_i^t+(\varepsilon+\varepsilon_{i,t})\mathbf1$
and $A_i^{t+1}(\varepsilon)\subseteq A_i^t(\varepsilon+\varepsilon_{i,t})$;
intersecting over $i$ gives the cluster statement
in~\eqref{eq:cp_cluster_nested_drift}.  If $r_i^t\ge\epsilon_{0,i}$ then
$\psi_i^{t+1}\le\psi_i^t$ by Item~1, hence $q_i^{t+1}\le q_i^t$ componentwise and
the inclusions hold with no inflation, which is~\eqref{eq:cp_cluster_nested}.

\emph{Item 4.}  From $q_i^{t,(j)}=\psi_i^t/\lambda_i^{(j)}=(\psi_i^*+r_i^t)/\lambda_i^{(j)}
=q_i^{*,(j)}+r_i^t/\lambda_i^{(j)}\le q_i^{*,(j)}+r_i^t/\lambda_{\min,i}$ and
\eqref{eq:cp_unrolled}, every $w\in\mathcal A^t$ satisfies
$l_i(w)\le q_i^t\le q_i^*+\lambda_{\min,i}^{-1}\bigl((1-\eta c_d)^{t-t_0}r_i^{t_0}+\epsilon_{0,i}\bigr)\mathbf1$
for all $i$, which is~\eqref{eq:cp_At_band}; the right-hand sets decrease in
$t$ and their intersection over $t$ is~\eqref{eq:cp_cluster_band}.
\end{proof}

\begin{lemma}\label{prop:cheby2}
(Adapted from~\citep{choo1983proper}, Theorem 3.1). A feasible solution $w$ is weakly Pareto optimal if and
only if there exists a weight vector $\boldsymbol{\lambda_i}$ such that $w$ is an optimal solution to the aggregation function $\psi_i(w)=\max_{j\in[m]}\lambda_i^{(j)} (l_i^{(j)}(w)-l_{i}^{(j)*})$ defined in the main paper.
\end{lemma}
\begin{theorem}[Cluster solution approximates every local optimum under the descent-margin condition]
\label{thm:collab_converge}
Let Assumptions~\ref{asm:bounded}--\ref{asm:smooth} and~\ref{asm:descent_margin}
hold and let $0<\eta\le1/c_d$.  Suppose that the shared
iteration~\eqref{eq:cp_sync} admits a round $t_0$ such that, for all $t\ge t_0$
and all $i\in C_n$,
\begin{equation}\label{eq:cp_traj_cond}
    w^t\in\mathcal N,
    \qquad
    \beta_i(\boldsymbol{l}_i(w^t))\le\bar\beta,
    \qquad
    \Delta_i^t=\|\bar d^{\,t}-d_i(w^t)\|\le\bar\Delta_i
\end{equation}
(the residual-floor regime of Lemma~\ref{lem:beta_floor}, with the direction
deviation bounded along the trajectory), and that every segment
$[w^t,w^{t+1}]$, $t\ge t_0$, lies in the region where
Assumptions~\ref{asm:bounded}--\ref{asm:smooth} hold.  Let
$\epsilon_{0,i}=\kappa_{H,i}\bar\Delta_i+\kappa_{\eta,i}\eta+b_i/c_d$
be the band radius~\eqref{eq:cp_band_radius}.  Then for every client $i\in C_n$:
\begin{enumerate}[leftmargin=2em,itemsep=2pt]
  \item \textup{(Finite-horizon rate)}  For all $T\ge t_0$,
        $\psi_i(w^T)-\psi_i^*\le(1-\eta c_d)^{T-t_0}\bigl(\psi_i(w^{t_0})-\psi_i^*\bigr)+\epsilon_{0,i}$.
  \item \textup{(Objective space: asymptotic band)}
        \begin{equation}\label{eq:cp_obj_band}
            \limsup_{T\to\infty}\bigl(\psi_i(w^T)-\psi_i^*\bigr)\;\le\;\epsilon_{0,i},
            \qquad
            \limsup_{T\to\infty}\,l_i^{(j)}(w^T)\;\le\;q_i^{*,(j)}+\frac{\epsilon_{0,i}}{\lambda_{\min,i}}
            \quad\text{for every }j\in[m],
        \end{equation}
        where $q_i^*=\psi_i^*\bigl(1/\lambda_i^{(1)},\dots,1/\lambda_i^{(m)}\bigr)$
        is the ray point of client $i$'s local optimum $w_i^*$.  \emph{If} the
        sequence $\{w^t\}$ has a limit point $w^*$, then by continuity of the
        losses $w^*$ inherits these bounds,
        $\psi_i(w^*)\le\psi_i^*+\epsilon_{0,i}$ and
        $l_i(w^*)\le q_i^*+(\epsilon_{0,i}/\lambda_{\min,i})\mathbf1$
        componentwise, i.e.\ $w^*\in\mathcal A^\infty$, the fixed limiting
        cluster band of~\eqref{eq:cp_cluster_band}.  We do not claim that
        $\{w^t\}$ converges or that a limit point exists, nor that the closed
        band $r_i^t\le\epsilon_{0,i}$ is entered after finitely many rounds;
        the guarantee is asymptotic (for every $\zeta>0$ the iterates are
        eventually within $\epsilon_{0,i}+\zeta$).
  \item \textup{(Parameter space, under Assumption~\ref{asm:pl})}  If a limit
        point $w^*$ exists and lies in the neighbourhood of $w_i^*$
        where~\eqref{eq:pl} holds, then
        \begin{equation}\label{eq:cp_param}
            \|w^*-w_i^*\|^2\;\le\;\frac{2}{\mu}\,\epsilon_{0,i}
            \;=\;\frac{2\lambda_{\max,i}}{\mu\,c_d}\Bigl(B\,\bar\Delta_i+\frac{L_0B^2}{2}\,\eta\Bigr)+\frac{2\,b_i}{\mu\,c_d}.
        \end{equation}
\end{enumerate}
The band radius is horizon-independent and, for fixed $(c_d,b_i,\mathcal N)$, is
controlled by two quantities only: the trajectory bound $\bar\Delta_i$ on the
deviation of the local LP direction from the shared aggregate, and the step size
$\eta$.  By Corollary~\ref{cor:collab_range},
$\Delta_i^t\le\sqrt{H_{C_n}(w^t)/\rho_{\min}}$, so one may take
$\bar\Delta_i=\sqrt{\bar H/\rho_{\min}}$ with
$\bar H:=\sup_{t\ge t_0}H_{C_n}(w^t)$; this is the quantity that the Stage-2
clustering criterion is designed to reduce (Remark~\ref{rem:cluster_motivation}).
Only in the restricted case in which the hypotheses of
Corollary~\ref{cor:homo} (the common-objective and common-optimal-basis
idealizations, Assumptions~\ref{asm:common_G}--\ref{asm:lp_stable}) hold at
every iterate $w^t$, $t\ge t_0$, does one obtain $\bar H=O(\delta_\lambda^2)$
and $\epsilon_{0,i}=O(\delta_\lambda)+O(\eta)+b_i/c_d$; in general
$\bar\Delta_i$ remains an independent, measured quantity and the band radius
is stated in terms of it.

Four qualifications apply.  (a)~The componentwise bound
$l_i(w^*)\le q_i^*+(\epsilon_{0,i}/\lambda_{\min,i})\mathbf1$ is a
\emph{one-sided} upper bound on each loss coordinate; it does not bound
$\|l_i(w^*)-q_i^*\|$ or place $l_i(w^*)$ near the Pareto front, and a point
that is far below $q_i^*$ in some coordinate is not excluded by it.  (b)~The
constants $c_d$, $b_i$ and the region $\mathcal N$ in
Assumption~\ref{asm:descent_margin} are defined through the near-active set
$J^*_{i,\varepsilon}$, whose width $\varepsilon_{\mathrm{act}}=2\eta\lambda_{\max,i}B^2$
depends on $\eta$; consequently the informal statement ``the band closes as
$\bar\Delta_i\to0$, $\eta\to0$'' is meaningful only if
Assumption~\ref{asm:descent_margin} holds with $c_d$ bounded away from $0$,
$b_i\to0$ and a fixed $\mathcal N$ uniformly in $\eta$, which we do not
establish.  (c)~Even in that idealized limit we do \emph{not} claim
$w^*=w_i^*$ unless the minimizer of $\psi_i$ is unique on $\mathcal N$ (which
Assumption~\ref{asm:pl} provides locally).  (d)~The radius is a conditional
error bound, not a smallness claim: the tolerance term $b_i/c_d$ has no proved
vanishing rate, so even for small $\bar\Delta_i$ and $\eta$ we do not assert
that $w^t$ is close to the optimum unless $b_i/c_d$ is separately controlled.
\end{theorem}

\begin{proof}
\emph{Item 1} is~\eqref{eq:cp_unrolled} of Corollary~\ref{cor:collab_monotone},
whose hypotheses are exactly~\eqref{eq:cp_traj_cond}, the segment condition,
and $\eta\le1/c_d$.

\emph{Item 2.}  Since $0\le1-\eta c_d<1$, the first term in Item~1 tends to $0$,
which gives the first inequality in~\eqref{eq:cp_obj_band}.  For every $j$,
$\lambda_i^{(j)}l_i^{(j)}(w^T)\le\psi_i(w^T)$; dividing by
$\lambda_i^{(j)}\ge\lambda_{\min,i}$, using $q_i^{*,(j)}=\psi_i^*/\lambda_i^{(j)}$
and taking $\limsup_T$ gives the second.  If $w^{t_k}\to w^*$ along a
subsequence, each $l_i^{(j)}$ is continuous (indeed $L_0$-smooth), hence so is
$\psi_i$, and $\psi_i(w^*)=\lim_k\psi_i(w^{t_k})\le\limsup_T\psi_i(w^T)\le\psi_i^*+\epsilon_{0,i}$;
the componentwise statement follows in the same way, and intersecting over
$i\in C_n$ is the definition of $\mathcal A^\infty$ in~\eqref{eq:cp_cluster_band}.

\emph{Item 3.}  Assumption~\ref{asm:pl} at $w=w^*$ gives
$\tfrac{\mu}{2}\|w^*-w_i^*\|^2\le\psi_i(w^*)-\psi_i^*\le\epsilon_{0,i}$, which
rearranges to~\eqref{eq:cp_param}.
\end{proof}

\begin{corollary}[Phase~2 as a continuation of the shared iteration]
\label{cor:phase2_cont}
Phase~2 of Algorithm~\ref{alg:1} uses the \emph{same} clients, the same local
data and objective functions $l_1,\dots,l_m$, the same aggregation weights
$\rho_i$ and the same server update $w_{r+1}=w_r+\alpha\sum_i\rho_i g_i$ as
Phase~1; the $S$-shot episodes are subsamples of the same local datasets, not
new data.  Its inner update solves~\eqref{eq:fewshot_obj}, which is the
LP~\eqref{eq:6} in balancing mode ($\mathbbm 1_{\beta_i}=1$) with the same
constraints~\eqref{eq:6b}--\eqref{eq:6c}, so the inner direction is the same
LP direction $d_i=G_i\mu_i^*$ whenever $\beta_i>0$ (for $\beta_i=0$ the
objective of~\eqref{eq:fewshot_obj} vanishes; in the analysis we adopt the
convention that the uniform-descent direction of~\eqref{eq:6} is then used, as
in Phase~1).  Consequently, under the same idealization as
Theorem~\ref{thm:collab_converge}---one inner step per round ($K=1$, in place
of $\tau=1$) and population gradients in place of the $S$-shot
gradients---the Phase-2 round map is the shared iteration~\eqref{eq:cp_sync}
with effective step size $\eta_{\mathrm{II}}:=\alpha\eta_{\mathrm{meta}}$ in
place of the Phase-1 effective step size $\eta_{\mathrm{I}}$ (the $\eta$
of~\eqref{eq:cp_sync}, into which the server step size is absorbed), and the
two phases form a single trajectory
\[
    w^0,\dots,w^{R}=w_{\mathrm{I}},\ w^{R+1},\dots,w^{R+R'}=w_{\mathrm{II}},
    \qquad
    w^{t+1}=w^t-\eta_t\,\bar d^{\,t},\quad
    \eta_t=\begin{cases}\eta_{\mathrm{I}}, & t<R,\\ \eta_{\mathrm{II}}, & t\ge R.\end{cases}
\]
Let Assumptions~\ref{asm:bounded}--\ref{asm:smooth} hold, let
Assumption~\ref{asm:descent_margin} hold with the near-active set
$J^*_{i,\varepsilon}$ taken at width
$\varepsilon_{\mathrm{act}}=2\eta_t\lambda_{\max,i}B^2$ for the step size
$\eta_t$ in force at round $t$, let $0<\eta_{\mathrm{I}},\eta_{\mathrm{II}}\le1/c_d$,
and let the trajectory hypothesis~\eqref{eq:cp_traj_cond} hold for all
$t\ge t_0$ along the concatenated trajectory (including the Phase-2 rounds),
with the segment condition of Theorem~\ref{thm:collab_converge}.  Define the
two-phase band radius
\begin{equation}\label{eq:cp_band_radius_2ph}
    \hat\epsilon_{0,i}
    \;:=\;\kappa_{H,i}\,\bar\Delta_i+\kappa_{\eta,i}\,\max(\eta_{\mathrm{I}},\eta_{\mathrm{II}})+\frac{b_i}{c_d}
    \;\ge\;\max\bigl(\epsilon_{0,i}(\eta_{\mathrm{I}}),\epsilon_{0,i}(\eta_{\mathrm{II}})\bigr).
\end{equation}
Then for every $i\in C_n$ and every $T\ge t_0$ on the concatenated trajectory,
\begin{equation}\label{eq:cp_unrolled_2ph}
    \psi_i(w^T)-\psi_i^*
    \;\le\;\Bigl(\prod_{t=t_0}^{T-1}(1-\eta_tc_d)\Bigr)\bigl(\psi_i(w^{t_0})-\psi_i^*\bigr)
    +\hat\epsilon_{0,i},
\end{equation}
so that in particular the Phase-2 output satisfies
$\psi_i(w_{\mathrm{II}})-\psi_i^*\le(1-\eta_{\min}c_d)^{R+R'-t_0}(\psi_i(w^{t_0})-\psi_i^*)+\hat\epsilon_{0,i}$
with $\eta_{\min}:=\min(\eta_{\mathrm{I}},\eta_{\mathrm{II}})$, and
Items~2--3 of Theorem~\ref{thm:collab_converge} hold for the concatenated
trajectory with $\epsilon_{0,i}$ replaced by $\hat\epsilon_{0,i}$.  Everything
that Theorem~\ref{thm:collab_converge} and Section~\ref{sec:approx_optimum}
state for the Phase-1 output $w_{\mathrm{I}}$ therefore also holds for
$w_{\mathrm{II}}$, with the band radius~\eqref{eq:cp_band_radius_2ph}.

The limitations are those already stated for Phase~1, not new ones:
(i)~the analysis covers $K=1$ inner step, exactly as it covers $\tau=1$
local step (the discussion after~\eqref{eq:cp_sync} applies verbatim to
$K>1$); (ii)~it uses population gradients, whereas Phase~2 evaluates $G_i$ and
$\boldsymbol A_i$ on an $S$-shot episode---the same population-versus-minibatch
gap as in Phase~1, but larger for $S=20$, and we do not quantify it; (iii)~the
trajectory hypothesis~\eqref{eq:cp_traj_cond} is imposed on the Phase-2 rounds
as well, and is not derived.  What the corollary rules out is a
\emph{qualitative} degradation: Phase~2 does not change the objective, the
data or the direction rule, so it cannot move the shared model out of the band
of Theorem~\ref{thm:collab_converge} other than through the step-size term
$\kappa_{\eta,i}\eta_{\mathrm{II}}$.
\end{corollary}

\begin{proof}
The one-step contraction~\eqref{eq:cp_monotone} of
Corollary~\ref{cor:collab_monotone} is proved for a single round of the
iteration~\eqref{eq:cp_sync} with step size $\eta$ under the hypotheses listed
there at that round; nothing in its proof uses the value of $\eta$ at other
rounds.  Applying it at round $t$ with $\eta=\eta_t$ gives
$r_i^{t+1}\le(1-\eta_tc_d)r_i^t+\eta_tc_d\,\epsilon_{0,i}(\eta_t)
\le(1-\eta_tc_d)r_i^t+\eta_tc_d\,\hat\epsilon_{0,i}$, using
$\epsilon_{0,i}(\eta_t)\le\hat\epsilon_{0,i}$ from~\eqref{eq:cp_band_radius_2ph}
and $\eta_tc_d\le1$.  Subtracting $\hat\epsilon_{0,i}$ from both sides yields
$r_i^{t+1}-\hat\epsilon_{0,i}\le(1-\eta_tc_d)(r_i^t-\hat\epsilon_{0,i})$, and
unrolling from $t_0$ to $T$ gives~\eqref{eq:cp_unrolled_2ph} (the product is
at most $(1-\eta_{\min}c_d)^{T-t_0}$).  The $\limsup$, limit-point and
parameter-space statements follow exactly as in the proof of
Theorem~\ref{thm:collab_converge}, with $\hat\epsilon_{0,i}$ in place of
$\epsilon_{0,i}$.
\end{proof}

\begin{remark}[What changed relative to a first-order argument]
\label{rem:why_dm}
A bound of the form $\psi_i^{t+1}\le\psi_i^t+\eta\varepsilon$ for every step,
iterated over $T$ rounds, only yields $\psi_i^T\le\psi_i^0+T\eta\varepsilon$;
it neither implies descent nor a horizon-independent distance to $\psi_i^*$.
The horizon-independent band~\eqref{eq:cp_obj_band} is obtained here because
Assumption~\ref{asm:descent_margin} supplies a decrease that is
\emph{proportional to the current suboptimality} $r_i^t$ (up to the tolerance
$b_i$), turning the one-step bound into the contraction~\eqref{eq:cp_monotone}.
Without such a condition the conclusion of Theorem~\ref{thm:collab_converge} is
false in general (see the double-well example preceding
Assumption~\ref{asm:descent_margin}), so the condition is not a technicality
that can be dropped.

\emph{Scope.}  The convergence result is conditional on a quantitative
descent-margin assumption for the actual update direction and on the iterates
remaining in the stated region.  The assumption is not implied by smoothness,
preference balance, or the LP constraints alone.  In particular, a fixed
near-active set can contain conflicting gradients near a nonsmooth optimum
(example~(b) preceding Assumption~\ref{asm:descent_margin}), so the assumption
need not hold throughout a neighbourhood of that optimum.  The theorem covers
the synchronous iteration~\eqref{eq:cp_sync} with one local step per round; the
multi-local-step Algorithm~\ref{alg:1} is not covered verbatim.
\end{remark}

%% =========================================================================
\subsection{The Initializer Approximates Each Client's Individual Optimum}
\label{sec:approx_optimum}
%% =========================================================================

Theorem~\ref{thm:collab_converge} establishes, for each same-cluster client and
under the trajectory hypothesis~\eqref{eq:cp_traj_cond} together with
Assumption~\ref{asm:descent_margin}, the eventual objective-space bound
$\limsup_t(\psi_i(w^t)-\psi_i^*)\le\epsilon_{0,i}$ and, \emph{whenever a
limit point $w^*$ of the shared iterates exists}, the inherited guarantees
$\psi_i(w^*)\le\psi_i^*+\epsilon_{0,i}$,
$l_i(w^*)\le q_i^*+(\epsilon_{0,i}/\lambda_{\min,i})\mathbf1$ and (under
Assumption~\ref{asm:pl}) $\|w^*-w_i^*\|^2\le\tfrac{2}{\mu}\epsilon_{0,i}$, with
$\epsilon_{0,i}=\kappa_{H,i}\bar\Delta_i+\kappa_{\eta,i}\eta+b_i/c_d$
from~\eqref{eq:cp_band_radius}.  Throughout this subsection $w^*$ denotes such
a limit point of the idealized one-local-step shared iteration~\eqref{eq:cp_sync},
or its output after finitely many rounds, for which
Theorem~\ref{thm:collab_converge}(1) gives the same bounds up to the
geometrically decaying transient.  By Corollary~\ref{cor:phase2_cont} this
covers both the Phase-1 output $w_{\mathrm{I}}$ and the Phase-2 output
$w_{\mathrm{II}}$ that Algorithm~\ref{alg:1} returns (Phase~2 is the same
shared iteration with the effective step size $\eta_{\mathrm{II}}$), provided
the band radius is read as the two-phase radius
$\hat\epsilon_{0,i}$ of~\eqref{eq:cp_band_radius_2ph}, i.e.\ with $\eta$
replaced by $\max(\eta_{\mathrm{I}},\eta_{\mathrm{II}})$ in
$\epsilon_{0,i}$; for brevity we keep writing $\epsilon_{0,i}$ and $\eta$
below with this understanding.  The scope is the idealization already stated
(one local/inner step per round, population gradients, the trajectory
hypothesis~\eqref{eq:cp_traj_cond}).  Within this scope, $w^*$ is, within a
quantifiable range, an approximation of \emph{each} same-cluster client's
individual optimal solution $w_i^*$.  We now re-express the same
suboptimality as an explicit decomposition---separating the collaborative
training gap from the preference spread---so that the guarantee transfers to a
new client whose preference lies in the cluster, which shares the cluster's
objective functions, but which \emph{did not} participate in training.  For
each client $i$, let
\begin{equation}\label{eq:opt_defs}
    w_i^* \in \argmin_{w}\;\psi_i(w),
    \qquad
    \psi_i^* := \psi_i(w_i^*) = \min_w \psi_i(w),
    \qquad
    \psi_i(w) = \max_{j}\lambda_i^{(j)} l_j(w),
\end{equation}
so that $w_i^*$ is the solution for preference $\boldsymbol{\lambda}_i$
(its weighted losses are balanced and lie on the $\boldsymbol{\lambda}_i^{-1}$
ray, as in the main text).

We add a standard local curvature condition, satisfied near the optimum for the
LoRA-restricted, smooth losses considered here.

\begin{assumption}[Local quadratic growth]\label{asm:pl}
For each client $i$ there is a neighborhood of $w_i^*$ and a constant
$\mu > 0$ such that
\begin{equation}\label{eq:pl}
    \tfrac{\mu}{2}\,\|w - w_i^*\|^2
    \;\leq\; \psi_i(w) - \psi_i^* .
\end{equation}
\end{assumption}

\begin{remark}
Condition~\eqref{eq:pl} is the local \emph{quadratic-growth} property of the
Chebyshev scalarization $\psi_i$.  It is weaker than, and not equivalent to, a
Polyak--{\L}ojasiewicz (PL) inequality: PL implies quadratic growth, but the
converse fails in general.  The PL-type condition that the collaborative
contraction argument needs is stated separately in
Assumption~\ref{asm:descent_margin} (with a sufficient condition in
Lemma~\ref{lem:dm_sufficient}), and the gradient-norm PL inequality needed for
the adaptation-phase contraction is stated separately in
Assumption~\ref{asm:pl_adapt}; \eqref{eq:pl} is used only to convert an
objective-space bound into a parameter-space bound and is never used as a
substitute for either.
We state~\eqref{eq:pl} as an explicit assumption: neither the LoRA
parameterization nor $L_0$-smoothness (Assumption~\ref{asm:smooth}) implies it.
It is plausible on a region where the active bottleneck is stable and the active
weighted loss has positive curvature at its minimizer, in which case the max of
finitely many such functions inherits quadratic growth, but we do not prove
this for the DPO losses used here.
\end{remark}

We first record an elementary preference-transfer bound: a new client whose
preference is close to the cluster inherits the cluster's convergence indicator
up to the preference spread.  The bound, and everything in this and the next
subsection that builds on it, is a statement about \emph{preference} transfer
under a \emph{common} loss vector: all clients and the new client evaluate the
same objective functions $l_1,\dots,l_m$ (the same $l(w)$), and only the
preference vector differs.  A new client whose loss functions differ from the
cluster's would additionally require control of $l_{\mathrm{new}}(w)-l_i(w)$ at
the initializer, along the adaptation iterates and at the compared optima; we
do not treat that case.

\begin{proposition}[Starting Point Gap: Cluster vs.\ New Client]
\label{prop:start_gap}
Suppose the new client evaluates the same objective functions $l_j$ as the
cluster (only $\boldsymbol\lambda$ differs).  If
$\boldsymbol{\lambda}_{\mathrm{new}}$ belongs to cluster $C_n$ (i.e.,
$\|\boldsymbol{\lambda}_{\mathrm{new}} - \boldsymbol{\lambda}_i\| \leq \delta_\lambda$ for all $i \in C_n$),
then:
\begin{equation}\label{eq:start_transfer}
    \psi_{\mathrm{new}}(w^*)
    \leq \psi_{C_n}(w^*)
    + \delta_\lambda\,\|l(w^*)\|_\infty,
\end{equation}
where $\psi_{C_n}(w)=\frac{1}{|C_n|}\sum_{i\in C_n}\psi_i(w)$.
\end{proposition}

\begin{proof}
For each objective $j$:
\begin{align*}
    \lambda_{\mathrm{new}}^{(j)}\,l_j(w^*)
    &= \lambda_i^{(j)}\,l_j(w^*)
    + (\lambda_{\mathrm{new}}^{(j)} - \lambda_i^{(j)})\,l_j(w^*) \\
    &\leq \lambda_i^{(j)}\,l_j(w^*)
    + |\lambda_{\mathrm{new}}^{(j)} - \lambda_i^{(j)}|\,|l_j(w^*)|.
\end{align*}
Taking $\max_j$ on both sides:
\[
    \psi_{\mathrm{new}}(w^*)
    = \max_j \lambda_{\mathrm{new}}^{(j)}\,l_j(w^*)
    \leq \max_j \lambda_i^{(j)}\,l_j(w^*)
    + \max_j |\lambda_{\mathrm{new}}^{(j)} - \lambda_i^{(j)}|\,|l_j(w^*)|.
\]
The first term is $\psi_i(w^*)$.  For the second:
$|\lambda_{\mathrm{new}}^{(j)} - \lambda_i^{(j)}| \leq \|\boldsymbol{\lambda}_{\mathrm{new}} - \boldsymbol{\lambda}_i\|
\leq \delta_\lambda$ and $|l_j(w^*)| \leq \|l(w^*)\|_\infty$.  So
$\psi_{\mathrm{new}}(w^*)\leq \psi_i(w^*) + \delta_\lambda\,\|l(w^*)\|_\infty$.
Averaging over $i \in C_n$:
\[
    \psi_{\mathrm{new}}(w^*)
    \leq \frac{1}{|C_n|}\sum_i \psi_i(w^*)
    + \delta_\lambda\,\|l(w^*)\|_\infty
    = \psi_{C_n}(w^*) + \delta_\lambda\,\|l(w^*)\|_\infty.
    \qedhere
\]
\end{proof}

\begin{corollary}[Suboptimality decomposition for an in-cluster client]
\label{cor:approx_param}
Under Assumptions~\ref{asm:bounded}--\ref{asm:smooth}, \ref{asm:descent_margin}
and~\ref{asm:pl}, the trajectory hypothesis~\eqref{eq:cp_traj_cond}, and
$\|\boldsymbol{\lambda}_i-\boldsymbol{\lambda}_{i'}\|\leq\delta_\lambda$ for all
$i,i'\in C_n$, any client $i$ whose preference lies in the cluster (a member
of $C_n$, or a client sharing the cluster's objective functions with
$\|\boldsymbol{\lambda}_i-\boldsymbol{\lambda}_{i'}\|\le\delta_\lambda$ for all
$i'\in C_n$) satisfies the objective-space decomposition
\begin{equation}\label{eq:approx_decomp}
    \psi_i(w^*)-\psi_i^*
    \;\leq\;
    \underbrace{G_{\mathrm{opt}}}_{\text{training gap}}
    + \underbrace{\delta_\lambda\,\|l(w^*)\|_\infty}_{\text{preference gap}}
    + \underbrace{\delta_\lambda\,L^*}_{\text{optimum spread}},
\end{equation}
where $G_{\mathrm{opt}} := \psi_{C_n}(w^*) - \bar\psi^*$ with
$\bar\psi^* = \frac{1}{|C_n|}\sum_{i}\psi_{i}^*$, and $L^*$ is a common upper
bound on the loss vectors at all optima that are compared,
\begin{equation}\label{eq:Lstar}
    L^*\;\ge\;\max_{i'\in C_n\cup\{i\}}\|l(w_{i'}^*)\|_\infty ,
\end{equation}
where $i$ is the client to which the corollary is applied (a member of $C_n$,
or the new client of Section~\ref{sec:data_scarce}, in which case
$w_{\mathrm{new}}^*$ must be included); $L^*$ is finite because every
$w_{i'}^*$ lies in the bounded parameter region of Assumption~\ref{asm:bounded}
on which the losses are continuous.  Writing
$\bar\epsilon_0:=\frac{1}{|C_n|}\sum_{i\in C_n}\epsilon_{0,i}\le\max_{i\in C_n}\epsilon_{0,i}$
for the cluster-average band radius of~\eqref{eq:cp_band_radius} (for
$w^*=w_{\mathrm{II}}$, of~\eqref{eq:cp_band_radius_2ph}, i.e.\ with
$\eta=\max(\eta_{\mathrm{I}},\eta_{\mathrm{II}})$),
the collaborative convergence band of Theorem~\ref{thm:collab_converge}(2)
(Corollary~\ref{cor:phase2_cont} for $w_{\mathrm{II}}$) controls
the training gap,
\begin{equation}\label{eq:Gopt_bound}
    G_{\mathrm{opt}}
    \;\leq\; \bar\epsilon_0
    \;=\;\frac{1}{|C_n|}\sum_{i\in C_n}\Bigl(\kappa_{H,i}\,\bar\Delta_i+\kappa_{\eta,i}\,\eta+\frac{b_i}{c_d}\Bigr).
\end{equation}
Consequently, under the local quadratic growth of Assumption~\ref{asm:pl},
\begin{equation}\label{eq:approx_param}
    \|w^* - w_i^*\|^2
    \;\leq\; \frac{2}{\mu}\bigl(\psi_i(w^*)-\psi_i^*\bigr)
    \;\leq\; \frac{2}{\mu}\bigl(\bar\epsilon_0
        + \delta_\lambda\,\|l(w^*)\|_\infty + \delta_\lambda\,L^*\bigr),
\end{equation}
which is horizon-independent and, in the strict form $b_i=0$, vanishes in the
limit $\bar\Delta_i,\delta_\lambda,\eta\to0$ (subject to qualification~(b) of
Theorem~\ref{thm:collab_converge}).  In the restricted case in which the
hypotheses of Corollary~\ref{cor:homo} (Assumptions~\ref{asm:common_G}
and~\ref{asm:lp_stable}) hold at every iterate along the trajectory,
$\bar\Delta_i=O(\delta_\lambda)$ and the right-hand sides are
$O(\delta_\lambda)+O(\eta)$; otherwise $\bar\Delta_i$ is a measured quantity.
\end{corollary}

\begin{proof}
Insert $\pm\psi_{C_n}(w^*)$ and $\pm\bar\psi^*$ into $\psi_i(w^*)-\psi_i^*$:
\[
    \psi_i(w^*) - \psi_i^*
    = \underbrace{\bigl[\psi_i(w^*) - \psi_{C_n}(w^*)\bigr]}_{\text{(a)}}
    + \underbrace{\bigl[\psi_{C_n}(w^*) - \bar\psi^*\bigr]}_{\text{(b) }=\,G_{\mathrm{opt}}}
    + \underbrace{\bigl[\bar\psi^* - \psi_i^*\bigr]}_{\text{(c)}}.
\]
\emph{Term (a).}  By the per-objective preference bound of
Proposition~\ref{prop:start_gap}, $\psi_i(w^*)\le\psi_{i'}(w^*)+\delta_\lambda\|l(w^*)\|_\infty$
for any $i'$; averaging over $i'$ gives
$\psi_i(w^*)-\psi_{C_n}(w^*)\le\delta_\lambda\|l(w^*)\|_\infty$.
\emph{Term (c).}  We need an upper bound on $\psi_{i'}^*-\psi_i^*$ for every
$i'\in C_n$.  By optimality of $w_{i'}^*$ for $\psi_{i'}$ and the per-objective
preference bound evaluated at $w_i^*$,
\[
    \psi_{i'}^*\;\le\;\psi_{i'}(w_i^*)
    \;\le\;\psi_i(w_i^*)+\delta_\lambda\|l(w_i^*)\|_\infty
    \;=\;\psi_i^*+\delta_\lambda\|l(w_i^*)\|_\infty ,
\]
so $\psi_{i'}^*-\psi_i^*\le\delta_\lambda\|l(w_i^*)\|_\infty\le\delta_\lambda L^*$
(the point at which the loss is evaluated is $w_i^*$, the optimum of the client
being bounded, which is covered by~\eqref{eq:Lstar}).  Averaging over $i'$
gives $\bar\psi^*-\psi_i^*\le\delta_\lambda L^*$.  Note that the reverse
inequality $\psi_i^*\le\psi_{i'}^*+\delta_\lambda\|l(w_{i'}^*)\|_\infty$
controls the opposite difference and would not suffice here.
\emph{Term (b)} is $G_{\mathrm{opt}}$ by definition; this
proves~\eqref{eq:approx_decomp}.  Averaging the objective-space guarantee of
Theorem~\ref{thm:collab_converge}(2),
$\psi_i(w^*)\le\psi_i^*+\epsilon_{0,i}$,
over $i\in C_n$ gives $G_{\mathrm{opt}}=\frac{1}{|C_n|}\sum_i(\psi_i(w^*)-\psi_i^*)
\le\bar\epsilon_0$,
which is~\eqref{eq:Gopt_bound}; it depends only on the trajectory bounds
$\bar\Delta_i$ on the direction deviation, on the step size $\eta$ and on the
tolerances $b_i$, not on the number of rounds, so it is horizon-independent.
Finally Assumption~\ref{asm:pl}
at $w=w^*$ gives $\|w^*-w_i^*\|^2\le\frac{2}{\mu}(\psi_i(w^*)-\psi_i^*)$; combining
with~\eqref{eq:approx_decomp} and~\eqref{eq:Gopt_bound}
yields~\eqref{eq:approx_param}, and for $b_i=0$ every term vanishes as
$\bar\Delta_i,\delta_\lambda,\eta\to0$.
\end{proof}

\begin{remark}[Reading the three terms of~\eqref{eq:approx_decomp}]
\label{rem:decomp_terms}
The first term, the \emph{training gap} $G_{\mathrm{opt}}$, is the
cluster-average of how far the shared model $w^*$ sits above each member's own
optimum; it is the only term that depends on the collaborative training, and it
is bounded by the cluster-average band radius $\bar\epsilon_0$
of~\eqref{eq:Gopt_bound}.  The second term, the \emph{preference gap}
$\delta_\lambda\|l(w^*)\|_\infty$, measures how much client $i$'s Chebyshev
value $\psi_i(w^*)$ at the achieved model deviates from the cluster average
$\psi_{C_n}(w^*)$ because its preference differs from the others' by at most
$\delta_\lambda$.  The third term, the \emph{optimum spread}
$\delta_\lambda L^*$, measures how far the per-client optimal values
$\psi_{i'}^*$ are spread across the cluster for the same reason; it is
independent of the trained model.  The decomposition holds for the output of
the idealized shared iteration (one local or inner step per round, population
gradients, the trajectory hypothesis~\eqref{eq:cp_traj_cond}), hence for both
$w_{\mathrm{I}}$ and, through Corollary~\ref{cor:phase2_cont}, the returned
$w_{\mathrm{II}}$, with $\bar\epsilon_0$ computed from the two-phase
radius~\eqref{eq:cp_band_radius_2ph}; it is not a statement about the
multi-step, $S$-shot implementation of Algorithm~\ref{alg:1}.
\end{remark}

%% =========================================================================
\subsection{Data-Scarce Adaptation on the Initial ALigner: the $\delta_0$ Gap}
\label{sec:data_scarce}
\phantomsection\label{sec:twophase_gap}
%% =========================================================================

We now compare two ways of serving a data-scarce client whose preference
$\boldsymbol{\lambda}_{\mathrm{new}}$ falls in cluster $C_n$ (by
the adjustment-vector criterion of Section~\ref{cluster}).  Throughout this
subsection the new client evaluates the \emph{same} objective functions
$l_1,\dots,l_m$ as the cluster and differs from its members only in the
preference vector; a new client whose loss functions themselves differ from the
cluster's (a distribution shift in $l$) is outside the scope of the bounds
below, which would then require an additional term controlling
$l_{\mathrm{new}}(w)-l_i(w)$ at the relevant points.  This client
\emph{did not participate} in initializer training in either scenario, and both
scenarios optimize the \emph{same} balancing objective~\eqref{eq:fewshot_obj}
under the \emph{same} preference $\boldsymbol{\lambda}_{\mathrm{new}}$; they
differ only in how the model is produced.
\begin{itemize}
\item \textbf{Pipeline 1:} the new client starts from a shared starting point
$w_0$ produced by the initializer training of Section~\ref{sec:initial_train}
(in Algorithm~\ref{alg:1}, $w_0=w_{\mathrm{II}}$; both $w_{\mathrm{I}}$ and
$w_{\mathrm{II}}$ are covered by Corollary~\ref{cor:approx_param} via
Corollary~\ref{cor:phase2_cont}) and
runs the $K$-step fresh-sample adaptation of~\eqref{eq:adapt} with batch size
$S$ per step toward its own preference optimum, producing $w_K$.
\item \textbf{Pipeline 2:} the client is assumed to have
\emph{abundant} data and \emph{does not use collaborative learning at all}:
it trains a model \emph{locally} on the same balancing
objective~\eqref{eq:fewshot_obj} with the same preference
$\boldsymbol{\lambda}_{\mathrm{new}}$ to convergence, reaching its own preference
optimum
\begin{equation}\label{eq:ds_target}
    w_{\mathrm{new}}^* \in \argmin_w \psi_{\mathrm{new}}(w),
    \qquad
    \psi_{\mathrm{new}}^* := \psi_{\mathrm{new}}(w_{\mathrm{new}}^*),
    \qquad
    \psi_{\mathrm{new}}(w) = \max_j \lambda_{\mathrm{new}}^{(j)} l_j(w).
\end{equation}
With abundant data and training to convergence this pipeline attains
$\psi_{\mathrm{new}}(w_{\mathrm{new}}^*)=\psi_{\mathrm{new}}^*$ and, since
$w_{\mathrm{new}}^*$ lies on the $\boldsymbol{\lambda}_{\mathrm{new}}^{-1}$ ray,
$\beta_{\mathrm{new}}(w_{\mathrm{new}}^*)=0$.
\end{itemize}
To quantify what Pipeline 1 costs relative to this data-rich baseline, we track
two quantities---the preference uniformity $\beta_i(w)$ (zero exactly on the
$\boldsymbol{\lambda}_i^{-1}$ ray) and the Chebyshev indicator
\begin{equation}\label{eq:twophase_psi}
    \psi_i(w) := \max_{j\in[m]} \lambda_i^{(j)} l_j(w),
\end{equation}
and bound the gap $\psi_{\mathrm{new}}(w_K)-\psi_{\mathrm{new}}(w_{\mathrm{new}}^*)
=\psi_{\mathrm{new}}(w_K)-\psi_{\mathrm{new}}^*$ and
$\|w_K-w_{\mathrm{new}}^*\|$.  The adaptation bound is stated from an
\emph{arbitrary} starting point $w_0$ in the region of
Assumption~\ref{asm:pl_adapt}, with $\Delta_0:=\psi_{\mathrm{new}}(w_0)-\psi_{\mathrm{new}}^*$.
When $w_0$ is the output of the idealized shared iteration---$w_{\mathrm{I}}$,
or the point $w_{\mathrm{II}}$ that Algorithm~\ref{alg:1} actually returns,
which Corollary~\ref{cor:phase2_cont} places on the same trajectory with the
two-phase band radius~\eqref{eq:cp_band_radius_2ph}---and
$\boldsymbol{\lambda}_{\mathrm{new}}$ lies in cluster $C_n$, $\Delta_0$ is
controlled by Corollary~\ref{cor:approx_param}.  Objective~(\ref{eq:6}) coincides with
the balancing objective~\eqref{eq:fewshot_obj} whenever $\beta_i>0$, so every step of
Pipeline 1 drives the same $(\beta,\psi)$ dynamics.  We show the expected gap
splits into an \emph{initialization/finite-$K$} term proportional to $\Delta_0$
and a \emph{statistical} floor decaying as $O(1/S)$
(Proposition~\ref{prop:twophase_gap}); the rest of the subsection records the
idealized adaptation model, the one-step contraction it rests on, and this gap.

The adaptation procedure runs $K$ local preference-conditioned descent steps
starting from $w_0$:
\begin{equation}\label{eq:adapt}
    w_{\mathrm{new},0} = w_0,
    \quad
    w_{\mathrm{new},k+1} = w_{\mathrm{new},k}
    - \alpha\,\tilde{d}_{\mathrm{new}}
      (w_{\mathrm{new},k}),
    \quad k = 0, \ldots, K{-}1,
\end{equation}
where $\tilde{d}_{\mathrm{new}}(w_{\mathrm{new},k})$ is the stochastic
preference-conditioned direction computed from a batch of $S$ of the client's
samples.

\paragraph{Modeling data scarcity: an idealized fresh-sample adaptation process.}
The direction $\tilde d_{\mathrm{new}}(w)$ of~\eqref{eq:adapt} is the
empirical preference-conditioned direction computed from a batch of $S$
samples.  The analysis below is carried out for an \emph{idealized}
preconditioned stochastic-gradient process with \emph{fresh samples}---a new
batch of $S$ samples is drawn independently at every step, so $K$ steps consume
$K\cdot S$ samples---whose assumptions we now list separately.  It is
\emph{not} a statistical guarantee for a fixed total budget of $S$ samples
reused across all $K$ steps (the experimental setting with $S=20$); analyzing
that case would require an optimization analysis on the fixed dataset together
with a generalization bound, which we do not provide.  It is also \emph{not}
an analysis of the AdamW optimizer used in the experiments;
Remark~\ref{rem:adamw_relation} records the differences.

\begin{assumption}[Local smoothness and PL of the adaptation objective]
\label{asm:pl_adapt}
There are a region $\mathcal U$ containing the starting point $w_0$ and a constant $\mu_{\mathrm{PL}}>0$
such that (i) on $\mathcal U$ the Chebyshev objective $\psi_{\mathrm{new}}$ has a
stable active bottleneck, so that it is differentiable and $L_0$-smooth with
descent field $\nabla\psi_{\mathrm{new}}$ and the population direction coincides
with it, $d_{\mathrm{new}}(w)=\nabla\psi_{\mathrm{new}}(w)$; (ii) the
Polyak--{\L}ojasiewicz inequality
\begin{equation}\label{eq:pl_adapt}
    \|\nabla\psi_{\mathrm{new}}(w)\|^2
    \;\ge\;2\mu_{\mathrm{PL}}\bigl(\psi_{\mathrm{new}}(w)-\psi_{\mathrm{new}}^*\bigr)
    \qquad\forall\,w\in\mathcal U ;
\end{equation}
and (iii) all adaptation iterates $w_{\mathrm{new},0},\dots,w_{\mathrm{new},K}$
and the segments joining consecutive iterates lie in $\mathcal U$ (a trajectory
hypothesis).  Condition~(ii) is a gradient-norm lower bound and is
\emph{different from} the quadratic growth of Assumption~\ref{asm:pl}: PL
implies quadratic growth but not conversely, and Assumption~\ref{asm:pl} alone
does not supply~\eqref{eq:pl_adapt}.  Likewise the descent-margin condition of
Assumption~\ref{asm:descent_margin} concerns the LP direction $d_i$ and does
not replace~\eqref{eq:pl_adapt} for $\nabla\psi_{\mathrm{new}}$.
\end{assumption}

\begin{assumption}[Idealized finite-sample adaptation direction]\label{asm:finite_sample}
The stochastic direction used at step $k$ is conditionally unbiased for the
population descent field and has variance shrinking with the sample size: for
every $w\in\mathcal U$,
\begin{equation}\label{eq:ds_fs}
    \E\bigl[\tilde d_{\mathrm{new}}(w)\,\big|\,w\bigr]
    = \nabla\psi_{\mathrm{new}}(w),
    \qquad
    \E\bigl[\|\tilde d_{\mathrm{new}}(w)-\nabla\psi_{\mathrm{new}}(w)\|^2
    \,\big|\,w\bigr]
    \;\le\; \frac{\sigma_0^2}{S},
\end{equation}
where $\sigma_0^2$ is the per-sample gradient variance.  This models
\emph{fresh-sample adaptation with batch size $S$ per step}: each adaptation
step uses a new batch of $S$ samples drawn independently from the client's
population, so the $K$ steps consume $K\cdot S$ samples in total.  It is an
idealization: if the \emph{same} fixed $S$ samples are reused at every step (a
fixed total budget of $S$), the empirical gradient is a deterministic function
of $w$ and~\eqref{eq:ds_fs} does not follow from the mere availability of $S$
samples; we do not claim that it does, and the bounds below are not
fixed-budget statistical guarantees.  ``Sufficient data'' corresponds
to $S\to\infty$ (equivalently $\sigma_0^2/S\to0$), the exact-gradient adaptation
with direction $d_{\mathrm{new}}$.
\end{assumption}

\begin{assumption}[Idealized preconditioned stochastic adaptation]\label{asm:precond}
Instead of the plain update~\eqref{eq:adapt}, the adaptation step is
\begin{equation}\label{eq:ds_precond}
    w_{\mathrm{new},k+1}
    = w_{\mathrm{new},k} - \alpha\,P_k\,
      \tilde d_{\mathrm{new}}(w_{\mathrm{new},k}),
\end{equation}
where, writing $\mathcal F_k$ for the $\sigma$-field generated by the iterate
history up to $w_{\mathrm{new},k}$ (before the step-$k$ samples are drawn):
\begin{enumerate}[leftmargin=2em,itemsep=2pt,label=\textup{(\alph*)}]
  \item the preconditioner $P_k$ is $\mathcal F_k$-measurable, i.e.\ it is fixed
        before the step-$k$ stochastic direction is sampled;
  \item $P_k$ is symmetric with uniformly bounded eigenvalues,
        $0<p_-\,I\preceq P_k\preceq p_+\,I$;
  \item the update uses the instantaneous direction $\tilde d_{\mathrm{new}}(w_{\mathrm{new},k})$
        (no momentum);
  \item there is no weight decay ($\lambda_{\mathrm{wd}}=0$).
\end{enumerate}
\end{assumption}

\begin{remark}[Relation to AdamW]\label{rem:adamw_relation}
The experiments adapt with AdamW.  The idealized process of
Assumption~\ref{asm:precond} differs from it in three respects, each of which
would require additional analysis that we do not carry out.
\emph{(1)~Momentum.}  AdamW replaces $\tilde d_{\mathrm{new}}(w_{\mathrm{new},k})$
by an exponential moving average of past stochastic gradients evaluated at
past iterates; its conditional mean given $\mathcal F_k$ is in general
\emph{not} $\nabla\psi_{\mathrm{new}}(w_{\mathrm{new},k})$, so the
``same mean, no larger variance'' shortcut is not available and the descent
step in Lemma~\ref{lem:pl_adamw_contract} cannot be taken as written.
\emph{(2)~Preconditioner measurability.}  AdamW's second-moment estimate
$v_k$ includes the current stochastic gradient, so
$P_k=\mathrm{diag}(1/(\sqrt{v_k}+\epsilon))$ is correlated with
$\tilde d_{\mathrm{new}}(w_{\mathrm{new},k})$ and is not
$\mathcal F_k$-measurable; the factorization
$\E[g_k^\top P_k\tilde d_{\mathrm{new}}\mid\mathcal F_k]=g_k^\top P_kg_k$ used
in the proof then fails.  Condition~(a) holds for a variant that uses the
previous estimate $v_{k-1}$.
\emph{(3)~Decoupled weight decay.}  AdamW's update is
$w^+=w-\alpha P_kg-\alpha\lambda_{\mathrm{wd}}w$, whereas the preconditioned
step on the regularized objective $F(w)+\tfrac{\lambda_{\mathrm{wd}}}{2}\|w\|^2$
is $w^+=w-\alpha P_kg-\alpha\lambda_{\mathrm{wd}}P_kw$.  The two coincide only
for $P_k\propto I$, so decoupled weight decay is \emph{not} equivalent to
running the same iteration on an $L_2$-regularized objective, and
Lemma~\ref{lem:pl_adamw_contract} does not transfer to it by a change of
target.  Consequently the guarantees of this subsection are statements about
the idealized process, and their relevance to the AdamW runs in
Section~\ref{sec: experiments} is empirical, not proved.
\end{remark}

\begin{lemma}[Contraction of the idealized preconditioned stochastic adaptation]
\label{lem:pl_adamw_contract}
Under Assumptions~\ref{asm:pl_adapt}, \ref{asm:finite_sample}
and~\ref{asm:precond}, the updates~\eqref{eq:ds_precond} with step size
$0<\alpha\le p_-/(L_0 p_+^2)$ satisfy the one-step contraction
\begin{equation}\label{eq:ds_onestep}
    \E\bigl[\psi_{\mathrm{new}}(w_{\mathrm{new},k+1})-\psi_{\mathrm{new}}^*\bigr]
    \;\le\;(1-\alpha\mu_{\mathrm{PL}} p_-)\,
      \E\bigl[\psi_{\mathrm{new}}(w_{\mathrm{new},k})-\psi_{\mathrm{new}}^*\bigr]
    +\frac{L_0\alpha^2 p_+^2}{2}\cdot\frac{\sigma_0^2}{S},
\end{equation}
and hence, telescoping from $w_{\mathrm{new},0}=w_0$,
\begin{equation}\label{eq:ds_Kstep}
    \E\bigl[\psi_{\mathrm{new}}(w_{\mathrm{new},K})-\psi_{\mathrm{new}}^*\bigr]
    \;\le\;(1-\alpha\mu_{\mathrm{PL}} p_-)^K
      \bigl(\psi_{\mathrm{new}}(w_0)-\psi_{\mathrm{new}}^*\bigr)
    +\frac{L_0\alpha p_+^2}{2\mu_{\mathrm{PL}} p_-}\cdot\frac{\sigma_0^2}{S}.
\end{equation}
Both bounds hold in expectation over the sampling of the adaptation directions;
they are not per-run statements.
\end{lemma}

\begin{proof}
Write $F:=\psi_{\mathrm{new}}$, $g_k:=\nabla F(w_{\mathrm{new},k})$.  By
Assumption~\ref{asm:pl_adapt}(i),(iii) $F$ is $L_0$-smooth along the segment
from $w_{\mathrm{new},k}$ to $w_{\mathrm{new},k+1}$, so the descent lemma applied
to~\eqref{eq:ds_precond} gives
\[
    F(w_{\mathrm{new},k+1})
    \le F(w_{\mathrm{new},k})
    -\alpha\,g_k^\top P_k\,\tilde d_{\mathrm{new}}
    +\frac{L_0\alpha^2}{2}\|P_k\,\tilde d_{\mathrm{new}}\|^2 .
\]
Take conditional expectation given $\mathcal F_k$.  Since $P_k$ is
$\mathcal F_k$-measurable (Assumption~\ref{asm:precond}(a)) it is constant
under this expectation, so by unbiasedness~\eqref{eq:ds_fs}
$\E[g_k^\top P_k\,\tilde d_{\mathrm{new}}\mid\mathcal F_k]=g_k^\top P_k g_k
\ge p_-\|g_k\|^2$ (using $P_k\succeq p_-I$), and by
$P_k\preceq p_+I$ together with the bias--variance split,
$\E[\|P_k\,\tilde d_{\mathrm{new}}\|^2\mid\mathcal F_k]
\le p_+^2\,\E[\|\tilde d_{\mathrm{new}}\|^2\mid\mathcal F_k]
\le p_+^2(\|g_k\|^2+\sigma_0^2/S)$.  Therefore
\[
    \E\bigl[F(w_{\mathrm{new},k+1})\mid\mathcal F_k\bigr]
    \le F(w_{\mathrm{new},k})
    -\alpha\Bigl(p_- -\tfrac{L_0\alpha p_+^2}{2}\Bigr)\|g_k\|^2
    +\frac{L_0\alpha^2 p_+^2}{2}\cdot\frac{\sigma_0^2}{S}.
\]
Since $\alpha\le p_-/(L_0 p_+^2)$ we have
$p_- -\tfrac{L_0\alpha p_+^2}{2}\ge\tfrac{p_-}{2}$, so
$-\alpha(p_- -\tfrac{L_0\alpha p_+^2}{2})\|g_k\|^2\le-\tfrac{\alpha p_-}{2}\|g_k\|^2$.
Applying the PL inequality~\eqref{eq:pl_adapt} of
Assumption~\ref{asm:pl_adapt}(ii) at $w_{\mathrm{new},k}\in\mathcal U$,
$\|g_k\|^2\ge2\mu_{\mathrm{PL}}(F(w_{\mathrm{new},k})-F^*)$, gives
$-\tfrac{\alpha p_-}{2}\|g_k\|^2\le-\alpha\mu_{\mathrm{PL}} p_-(F(w_{\mathrm{new},k})-F^*)$;
subtracting $F^*$ from both sides and taking total expectation
yields~\eqref{eq:ds_onestep}.
Iterating~\eqref{eq:ds_onestep} and summing the geometric series,
\begin{equation*}
\begin{aligned}
    \E[F(w_{\mathrm{new},K})-F^*]
    \le(1-\alpha\mu_{\mathrm{PL}} p_-)^K(F(w_0)-F^*)
    +\frac{L_0\alpha^2 p_+^2}{2}\cdot\frac{\sigma_0^2}{S}
      \sum_{j=0}^{K-1}(1-\alpha\mu_{\mathrm{PL}} p_-)^j
    \\ \le(1-\alpha\mu_{\mathrm{PL}} p_-)^K(F(w_0)-F^*)
    +\frac{L_0\alpha p_+^2}{2\mu_{\mathrm{PL}} p_-}\cdot\frac{\sigma_0^2}{S},
\end{aligned}
\end{equation*}
using $\sum_{j\ge0}(1-\alpha\mu_{\mathrm{PL}} p_-)^j=1/(\alpha\mu_{\mathrm{PL}} p_-)$.  This
is~\eqref{eq:ds_Kstep}.
\end{proof}

\begin{proposition}[Gap: initializer based few-shot adaptation vs.\ data-rich local
training]\label{prop:twophase_gap}
Under Assumptions~\ref{asm:bounded}--\ref{asm:smooth}, \ref{asm:pl}
(applied to $\psi_{\mathrm{new}}$ at $w_{\mathrm{new}}^*$), \ref{asm:pl_adapt},
\ref{asm:finite_sample} and~\ref{asm:precond}, with
$0<\alpha\le p_-/(L_0p_+^2)$, consider a new client that evaluates the same
objective functions $l_j$ as the cluster, whose preference
$\boldsymbol{\lambda}_{\mathrm{new}}$ lies in cluster $C_n$
($\|\boldsymbol{\lambda}_{\mathrm{new}}-\boldsymbol{\lambda}_i\|\le\delta_\lambda$
for $i\in C_n$) and that did not participate in initializer training.  Let
\begin{itemize}
\item $w_K$ be the \emph{Pipeline 1} output: $K$ idealized fresh-sample
preconditioned stochastic steps~\eqref{eq:ds_precond} with batch size $S$ per
step, started from a point $w_0$ in the region $\mathcal U$ of
Assumption~\ref{asm:pl_adapt}; and
\item $w^{(2)}:=w_{\mathrm{new}}^*$ be the \emph{Pipeline 2} output: the data-rich
local model that trains~\eqref{eq:fewshot_obj} to convergence with the same
preference, so $\psi_{\mathrm{new}}(w^{(2)})=\psi_{\mathrm{new}}^*$
and $\beta_{\mathrm{new}}(w^{(2)})=0$.
\end{itemize}
Write $\psi_{\mathrm{new}}^*=\min_w\psi_{\mathrm{new}}(w)$, denote the
starting-point suboptimality by
$\Delta_0=\Delta_0(w_0):=\psi_{\mathrm{new}}(w_0)-\psi_{\mathrm{new}}^*$, and assume
$\beta_{\mathrm{new}}$ is locally $L_\beta$-Lipschitz in the loss vector on the
adaptation region.  Then:
\begin{itemize}
\item[(a)] \emph{Chebyshev indicator.}  Pipeline 1 obeys, in expectation over
the sampling of the adaptation directions,
\begin{equation}\label{eq:tp_psi_scarce}
    \E\bigl[\psi_{\mathrm{new}}(w_K)-\psi_{\mathrm{new}}^*\bigr]
    \;\le\; (1-\alpha\mu_{\mathrm{PL}} p_-)^{K}\,\Delta_0(w_0)
    +\frac{L_0\alpha p_+^2}{2\mu_{\mathrm{PL}} p_-}\cdot\frac{\sigma_0^2}{S},
\end{equation}
while Pipeline 2 attains $\psi_{\mathrm{new}}(w^{(2)})-\psi_{\mathrm{new}}^*=0$.
For a starting point $w_0\in\{w_{\mathrm{I}},w_{\mathrm{II}}\}$ produced by
the idealized shared iteration (Section~\ref{sec:approx_optimum};
$w_{\mathrm{II}}$ is covered by Corollary~\ref{cor:phase2_cont}), if the
hypotheses of Corollary~\ref{cor:approx_param} hold with $L^*$ covering
$w_{\mathrm{new}}^*$ as in~\eqref{eq:Lstar}, then
\begin{equation}\label{eq:tp_psi_full}
    \Delta_0(w_0) \;\le\;
    G_{\mathrm{opt}}+\delta_\lambda\|l(w_0)\|_\infty+\delta_\lambda L^*
    \;\le\;\bar\epsilon_0+\delta_\lambda\bigl(\|l(w_0)\|_\infty+L^*\bigr),
\end{equation}
where for $w_0=w_{\mathrm{II}}$ the cluster-average band radius
$\bar\epsilon_0$ is that of the two-phase radius~\eqref{eq:cp_band_radius_2ph}
(step size $\max(\eta_{\mathrm{I}},\eta_{\mathrm{II}})$).  In particular the
bound applies to the point Algorithm~\ref{alg:1} actually returns, within the
idealization of Corollary~\ref{cor:phase2_cont} ($K=1$ inner step, population
gradients, trajectory hypothesis on the Phase-2 rounds).
\item[(b)] \emph{Parameter space.}  If, in addition, $w_K$ lies in the
neighbourhood of $w_{\mathrm{new}}^*$ where the local quadratic growth
$\frac{\mu}{2}\|w_K-w_{\mathrm{new}}^*\|^2\le\psi_{\mathrm{new}}(w_K)
-\psi_{\mathrm{new}}^*$ of Assumption~\ref{asm:pl} holds, then
\begin{equation}\label{eq:tp_beta}
    \E\|w_K-w_{\mathrm{new}}^*\|^2
    \;\le\;\frac{2}{\mu}\Bigl[(1-\alpha\mu_{\mathrm{PL}} p_-)^{K}\,\Delta_0
    +\frac{L_0\alpha p_+^2}{2\mu_{\mathrm{PL}} p_-}\cdot\frac{\sigma_0^2}{S}\Bigr],
    \qquad \|w^{(2)}-w_{\mathrm{new}}^*\|=0,
\end{equation}
hence, by Jensen's inequality,
$\E\|w_K-w_{\mathrm{new}}^*\|\le\bigl(\E\|w_K-w_{\mathrm{new}}^*\|^2\bigr)^{1/2}$,
and the preference uniformity of Pipeline 1 satisfies
$\E\,\beta_{\mathrm{new}}(w_K)\le L_\beta\,\E\|w_K-w_{\mathrm{new}}^*\|$ while
$\beta_{\mathrm{new}}(w^{(2)})=0$.  These are mean-square and mean statements;
no per-run bound on $\|w_K-w_{\mathrm{new}}^*\|$ is claimed.
\item[(c)] \emph{Gap.}  The expected suboptimality of the collaborative
data-scarce pipeline relative to data-rich local training splits into two
independent sources,
\begin{equation}\label{eq:tp_gap}
    \E[\Delta\psi]
    :=\E\bigl[\psi_{\mathrm{new}}(w_K)-\psi_{\mathrm{new}}(w^{(2)})\bigr]
    \;\le\;
    \underbrace{(1-\alpha\mu_{\mathrm{PL}} p_-)^{K}\,\Delta_0}_{\text{initialization }+\text{ finite-}K}
    \;+\;
    \underbrace{\frac{L_0\alpha p_+^2}{2\mu_{\mathrm{PL}} p_-}\cdot\frac{\sigma_0^2}{S}}_{\text{statistical scarcity}},
\end{equation}
and $\E[\Delta\beta]:=\E\bigl[\beta_{\mathrm{new}}(w_K)-\beta_{\mathrm{new}}(w^{(2)})\bigr]
=O\!\bigl(\sqrt{(1-\alpha\mu_{\mathrm{PL}} p_-)^{K}\Delta_0}+\sqrt{\sigma_0^2/S}\bigr)$.  The
initialization term depends on the starting point only through
$\Delta_0(w_0)$ and decays geometrically in $K$.  For
$w_0\in\{w_{\mathrm{I}},w_{\mathrm{II}}\}$ its prefactor is controlled
by~\eqref{eq:tp_psi_full}, which is independent of the round counts $R,R'$
(horizon-independent by Corollary~\ref{cor:approx_param}) but whose
$\bar\epsilon_0$ contains the tolerance term
$\tfrac{1}{|C_n|}\sum_ib_i/c_d$ that is not shown to be small
(Theorem~\ref{thm:collab_converge}), so $\Delta_0(w_0)\to0$ as
$\bar\Delta_i,\delta_\lambda,\eta_{\mathrm{I}},\eta_{\mathrm{II}}\to0$ is
asserted only in the strict form $b_i=0$.  The statistical term vanishes as
$S\to\infty$.
\end{itemize}
\end{proposition}

\begin{proof}
Every adaptation step of Pipeline 1 is an idealized preconditioned stochastic
step on the Chebyshev objective $\psi_{\mathrm{new}}$ with the balancing
direction of~\eqref{eq:fewshot_obj}, whose finite-sample error obeys
Assumption~\ref{asm:finite_sample}.  Applying the one-step contraction of
Lemma~\ref{lem:pl_adamw_contract} with $F=\psi_{\mathrm{new}}$ over the $K$
adaptation steps started from $w_{\mathrm{new},0}=w_0$ and telescoping the
geometric series gives~\eqref{eq:ds_Kstep} with
$F(w_0)-F^*=\psi_{\mathrm{new}}(w_0)-\psi_{\mathrm{new}}^*=\Delta_0(w_0)$, which
is~\eqref{eq:tp_psi_scarce}.  For $w_0\in\{w_{\mathrm{I}},w_{\mathrm{II}}\}$:
because the new client shares the cluster's objective functions and
$\boldsymbol{\lambda}_{\mathrm{new}}$ lies in cluster $C_n$,
Corollary~\ref{cor:approx_param} applied to $i=\mathrm{new}$
(so that $L^*$ covers $w_{\mathrm{new}}^*$, cf.~\eqref{eq:Lstar}; Term~(c) of
its proof then reads
$\bar\psi^*-\psi_{\mathrm{new}}^*\le\delta_\lambda\|l(w_{\mathrm{new}}^*)\|_\infty\le\delta_\lambda L^*$)
bounds $\Delta_0(w_0)$ exactly as
in~\eqref{eq:approx_decomp}--\eqref{eq:Gopt_bound}, giving~\eqref{eq:tp_psi_full};
for $w_0=w_{\mathrm{II}}$ the training-gap bound~\eqref{eq:Gopt_bound} is
supplied by Corollary~\ref{cor:phase2_cont} with the two-phase
radius~\eqref{eq:cp_band_radius_2ph}.  Pipeline 2 trains~\eqref{eq:fewshot_obj} locally
with abundant data ($\sigma_0^2/S=0$) to convergence, so its transient vanishes
and it attains $\psi_{\mathrm{new}}(w^{(2)})=\psi_{\mathrm{new}}^*$, i.e.\
$w^{(2)}=w_{\mathrm{new}}^*$.  Subtracting the deterministic value
$\psi_{\mathrm{new}}(w^{(2)})=\psi_{\mathrm{new}}^*$ inside the expectation
in~\eqref{eq:tp_psi_scarce} gives~\eqref{eq:tp_gap}: an
initialization/finite-$K$ term and a statistical floor.

For~(b), local quadratic growth (Assumption~\ref{asm:pl}) converts the
$\psi_{\mathrm{new}}$-suboptimality of the scarce run into a parameter-space
deviation, $\frac{\mu}{2}\|w_K-w_{\mathrm{new}}^*\|^2\le\psi_{\mathrm{new}}(w_K)
-\psi_{\mathrm{new}}^*$; taking expectations and
substituting~\eqref{eq:tp_psi_scarce} gives~\eqref{eq:tp_beta}, and
$\|w^{(2)}-w_{\mathrm{new}}^*\|=0$ since Pipeline 2 reaches the optimum.  Note
that the adaptation begins at $w_{\mathrm{new},0}=w_0$, where
Assumption~\ref{asm:pl} gives the starting distance
$\|w_0-w_{\mathrm{new}}^*\|^2\le\frac{2}{\mu}\Delta_0(w_0)$ (for
$w_0=w_{\mathrm{I}}$ this is~\eqref{eq:approx_param} of Corollary~\ref{cor:approx_param});
\eqref{eq:tp_beta} is exactly this starting distance after the
$(1-\alpha\mu_{\mathrm{PL}} p_-)^{K}$ contraction, plus the statistical floor.
Because $w_{\mathrm{new}}^*$ lies on the $\boldsymbol{\lambda}_{\mathrm{new}}^{-1}$
ray, $\beta_{\mathrm{new}}(w_{\mathrm{new}}^*)=0$; Lipschitzness of
$\beta_{\mathrm{new}}$ in the loss vector (hence, through $L_0$-smoothness, in
$w$) then gives $\beta_{\mathrm{new}}(w_K)\le L_\beta\|w_K-w_{\mathrm{new}}^*\|$
pointwise, and taking expectations together with Jensen's inequality
$\E\|w_K-w_{\mathrm{new}}^*\|\le(\E\|w_K-w_{\mathrm{new}}^*\|^2)^{1/2}$
and~\eqref{eq:tp_beta} yields
$\E[\Delta\beta]=O(\sqrt{(1-\alpha\mu_{\mathrm{PL}} p_-)^{K}\Delta_0}+\sqrt{\sigma_0^2/S})$.
For $w_0\in\{w_{\mathrm{I}},w_{\mathrm{II}}\}$ the initialization term inherits
horizon-independence from Corollary~\ref{cor:approx_param} (its bound depends
on the cluster's $\bar\Delta_i,\delta_\lambda,b_i$ and the step sizes, not on
the round counts $R,R'$), completing~(c).
\end{proof}

\begin{remark}[Local adaptation contracts the initializer gap]\label{rem:gap_reduction}
This makes precise how Pipeline 1 closes the gap that the subsection
``The Initializer Approximates Each Client's Individual Optimum'' had already
measured at the output of the shared iteration.  \emph{Before adaptation},
when the starting point is $w_0\in\{w_{\mathrm{I}},w_{\mathrm{II}}\}$,
Corollary~\ref{cor:approx_param} (with Corollary~\ref{cor:phase2_cont} for
$w_{\mathrm{II}}$) bounds the gap to the new client's own optimum
$w_{\mathrm{new}}^*$ in both spaces: in objective space
\begin{equation*}
    \psi_{\mathrm{new}}(w_0)-\psi_{\mathrm{new}}^*
    \;=\;\Delta_0(w_0)
    \;\le\; G_{\mathrm{opt}}+\delta_\lambda\|l(w_0)\|_\infty+\delta_\lambda L^*
    \qquad(\text{Eq.~\eqref{eq:approx_decomp}}),
\end{equation*}
and in parameter space
$\|w_0-w_{\mathrm{new}}^*\|^2\le\frac{2}{\mu}\Delta_0(w_0)$
(Eq.~\eqref{eq:approx_param}).  \emph{The adaptation
does not create this gap---it inherits it as the initial value and shrinks
it.}  Writing the objective gap before and after the $K$ local steps,
\begin{equation*}
    \underbrace{\psi_{\mathrm{new}}(w_{\mathrm{new},0})-\psi_{\mathrm{new}}^*
      =\Delta_0(w_0)}_{\text{start (at }w_0\text{)}}
    \;\;\xrightarrow{\;K\text{ steps}\;}\;\;
    \underbrace{\E\bigl[\psi_{\mathrm{new}}(w_K)-\psi_{\mathrm{new}}^*\bigr]
      \le(1-\alpha\mu_{\mathrm{PL}} p_-)^{K}\Delta_0
      +\tfrac{L_0\alpha p_+^2}{2\mu_{\mathrm{PL}} p_-}\tfrac{\sigma_0^2}{S}}_{\text{after (Eq.~\eqref{eq:tp_psi_scarce})}},
\end{equation*}
each step multiplies the expected population gap by the contraction factor
$1-\alpha\mu_{\mathrm{PL}} p_-\in(0,1)$, so $K$ steps reduce the initialization
gap by $(1-\alpha\mu_{\mathrm{PL}} p_-)^{K}$ down to the statistical floor.
Equivalently, the (root-mean-square) distance to the true optimum contracts from
$\|w_0-w_{\mathrm{new}}^*\|\le\sqrt{\tfrac{2}{\mu}\Delta_0}$ to
$\E\|w_K-w_{\mathrm{new}}^*\|\le\bigl(\E\|w_K-w_{\mathrm{new}}^*\|^2\bigr)^{1/2}
\le\sqrt{\tfrac{2}{\mu}\bigl[(1-\alpha\mu_{\mathrm{PL}} p_-)^{K}
\Delta_0+\tfrac{L_0\alpha p_+^2}{2\mu_{\mathrm{PL}} p_-}\tfrac{\sigma_0^2}{S}\bigr]}$
(Eq.~\eqref{eq:tp_beta}).  Since Pipeline 2 sits exactly at
$w_{\mathrm{new}}^*$ (zero gap, $\beta_{\mathrm{new}}=0$), the Pipeline~1
vs.\ Pipeline~2 gap $\E[\Delta\psi]$ of~\eqref{eq:tp_gap} is precisely the
\emph{residual} of this contraction: the starting point sets \emph{where} the
contraction starts ($\Delta_0(w_0)$, bounded by Cor.~\ref{cor:approx_param}
in terms of cluster tightness, step sizes and descent tolerance for
$w_0\in\{w_{\mathrm{I}},w_{\mathrm{II}}\}$), and local adaptation sets
\emph{how far} it travels (the
factor $(1-\alpha\mu_{\mathrm{PL}} p_-)^{K}$), leaving only the irreducible
$O(\sigma_0^2/S)$ scarcity floor.
\end{remark}

\begin{remark}[Interpretation]\label{rem:twophase_interp}
The gap in~\eqref{eq:tp_gap} isolates the two prices Pipeline 1 pays relative to
a data-rich client that self-trains to its own optimum; viewed through
Remark~\ref{rem:gap_reduction}, these are the two components of the residual
after adaptation contracts the initializer gap.  \emph{(i)~An
initialization/finite-$K$ term $(1-\alpha\mu_{\mathrm{PL}} p_-)^{K}\Delta_0$.}  A data-scarce
client cannot self-train to $w_{\mathrm{new}}^*$; it instead borrows the
shared starting point $w_0$ and refines it with only $K$ adaptation steps.
For $w_0\in\{w_{\mathrm{I}},w_{\mathrm{II}}\}$ the prefactor
$\Delta_0\le\bar\epsilon_0+\delta_\lambda(\|l(w_0)\|_\infty+L^*)$
is controlled by the cluster tightness ($\bar\Delta_i,\delta_\lambda$), by the
step sizes ($\eta_{\mathrm{I}}$, and $\eta_{\mathrm{II}}$ for $w_{\mathrm{II}}$
via Corollary~\ref{cor:phase2_cont}), and by the descent-tolerance term
$\tfrac{1}{|C_n|}\sum_ib_i/c_d$ inside $\bar\epsilon_0$; the last term has no
proved vanishing rate (Theorem~\ref{thm:collab_converge}), so we do not claim
that $\Delta_0$ is small, only that it is explicitly bounded under the stated
conditions (Corollary~\ref{cor:approx_param}), which for $w_{\mathrm{II}}$
include the one-inner-step, population-gradient idealization of Phase~2.
Whatever its size, this term decays geometrically in $K$.  \emph{(ii)~A statistical floor $O(\sigma_0^2/S)$} from adapting on only $S$
samples.  Pipeline 2 pays neither: it has abundant data (no floor) and optimizes
its own objective directly (no initialization gap), reaching
$\psi_{\mathrm{new}}^*$ with $\beta_{\mathrm{new}}=0$.  The expected gap therefore
vanishes as $K\to\infty$, $S\to\infty$, and, in the strict form $b_i=0$, as the
cluster tightens ($\bar\Delta_i,\delta_\lambda\to0$ with $\eta\to0$)---i.e.\ once
$S\gtrsim L_0\alpha p_+^2\sigma_0^2/(\mu_{\mathrm{PL}}^2 p_-\varepsilon)$, $K$ is a
few dozen adaptation steps, and $H_{C_n},\delta_\lambda$ are small.  The same two
terms govern the parameter space: by Proposition~\ref{prop:twophase_gap}(b),
$\E\|w_K-w_{\mathrm{new}}^*\|
=O\bigl(\sqrt{(1-\alpha\mu_{\mathrm{PL}} p_-)^{K}\Delta_0}+\sqrt{\sigma_0^2/S}\bigr)$, so in
expectation the collaborative data-scarce client matches the data-rich local
optimum up to an initialization term and an $O(\sqrt{1/S})$ deviation.  All of
this is for the idealized adaptation process of Assumption~\ref{asm:precond};
see Remark~\ref{rem:adamw_relation} for the gap to AdamW.
\end{remark}

\end{document}